\documentclass[journal]{IEEEtran}
\PassOptionsToPackage{usenames,dvipsnames,table}{xcolor}
\usepackage{enumerate}
\usepackage[T1]{fontenc}
\usepackage{lmodern}
\usepackage{mathptmx}
\usepackage{smile}
\usepackage[protrusion=true, expansion=true, nopatch=footnote]{microtype}
\usepackage{float,adjustbox,cancel,multirow,framed}
\usepackage[table]{xcolor}
\usepackage{amsfonts,amsmath,amssymb,amsthm,url,xspace,mathtools,mathrsfs}
\usepackage{algorithmic,algorithm}
\usepackage{array}
\usepackage[caption=false,font=footnotesize]{subfig}
\usepackage{textcomp}
\usepackage{stfloats}
\usepackage{url}
\usepackage{verbatim}
\usepackage{graphicx}
\usepackage[nospace]{cite}
\usepackage{xspace}
\usepackage{booktabs}
\usepackage{bibunits}
\usepackage{refcount}
\usepackage[colorlinks=true,citecolor=blue,linkcolor=red,urlcolor=black]{hyperref}
\makeatletter
\renewcommand{\tablename}{Table}
\let\REALMmakecaption\@makecaption
\long\def\@makecaption#1#2{%
  \ifx\@captype\@IEEEtablestring
    \begingroup
      \normalfont\footnotesize
      \par\raggedright
      \@IEEEtabletopskipstrut
      #1.\nobreakspace\nobreakspace #2\par
      \addvspace{0.5\baselineskip}%
    \endgroup
    \@IEEEtablecaptionsepspace
  \else
    \REALMmakecaption{#1}{#2}%
  \fi}
\def\@extra@b@citeb{\@bibunitname}
\AtBeginDocument{%
  \renewcommand{\bibcite}[2]{%
    \global\@namedef{b@#1\@bibunitname}{%
      \hyper@@link[cite]{}{cite.#1\@bibunitname}{#2}}}%
}

\let\REALMmaketitle\maketitle
\let\REALMtitleblock\@maketitle
\newcommand{\makesupplementtitle}{%
  \let\@maketitle\REALMtitleblock
  \REALMmaketitle}
\makeatother
\usepackage{tikz}
\usepackage{verbatim,comment}
\usetikzlibrary{arrows,shapes}
\usepackage[bottom]{footmisc}
\usepackage{lscape,diagbox}

\ifx\counterwithout\undefined\usepackage{chngcntr}\fi
\counterwithout{equation}{section}
\definecolor{shadecolor}{RGB}{240,240,240}
\usepackage{math}
\allowdisplaybreaks[1]
\allowdisplaybreaks

\usepackage{alphalph}

\newcommand{\method}{REALM\xspace}

\begin{document}
\begin{bibunit}[IEEEtran]

\title{\method: Regime-Switching, Explainable, and Activation-Induced \\Linear Models}

\author{Xiaoran Cheng, Sen Na, and Jia Li,~\IEEEmembership{Fellow,~IEEE}
\thanks{X. Cheng and J. Li are with the Department of Statistics, The Pennsylvania State University, University Park, PA 16802 USA (e-mail: xjc5161@psu.edu; jol2@psu.edu).}
\thanks{S. Na is with the H. Milton Stewart School of Industrial and Systems Engineering, Georgia Institute of Technology, Atlanta, GA 30332 USA (e-mail: senna@gatech.edu).}
}

\markboth{}
{Cheng \MakeLowercase{\textit{et al.}}: \method: Regime-Switching, Explainable, and Activation-Induced Linear Models}

\maketitle

\begin{abstract}
\begingroup
\spaceskip=\fontdimen2\font plus \fontdimen3\font minus \fontdimen4\font\relax
Deep ReLU networks are piecewise-affine mappings that partition the input space into cells, each characterized by a distinct activation pattern. This structure motivates fitting a local linear model within each cell to preserve predictive accuracy while improving interpretability. The challenge is to identify regimes that are stable, data-adaptive, and easy to explain. We propose \method, a mixture of linear models whose regimes are induced by neural activation patterns. Because the number of activation cells in a deep neural network (DNN) can grow rapidly with depth, we first distill a deep teacher into a wide, shallow student network (WSSN), then binarize and cluster its hidden-layer activations to define the regimes and fit a linear model within each regime. Since the regimes are discovered from internal structure, the router does not carry the predictive burden. To make regime assignment interpretable, we train a multiclass logistic regression, the explanatory gate, to reproduce the regime assignments. The two-level structure is interpretable at both stages in terms of raw tabular or learned convolutional features: the gate identifies features that determine regime assignments, while the linear models identify features that drive predictions within each regime. We analyze an idealized setting that illustrates a trade-off between partition complexity and stability: as the number of regimes grows, finer partitions can improve approximation but may reduce regime-assignment stability. Experiments on tabular and image datasets show that \method achieves competitive predictive performance relative to other DNN-guided mixture surrogates and inherently interpretable models while producing stable regime-level explanations.
\par
\endgroup
\end{abstract}


\section{Introduction}\label{sec:introduction}
DNNs are now standard tools in vision, language, and scientific computing, but their black-box nature remains a major barrier in settings that demand auditing, debugging, or accountability~\cite{Rudin2019Stop, Mumuni2025Explainable, Caruana2015Intelligible, Tabassi2023Artificial}. A common strategy is to replace a trained DNN with an interpretable surrogate, for example by distilling the DNN into a soft decision tree~\cite{Frosst2017Distilling}. However, a single global surrogate often struggles to capture the complex decision logic of a highly nonlinear model without sacrificing predictive accuracy. A more expressive alternative is to use a \emph{collection} of simple models, each responsible for a distinct regime of the input space.

This divide-and-conquer view motivates \emph{adaptive surrogate modeling}: a complex model can be approximated by partitioning the input space into regimes and fitting a simple local model within each regime. The central design question then becomes how to construct the partition: \emph{what criterion should define the regimes so that the surrogate remains accurate while both the~regime partitioning and the local models remain interpretable?}

Mixture-of-experts (MoE) models~\cite{Jacobs1991Adaptive, Jordan1993Hierarchical} provide a natural framework for \emph{adaptive surrogate modeling}. An MoE uses a router to assign each input to one or a few simple local experts. However, in classical MoE models, the router and experts are jointly optimized for predictive accuracy, so the resulting partition is not explicitly designed for interpretability~\cite{Jacobs1991Adaptive, Jordan1993Hierarchical}. Subsequent variants have introduced explicit input partitioning~\cite{Tang2002Input}, Bayesian non-parametric mixtures for explanation~\cite{Guo2018Explaining}, boosted mixtures of experts~\cite{Yuksel2012Twenty}, and sparsely gated layers for conditional computation~\cite{Shazeer2017Outrageously}. These methods expand the flexibility and scalability of MoE models, but they still do not provide a fully interpretable prediction path for a trained DNN.

The central challenge is that the router can absorb substantial predictive logic. Because the router and experts are trained jointly, the resulting decision logic can become complex~\cite{Badjie2026Decoupling, Schwab2025Interpretable}, making it difficult to disentangle whether a prediction is driven by the router, the experts, or their interaction. Large sparse MoE models can also suffer from router-induced training instability~\cite{Zoph2022STMoE}. Constraining either component to be inherently interpretable~\cite{Vasic2022MoET, Ismail2022Interpretable, Ghosh2023Dividing} often comes at the cost of predictive accuracy. Interpretability may further deteriorate under diffuse routing, where a single input activates many experts with similar weights, obscuring feature attribution. Moreover, when the partition is learned directly from the input, it may define a new predictive structure of its own rather than reflect the internal regime structure of the trained DNN~\cite{Jacobs1991Adaptive, Jordan1993Hierarchical, Ismail2022Interpretable, Shazeer2017Outrageously, Vasic2022MoET}. In this case, the router itself carries part of the predictive burden.

These limitations suggest that an interpretable DNN surrogate should not learn the partition solely through the inputs by optimizing prediction accuracy jointly with the experts. Instead, the regimes should align with the DNN's own internal geometric structure. ReLU networks provide a natural signal for this purpose. A trained ReLU network partitions the input space into polyhedral cells, within each of which the network reduces to an affine map~\cite{Hwang2020Un, Picchiotti2022Clustering, He2020ReLU, Sudjianto2020Unwrapping, Montufar2014Number, Raghu2017Expressive}. Thus, its activation patterns provide a natural basis for defining regimes that support locally linear modeling. Since ReLU is widely used in modern architectures~\cite{Dubey2022Activation}, this structure can be exploited to construct interpretable surrogates.

However, directly using the exact partition induced by activation patterns is impractical. The number of activation cells grows rapidly with network depth~\cite{Montufar2014Number, Raghu2017Expressive}, and a trained network may generate nearly instance-specific activation patterns~\cite{Picchiotti2022Clustering, Barua2026Mechanistic}. This limits interpretability because the resulting explanations become highly localized and difficult to compare across instances.

Several families of interpretable surrogates address parts of this challenge. IME~\cite{Ismail2022Interpretable} constructs a mixture in which both the experts and the assignment module are interpretable, while the route--interpret--repeat framework~\cite{Ghosh2023Dividing} iteratively decomposes a black-box model into interpretable models, each covering the samples it can explain. However, in both approaches, the partition is still learned through routing on the inputs, so the resulting regimes are determined by the routing objective rather than derived from the internal structure of the trained network. More closely tied to ReLU network geometry, ReLU Region Reasoning (Re3)~\cite{Barua2026Mechanistic} computes the exact per-sample affine map associated with each activation pattern, while MASALA~\cite{Anwar2024MASALA} fits a surrogate over an adaptively chosen neighborhood around each instance. Although both methods exploit local structure, neither coarsens the exact activation cells into a shared set of regimes. This suggests a natural direction for improving the interpretability of geometry-based surrogates: aggregate the exact activation cells into a small number of shared, interpretable regimes.

The methods most closely related to ours are DNN-guided mixture surrogates, particularly MLM (Mixture of Linear Models Co-supervised by DNNs)~\cite{Seo2022Mixture} and SEE-Net (Synced Explanation-Enhanced NN)~\cite{Seo2024Explainable}. Both methods cluster continuous hidden-layer representations of the DNN layer by layer, combine the resulting per-layer clusters through Cartesian products, and fit local linear models within the resulting regimes. Because these methods cluster hidden-layer values rather than activation patterns, inputs that share the same activation pattern may still be assigned to different regimes. SEE-Net further transforms the DNN into a pre-interpretable architecture, which can incur additional training cost. In contrast, our proposed \method uses the trained DNN only as a teacher for distillation and data augmentation, and directly clusters the binarized activation patterns of a WSSN. Consequently, its regimes are formed as unions of the WSSN's activation cells.

\subsection*{Overview of \method}
We propose \emph{Regime-Switching, Explainable, and Activation-Induced Linear Models (\method)}. Given a trained DNN teacher $f_{\mathrm{T}}$, \method first distills $f_{\mathrm{T}}$ into a WSSN $f_{\mathrm{S}}$ to match the teacher's soft predictions. This step is motivated by the fact that wide and shallow networks can faithfully mimic deep networks under knowledge distillation~\cite{Hinton2015Distilling, Bucilua2006Model, Ba2014Do}. As a ReLU network, the WSSN also partitions its input space into polyhedral activation cells within which it is affine. Its single hidden layer induces far fewer cells than the deep teacher~\cite{Montufar2014Number, Raghu2017Expressive, Serra2018Bounding}, and its activation patterns are lower-dimensional than the teacher's while remaining aligned with the teacher network. The reduced dimension stabilizes regime discovery since distance-based clustering becomes unreliable as the dimension grows~\cite{Aggarwal2001Surprising, Verleysen2005Curse}.

\method then clusters the WSSN's binarized hidden activation patterns into regimes and fits a regularized local linear expert within each regime. The final prediction is a weighted average of the expert outputs, with weights given by the posterior from the clustering step. To make interpretation more faithful, we employ an inverse temperature that sharpens the posterior toward hard routing, so that each prediction relies on only a few experts. To make regime assignment easy to interpret, we additionally train a regularized multiclass logistic regression, called the \emph{explanatory gate}, that predicts the regime assignment label directly from the input representation. In the basic \method formulation used for the main results in Section~\ref{sec:experiments}, the explanatory gate serves only to explain regime assignments rather than to determine routing. In Section~\ref{sec:experiments:ablation:gate_routing}, we also evaluate a variant, \method-EG, that uses the explanatory gate itself for routing and achieves performance comparable to that obtained by directly clustering the WSSN's binarized activation patterns.

\method achieves competitive predictive performance relative to SEE-Net across six datasets and to MLM and common inherently interpretable models on tabular data, while producing stable regime-specific explanations through its linear experts and explanatory gate.

\subsection*{Main Contributions}
\begin{itemize}
\item \textbf{Interpretable partitioning from ReLU networks' internal structure.}
We propose \method, an interpretable MoE framework that distills a trained DNN into a WSSN and clusters its low-dimensional, teacher-aligned ReLU activation patterns into a small set of regimes, yielding dual interpretability with respect to raw tabular or learned convolutional features: 1) local feature attribution within each regime; 2) an explanatory gate that attributes each regime assignment to these features.

\item \textbf{An idealized analysis of the complexity--stability trade-off.}
In an idealized setting, we analyze when regime assignments stay stable under input perturbations, and, assuming sparse within-regime coefficients, when LASSO recovers the correct support. This analysis reveals a trade-off between partition complexity and stability: finer partitions can improve approximation but may reduce regime-assignment stability.

\item \textbf{Practical pipeline with cross-modality validation.}
We benchmark \method against SEE-Net on both tabular and image tasks, and against MLM and other interpretable models on tabular tasks, demonstrating competitive accuracy with stable, regime-specific interpretations. Ablations show that 1) the router does not carry the predictive burden; 2) the explanatory gate can serve as the routing signal with little performance loss; 3) the results for different \(K\) values are mostly consistent with the complexity--stability trade-off.
\end{itemize}

The remainder of this paper is organized as follows. Section~\ref{sec:preliminary} reviews the piecewise-affine structure of ReLU networks. Section~\ref{sec:method} introduces \method and its learning algorithm. Section~\ref{sec:theory} illustrates a trade-off between partition complexity and stability by analyzing an idealized setting. Section~\ref{sec:experiments} reports experiments and ablation studies, and Section~\ref{sec:conclusion} concludes the paper.

\section{Preliminaries: Piecewise-Affine Structure of ReLU DNNs}\label{sec:preliminary}
Although a ReLU network is globally nonlinear, once the active or inactive status of every hidden unit is fixed, each activation pattern indexes a cell of the input space on which the network reduces to a single affine map~\cite{Hwang2020Un,Sudjianto2020Unwrapping,Picchiotti2022Clustering,Barua2026Mechanistic}. Thus, instead of learning an arbitrary partition through joint routing--prediction optimization, \method derives interpretable partitioning using activation patterns to discover the latent regimes and fits simple local models on the resulting regimes.

Consider a ReLU network with input $x\in\mathcal{X}=\mathbb{R}^p$, hidden layers $\ell=1,\dots,L$, and output dimension $O$:
\begin{equation*}
\begin{aligned}
    h^{(0)}(x) &= x, \\
    a^{(\ell)}(x) &= W^{(\ell)}h^{(\ell-1)}(x)+b^{(\ell)}, \\
    h^{(\ell)}(x) &= \rho\bigl(a^{(\ell)}(x)\bigr),
\end{aligned}
\end{equation*}
where $\rho(t)=\max\{0,t\}$. The output layer is
\begin{equation*}
    f(x)=W^{(L+1)}h^{(L)}(x)+b^{(L+1)}.
\end{equation*}
Here $n_0=p$ denotes the input dimension, $n_\ell$ is the width of hidden layer $\ell$, and $n_{L+1}=O$ is the output dimension. Thus, $h^{(\ell-1)}(x)\in\mathbb{R}^{n_{\ell-1}}$, $a^{(\ell)}(x),h^{(\ell)}(x),b^{(\ell)}\in\mathbb{R}^{n_\ell}$, $W^{(\ell)}\in\mathbb{R}^{n_\ell\times n_{\ell-1}}$, and $f(x),b^{(L+1)}\in\mathbb{R}^{O}$, $W^{(L+1)}\in\mathbb{R}^{O\times n_L}$.
Each hidden unit is either \emph{on} ($a^{(\ell)}_i(x) > 0$) or \emph{off} ($a^{(\ell)}_i(x) \le 0$). To track these states across the network, we use the \emph{activation pattern}~\cite{Raghu2017Expressive,Sudjianto2020Unwrapping}.

\begin{definition}[Activation Pattern {\cite{Sudjianto2020Unwrapping}}]\label{def:activation_pattern}
For each hidden layer, define the activation mask
\begin{equation}\label{eq:1}
\begin{aligned}
    s^{(\ell)}(x) = \mathbf{1}\bigl\{a^{(\ell)}(x)>0\bigr\}, \;
    D^{(\ell)}(x) = \operatorname{diag}\bigl(s^{(\ell)}(x)\bigr),
\end{aligned}
\end{equation}
where $s^{(\ell)}(x)\in\{0,1\}^{n_\ell}$ records which units in layer $\ell$ are active and $D^{(\ell)}(x)\in\mathbb{R}^{n_\ell\times n_\ell}$ is the corresponding diagonal selector matrix. Each ReLU layer can then be written as
\begin{equation*}
\begin{split}
    h^{(\ell)}(x) &= D^{(\ell)}(x)\,a^{(\ell)}(x) \\
    &= D^{(\ell)}(x)\bigl(W^{(\ell)}h^{(\ell-1)}(x)+b^{(\ell)}\bigr),
\end{split}
\end{equation*}
so the only source of nonlinearity is the variation of the masks $D^{(\ell)}(x)$. Once they are fixed, every layer becomes an affine map~\cite{Hwang2020Un,Sudjianto2020Unwrapping,Picchiotti2022Clustering,Barua2026Mechanistic}. Let
\begin{equation*}
s(x)=\bigl(s^{(1)}(x),\dots,s^{(L)}(x)\bigr)\in\{0,1\}^{\sum_{\ell=1}^L n_\ell}
\end{equation*}
denote the full activation pattern, and let
\begin{equation*}
\mathcal{S}_{\mathrm{exp}}\triangleq\{s(x):x\in\mathcal{X}\}
\end{equation*}
be the set of distinct patterns realized on \(\mathcal{X}\)~\cite{Sudjianto2020Unwrapping,Barua2026Mechanistic}.
\end{definition}

Since forward propagation assigns every input a unique pattern $s(x)$, each pattern $s$ selects an \emph{activation cell} $\mathcal{R}_s=\{x\in\mathcal{X}:s(x)=s\}$, an intersection of finitely many half-spaces and hence a convex polyhedral region. Since distinct patterns give disjoint cells, they partition the input space such that $\mathcal{X}=\bigsqcup_{s\in\mathcal{S}_{\mathrm{exp}}}\mathcal{R}_s$~\cite{Sudjianto2020Unwrapping,Montufar2014Number,Barua2026Mechanistic}. However, this exact partition $\mathcal{R}_s$ can be too fine-grained to support interpretable regime assignment, since the number of cells can grow rapidly with network depth~\cite{Montufar2014Number}. An upper bound on the number of cells is $\mathcal{O}(n_{\max}^{pL})$, where $n_{\max}=\max_{1\le \ell\le L} n_\ell$ denotes the maximum hidden-layer width~\cite{Raghu2017Expressive}. This fine partition may produce many tiny or even singleton cells~\cite{Montufar2014Number,Sudjianto2020Unwrapping,Picchiotti2022Clustering,Barua2026Mechanistic}. Practical local surrogate modeling therefore requires collapsing these cells into a small number of stable, interpretable regimes, each a union of activation cells on which the network is approximately affine. This motivates the activation-induced clustering of \method developed in Section~\ref{sec:method}.

\section{\method: Regime-Switching, Explainable, and Activation-Induced Linear Models}\label{sec:method}

\method has three stages: 1) distilling a DNN teacher $f_{\mathrm{T}}$ into a WSSN $f_{\mathrm{S}}$, which has teacher-aligned but lower-dimensional activation patterns; 2) discovering $K$ latent regimes by clustering the WSSN's activation patterns; 3) fitting a mixture of local experts together with a separate explanatory gate, making both feature attributions and regime assignments interpretable. A schematic illustration of \method is provided in Fig.~\ref{fig:actir}.

\begin{figure*}[!t]
\centering
\includegraphics[width=0.8\textwidth]{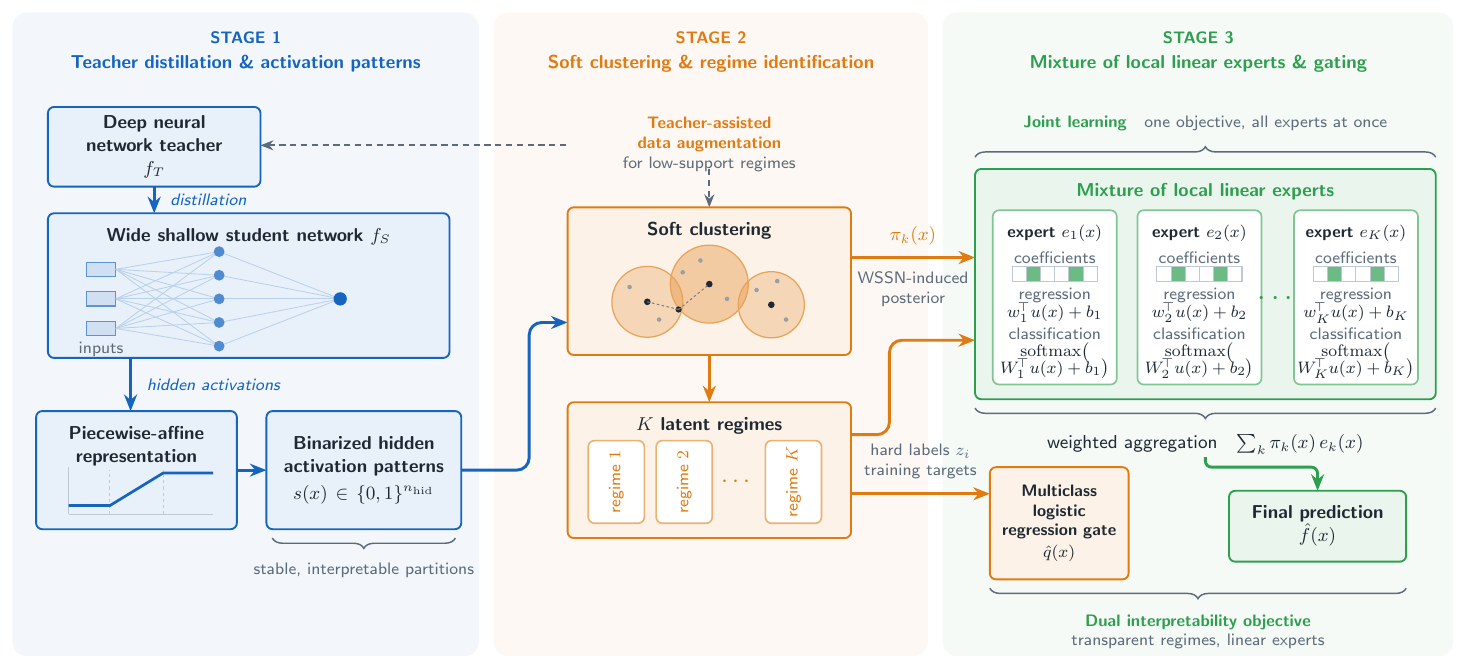}
\caption{Overview of the \method pipeline. Stage~1 distills a deep teacher $f_{\mathrm{T}}$ into the WSSN $f_{\mathrm{S}}$. Stage~2 clusters the WSSN's binarized activation patterns into $K$ regimes, giving the routing posterior $\pi_k(x)$. Stage~3 fits a mixture of local linear experts under $\pi_k(x)$ to produce the prediction $\hat f(x)$, and trains the explanatory gate $\hat q(x)$ to explain the assignments.}
\label{fig:actir}
\end{figure*}

\subsection*{\textbf{Stage 1: Teacher-to-Student Distillation}}

Let $u(x)$ denote the input to the student WSSN:
\begin{equation*}
    u(x) =
    \begin{cases}
        x, & \text{for tabular data } x \in \mathbb{R}^p, \\
        \varphi(x), & \text{for image data } x \in \mathbb{R}^{H \times W \times C}.
    \end{cases}
\end{equation*}
Here $\varphi(x)\in\mathbb{R}^p$ denotes the output of the WSSN's convolutional feature extractor. Throughout, we write $u(x)\in\mathbb{R}^p$. The extractor $\varphi$ is trained jointly with the WSSN during distillation and is frozen afterwards. Therefore, regime discovery, the local experts, and the explanatory gate all operate on the fixed representation $u(x)$. Note that for image data, local experts and the explanatory gate are linear in $u(x)$ rather than in the raw pixels.

We then pass $u(x)$ through the one-hidden-layer ReLU head, with pre-activation $a(x)=W_1 u(x)+b_1$ and hidden activation $h(x)=\rho(a(x))$, giving
\begin{equation*}
    f_{\mathrm{S}}(x) = W_2 h(x)+b_2.
\end{equation*}

We detail the construction of $\varphi(x)$ for each image dataset in the supplementary material. For both tabular and image data, the WSSN has a single hidden layer (excluding the feature-extraction layers for image data), so it induces only a manageable number of activation cells. For classification, $f_{\mathrm{S}}$ is trained by knowledge distillation from $f_{\mathrm{T}}$:
\begin{equation}\label{eq:2}
\begin{aligned}
\mathcal{L}_{\mathrm{KD}}^{\mathrm{cls}}
&=
\alpha T^2\,\mathcal{L}_{\mathrm{KL}}
\left(
\mathrm{softmax}\left(\frac{f_{\mathrm{T}}(x)}{T}\right)
\,\middle\|\,
\mathrm{softmax}\left(\frac{f_{\mathrm{S}}(x)}{T}\right)
\right)\\
&\quad +
(1-\alpha)\,
\mathcal{L}_{\mathrm{CE}}\bigl(f_{\mathrm{S}}(x),y\bigr),
\end{aligned}
\end{equation}
where $\mathcal{L}_{\mathrm{KL}}$ is the Kullback--Leibler divergence between the WSSN and the teacher distributions, and $\mathcal{L}_{\mathrm{CE}}$ is the cross-entropy against the label $y$. The temperature $T>0$ controls the softness of the softmax: $T>1$ smooths the output distribution and facilitates transfer of dark knowledge, while $T\to 0^+$ sharpens it toward a one-hot vector. The weight $\alpha\in[0,1]$ balances distillation and label-matching. For regression, we apply the mean-squared error:
\begin{equation*}
\mathcal{L}_{\mathrm{KD}}^{\mathrm{reg}}
=
\alpha\,\|f_{\mathrm{S}}(x)-f_{\mathrm{T}}(x)\|_2^2
+
(1-\alpha)\,\|f_{\mathrm{S}}(x)-y\|_2^2.
\end{equation*}

\subsection*{\textbf{Stage 2: Soft Clustering for Latent Regime Discovery}}
We freeze the WSSN after distillation, so its activation pattern $s(x)$ is fixed as defined in Definition~\ref{def:activation_pattern}. Therefore, the WSSN is an affine function of its input $u(x)$ within each exact activation cell. We then discover $K$ regimes by soft clustering with learnable centroids $\mu_1,\dots,\mu_K\in\mathbb{R}^{n_{\mathrm{hid}}}$, where $n_{\mathrm{hid}}$ is the hidden width of the WSSN, and the weighted distance:
\begin{equation*}
    D_k(x)
    =
    \frac{1}{n_{\mathrm{hid}}}\sum_{j=1}^{n_{\mathrm{hid}}}
    \left[1+(\omega-1)s_j(x)\right]\left(s_j(x)-\mu_{k,j}\right)^2,
\end{equation*}
where $\omega>1$ upweights mismatches on active units (weight $\omega$ when $s_j(x)=1$ and weight $1$ otherwise). The WSSN-induced posterior uses an inverse temperature $\beta_{\mathrm{clus}} > 0$:
\begin{equation}\label{eq:3}
    \pi_k(x)
    =
    \frac{\exp\{-\beta_{\mathrm{clus}}D_k(x)\}}
    {\sum_{j=1}^K \exp\{-\beta_{\mathrm{clus}}D_j(x)\}}.
\end{equation}
A small $\beta_{\mathrm{clus}}$ produces flatter regime assignments, whereas a large $\beta_{\mathrm{clus}}$ concentrates the assignments on a few regimes. Near-one-hot assignment enhances interpretability, since each prediction only draws from a few regimes and its feature attribution localizes to a small set of experts. To avoid imposing near-one-hot assignments at the initialization of the clustering algorithm, we progressively increase the inverse temperature according to
\begin{equation}\label{eq:4}
    \beta_{\mathrm{clus}}^{(t+1)}
    =
    \min\left\{\eta\,\beta_{\mathrm{clus}}^{(t)},\,\beta_{\max}\right\},
\end{equation}
where $t$ indexes clustering epochs, $\eta>1$ is the sharpening factor, and $\beta_{\max}$ is the upper bound on the inverse temperature. For a fixed $K$, the centroids are learned by minimizing
\begin{equation*}
    \mathcal{L}_{\mathrm{clus}}
    =
    \frac{1}{n}\sum_{i=1}^n\sum_{k=1}^K \pi_k(x_i)D_k(x_i),
\end{equation*}
where centroids are initialized from $K$ randomly chosen training activation patterns. We update the centroids by differentiating through both $D_k(x_i)$ and $\pi_k(x_i)$.

Define the regime assignment as $z(x)=\argmax_k \pi_k(x)$. Thus, each \emph{regime} $\mathcal{C}_k=\{x: z(x)=k\}$ is a union of activation cells on which the WSSN is piecewise-affine. Since clustering merges cells with similar activation patterns, the WSSN remains approximately affine within each regime, as supported empirically by the strong performance in Section~\ref{sec:experiments}. Because the WSSN is frozen, clustering merely summarizes how the distilled WSSN organizes the input space rather than absorbing the predictive burden into the router. This is shown in the ablation study in Section~\ref{sec:experiments:ablation:heavy_lifting}.

\subsection*{\textbf{Stage 3: Fixed-Router Mixture of Regularized Local Experts}}
\medskip
\noindent\textbf{Local Experts.}
With the posterior $\pi_k(x)$ fixed, we train a mixture of local linear experts. Each regime-$k$ expert is:
\begin{equation}\label{eq:5}
    e_k(x)=
    \begin{cases}
        \mathrm{softmax}(W_k^\top u(x)+b_k), & \text{classification}, \\
        w_k^\top u(x)+b_k, & \text{regression}.
    \end{cases}
\end{equation}
For $O$-class classification, $W_k\in\mathbb{R}^{p\times O}$ and $b_k\in\mathbb{R}^{O}$, whereas for regression $w_k\in\mathbb{R}^{p}$ and $b_k\in\mathbb{R}$. The final prediction aggregates the experts under the posterior $\pi_k(x)$:
\begin{equation}\label{eq:6}
    \hat f(x)=\sum_{k=1}^{K}\pi_k(x)\,e_k(x).
\end{equation}
To optimize the expert parameters, let $\mathcal{L}_y$ be the predictive loss of the fixed-router mixture, with
\begin{equation*}
\mathcal{L}_y=
\begin{cases}
\frac{1}{n}\sum_{i=1}^n\mathcal{L}_{\mathrm{CE}}\bigl(\hat f(x_i),y_i\bigr), & \text{classification},\\
\frac{1}{n}\sum_{i=1}^n\mathcal{L}_{\mathrm{MSE}}\bigl(\hat f(x_i),y_i\bigr), & \text{regression}.
\end{cases}
\end{equation*}
We estimate the expert parameters through
\begin{equation}\label{eq:7}
\min_{\theta_y}\quad
\mathcal{L}_y(\theta_y) + \lambda_y \Omega_y(\theta_y),
\end{equation}
where $\theta_y$ collects the local linear experts' parameters in Eq.~\eqref{eq:5} across all $K$ regimes,
\begin{equation*}
    \theta_y=
    \begin{cases}
        \{(w_k,b_k)\}_{k=1}^{K}, & \text{regression},\\
        \{(W_k,b_k)\}_{k=1}^{K}, & \text{classification},
    \end{cases}
\end{equation*}
and $\Omega_y$ penalizes only the coefficients,
\begin{equation*}
\Omega_y(\theta_y)=
    \begin{cases}
        \displaystyle\sum_{k=1}^{K}\sum_{j=1}^{p}|w_{k,j}|, & \text{regression},\\[6pt]
        \displaystyle\sum_{k=1}^{K}\sum_{j=1}^{p}\sum_{o=1}^{O}|(W_k)_{jo}|, & \text{classification}.
    \end{cases}
\end{equation*}

\medskip

\noindent\textbf{Explanatory Gate.}
The posterior $\pi_k(x)$ routes inputs but does not explain \emph{why} an input is assigned to a regime. To make the assignment rule interpretable, we train a separate multiclass logistic regression to predict the hard regime label $z_i=\argmax_k \pi_k(x_i)$:
\begin{equation}\label{eq:8}
\hat q(x)=\mathrm{softmax}(\Gamma u(x)+c),
\end{equation}
where $\Gamma=[\gamma_1,\dots,\gamma_K]^\top \in \mathbb{R}^{K\times p}$, $\gamma_k\in\mathbb{R}^{p}$, and $c\in\mathbb{R}^{K}$ are trainable gate parameters. We collect the gate parameters as $\theta_g=\{\Gamma,c\}$. Since a common shift of all $\gamma_k$ leaves $\hat q$ unchanged, coefficients are read as contrasts between regimes. The explanatory gate is trained using the cross-entropy loss
\[
\mathcal{L}_g(\theta_g)
=
\frac{1}{n}\sum_{i=1}^n
\mathcal{L}_{\mathrm{CE}}\bigl(\hat q(x_i),z_i\bigr).
\]
We estimate the gate parameters through
\begin{equation}\label{eq:9}
\min_{\theta_g}\quad
\mathcal{L}_g(\theta_g) + \lambda_g \|\Gamma\|_1,
\end{equation}
where $\|\Gamma\|_1=\sum_{k=1}^{K}\sum_{j=1}^{p}|\Gamma_{kj}|$ denotes the entrywise $\ell_1$ norm.

Since both the experts and the explanatory gate are linear models, \method yields dual interpretability: 1) the expert coefficients explain \emph{how} each prediction is formed; 2) the gate coefficients explain \emph{why} a sample is assigned to its regime.

\medskip

\noindent\textbf{Teacher-Assisted Data Augmentation.}
To stabilize expert estimation in low-support regimes, we augment small regimes with teacher-labeled synthetic samples, giving the DNN a second role as a source of supervision beyond distillation. Let $\mathcal{I}_k=\{i:z(x_i)=k\}$ denote the set of training samples assigned to regime \(k\), and let $n_k=|\mathcal{I}_k|$ be the number of original samples in regime \(k\). If $n_k<\tau$ for a predefined minimum regime size $\tau$, we generate $n_k^{\mathrm{aug}}=\tau-n_k$ synthetic samples. We split the variables into a continuous block $\mathcal{J}_{\mathrm{cont}}$ and a categorical (one-hot) block $\mathcal{J}_{\mathrm{cat}}$:
\begin{equation}\label{eq:10}
\begin{aligned}
x^{\mathrm{aug}}_{k,j}[\mathcal{J}_{\mathrm{cont}}] &= \bar{\mu}_k + \epsilon\,\sigma_k \odot \xi_{j}, \quad \xi_{j}\sim\mathcal{N}(0,I_{|\mathcal{J}_{\mathrm{cont}}|}),\\
x^{\mathrm{aug}}_{k,j}[\mathcal{J}_{\mathrm{cat}}] &= x_{i_j}[\mathcal{J}_{\mathrm{cat}}], \quad i_j\sim\mathrm{Unif}(\mathcal{I}_k),
\end{aligned}
\end{equation}
where $\bar{\mu}_k,\sigma_k\in\mathbb{R}^{|\mathcal{J}_{\mathrm{cont}}|}$ are the within-regime mean and standard deviation of the continuous variables, and $\epsilon>0$ controls the variance magnitude. Continuous variables are thus simulated around the regime mean, while categorical variables are copied from a uniformly sampled real member of the regime to preserve valid encodings. Each synthetic sample is labeled by the teacher,
\begin{equation*}
    y^{\mathrm{aug}}_{k,j}=
    \begin{cases}
        \argmax_{c} f_{\mathrm{T}}(x^{\mathrm{aug}}_{k,j})_c, & \text{classification},\\[1mm]
        f_{\mathrm{T}}(x^{\mathrm{aug}}_{k,j}), & \text{regression},
    \end{cases}
\end{equation*}
and appended to the training set. The experts and explanatory gate are then estimated through the objectives in Eqs.~\eqref{eq:7} and~\eqref{eq:9}, respectively, with the augmented samples routed by the same posterior $\pi_k$. Samples generated for regime $k$ may be assigned to another regime. Regimes with $n_k=0$ are skipped. We augment only the tabular datasets because every regime exceeds the minimum size $\tau$ on the image datasets.

\medskip

\noindent\textbf{\method-EG.}
In \method, predictions are weighted by the WSSN-induced posterior $\pi_k(x)$ in Eq.~\eqref{eq:3}. The explanatory gate $\hat q(x)$ in Eq.~\eqref{eq:8} explains these assignments post hoc but never routes inputs in the main results in Section~\ref{sec:experiments}. Promoting the gate to the router yields a variant we call \method-EG (EG for explanatory gate). \method-EG is then a \emph{two-layer linear model} on the representation $u(x)$: the gate reassigns every sample,
\begin{equation*}
z_i^{g}=\argmax_k \hat q_k(x_i),
\qquad
\mathcal{I}_k^{g}=\{i: z_i^{g}=k\},
\end{equation*}
and predictions are mixed by the gate posterior,
\begin{equation*}
\hat f_{\mathrm{EG}}(x)
=\sum_{k=1}^{K}\hat q_k(x)\,e_k^{g}(x),
\end{equation*}
which is Eq.~\eqref{eq:6} with $\hat q_k$ in place of $\pi_k$. The local experts $e_k^{g}$ are then refitted on all training samples by solving the same objective in Eq.~\eqref{eq:7} with $\hat f(x_i)$ replaced by $\hat f_{\mathrm{EG}}(x_i)$. The hard labels $z_i^{g}$ thus define the regimes $\mathcal{I}_k^{g}$ of \method-EG, whereas the refit uses the same soft mixture as \method, so $\hat q_k(x)$ weights the experts both during training and at test time. Both the assignment rule and the within-regime predictions of \method-EG are then directly readable from linear coefficients. The ablation study in Section~\ref{sec:experiments:ablation:gate_routing} shows that \method-EG performs competitively with \method.

\section{Theoretical Analysis of Structural Stability and Support Recovery}\label{sec:theory}

We analyze the joint reliability of regime assignment and sparse support recovery in \method under an idealized setting where the data follow a latent piecewise sparse linear model whose regimes are separated by a linear winner-take-all regime selector. Here $x$ represents $u(x)$, the winner-take-all rule is a linear abstraction of regime assignment, and $\lambda$ controls within-regime sparsity. The setting is idealized in three ways: 1)~the regimes are treated as known; 2)~the model is exactly linear within each of them, whereas \method discovers regimes by clustering activation patterns and is only approximately linear inside each regime; 3)~the true coefficients are sparse, whereas the fitted experts need not be. We allow the test-time regime selector to be learned, with error bounded by $\eta_n$, while retaining true regime membership for expert training. A recurring theme of the analysis is to make explicit how the number of regimes $K$ affects both stability and overall support recovery. Note that this analysis also assumes a regression model with additive Gaussian noise.

\subsection{Model Setup}\label{sec:theory:model_setup}

Let $x \in \mathcal X=\mathbb{R}^p$ denote the input signal and suppose the data are generated from a latent piecewise sparse linear model:
\begin{equation*}
y = x^\top \beta_{k^*}^* + \varepsilon,
\end{equation*}
where $k^* = g^*(x)$ denotes the true regime, $\beta_k^* \in \mathbb{R}^p$ are sparse vectors with support
\begin{equation*}
S_k = \operatorname{supp}(\beta_k^*), \quad |S_k| = s_k,
\end{equation*}
and $\varepsilon \sim \mathcal{N}(0,\sigma^2)$.

The true regime assignment is determined by a linear regime selector $g^*$ with gating vectors $\{v_k^*\}_{k=1}^{K}\subset\mathbb{R}^p$, specified in Assumptions~\ref{ass:gating_setup} and~\ref{ass:gate_norm}. For simplicity, the idealized selector omits intercepts.

Define the pairwise margin:
\begin{equation}\label{eq:11}
m(x) = \min_{j \neq k^*} (v_{k^*}^* - v_j^*)^\top x.
\end{equation}
We analyze stability at inputs with $m(x)>0$. This margin quantifies how much the winning regime's gating score exceeds that of its closest competitor.

\subsection{Stability of Regime Assignment}\label{sec:theory:stability}

We study how reliably the true regime selector preserves its decision under additive perturbations of a fixed input signal. Throughout this subsection we treat $x$ as a deterministic signal and take all probabilities with respect to the perturbation $\delta$ alone. For the probabilistic bounds, we assume $K\ge2$ and $\sigma_g>0$.

\begin{assumption}[Linear Gating, Deterministic Signal]\label{ass:gating_setup}
The true regime selector is the linear winner-take-all rule
\begin{equation}\label{eq:12}
g^*(x) = \argmax_{k \in \{1,\dots,K\}} v_k^{*\top} x,
\end{equation}
with gating vectors $\{v_k^*\}_{k=1}^K$ fixed.
\end{assumption}

\begin{assumption}[Additive Isotropic Perturbation]\label{ass:additive_noise}
We observe $x+\delta$ where $\delta\sim\mathcal{N}(0,\sigma_g^2 I_p)$ is independent of $x$.
\end{assumption}

\begin{proposition}[Deterministic Stability (fixed $x$)]\label{prop:det_stability}
Suppose Assumption~\ref{ass:gating_setup} holds. Fix $x$ with $k^*=g^*(x)$ and margin $m(x)>0$. For any additive perturbation $\delta\in\mathbb{R}^p$, if
\begin{equation*}
\|\delta\|_2 < \frac{m(x)}{\max_{j\neq k^*}\|v_{k^*}^*-v_j^*\|_2},
\end{equation*}
then $g^*(x+\delta)=g^*(x)$.
\end{proposition}

\begin{proof}
For every $j\neq k^*$, the Cauchy--Schwarz inequality gives
\begin{equation*}
\begin{aligned}
&(v_{k^*}^*-v_j^*)^\top(x+\delta) \\
&\qquad\ge m(x)-\|v_{k^*}^*-v_j^*\|_2\,\|\delta\|_2>0.
\end{aligned}
\end{equation*}
Hence $g^*(x+\delta)=k^*$.
\end{proof}

We now strengthen this deterministic result to a probabilistic stability bound under Gaussian perturbations. To make the dependence on the number of regimes $K$ explicit, we additionally assume bounded gating energy.

\begin{assumption}[Bounded Gating Energy]\label{ass:gate_norm}
$\|v_k^*\|_2 \le B_v$ for all $k\in\{1,\dots,K\}$.
\end{assumption}

Fix $x$ and let $k^*=g^*(x)$. For each $j\neq k^*$, define the gating-vector difference $a_j\triangleq v_j^*-v_{k^*}^*$. Define the Gaussian-max statistic (random in $\delta$) as $M_{k^*}(\delta) \triangleq \max_{j\neq k^*} a_j^\top \delta$. A misassignment occurs only if some competitor's perturbed gain exceeds the margin:
\begin{equation*}
\{g^*(x+\delta)\neq k^*\}\subseteq \Big\{M_{k^*}(\delta) \ge m(x)\Big\}.
\end{equation*}

The following result bounds this misassignment probability via Gaussian-max concentration~\cite{Vershynin2018High}.

\begin{proposition}[Probabilistic Stability Bound via Gaussian-Max Concentration]
\label{prop:sharp_stability}
Under Assumptions~\ref{ass:gating_setup}--\ref{ass:gate_norm}, the misassignment probability (with respect to $\delta$) satisfies
\begin{equation}\label{eq:13}
\mathbb{P}_\delta\!\big(g^*(x+\delta)\neq k^* \mid x\big)
\le
\exp\!\left(
-\frac{\big(m(x)-\mathbb{E}_\delta M_{k^*}(\delta)\big)_+^2}
{2\,\sigma_g^2 \max_{j\neq k^*}\|a_j\|_2^2}
\right),
\end{equation}
where $(u)_+ \triangleq \max\{u,0\}$ and
$\mathbb{E}_\delta$ denotes expectation with respect to
$\delta\sim\mathcal{N}(0,\sigma_g^2 I_p)$.

Moreover, using
\begin{equation*}
\mathbb{E}_\delta M_{k^*}(\delta)
\le
\sigma_g \max_{j\neq k^*}\|a_j\|_2 \sqrt{2\log(K-1)},
\end{equation*}
and $\|a_j\|_2 \le 2B_v$, we obtain the explicit bound
\begin{multline}\label{eq:14}
\mathbb{P}_\delta\!\big(g^*(x+\delta)\neq k^* \mid x\big)\\
\le
\exp\!\left(
-\frac{1}{2}
\Big(
\frac{m(x)}{2\sigma_g B_v}
-
\sqrt{2\log(K\!-\!1)}
\Big)_+^2
\right).
\end{multline}
For any prescribed margin threshold $m_{\min}>0$, Eq.~\eqref{eq:14} holds uniformly for all $x\in\mathcal X$ satisfying $m(x)\ge m_{\min}$, with $m(x)$ replaced by $m_{\min}$, since the bound is decreasing in $m(x)$.
\end{proposition}

\begin{proof}
The proof is given in the supplementary material.
\end{proof}

Three factors compete in the explicit bound in Eq.~\eqref{eq:14}: a larger margin $m(x)$ promotes stability, a higher noise level $\sigma_g$ erodes it, and the regime-complexity term $\sqrt{2\log(K-1)}$ imposes an additional cost that grows only slowly with $K$.

\begin{remark}[When the bound is nonvacuous]
The explicit bound in Proposition~\ref{prop:sharp_stability} yields a non-trivial guarantee (i.e., strictly less than $1$) only when
\begin{equation*}
m(x) > 2\sigma_g B_v\sqrt{2\log(K-1)}.
\end{equation*}
This condition requires the margin to exceed an upper bound on the expected maximum Gaussian fluctuation. Otherwise, Eq.~\eqref{eq:14} becomes vacuous.
\end{remark}

\subsection{Joint Regime and Support Recovery with Learning Error}
\label{sec:theory:joint}

We study support recovery with a learned regime selector and oracle-trained experts. This isolates regime uncertainty and does not account for the effect of assignment errors on expert training. All probabilities in this support-recovery analysis are conditional on the training inputs.

Let $\mathcal{D}_n$ denote the training data, independent of the test perturbation $\delta$. Let $\hat g$ be a learned regime selector constructed from $\mathcal{D}_n$, and define $\hat k = \hat g(x+\delta)$ at test time. For each regime $k$, define the oracle index set
\begin{equation*}
\mathcal I_k^*=\{i:g^*(x_i)=k\}, \qquad n_k=|\mathcal I_k^*|.
\end{equation*}
Let $\hat\beta_k$ be the LASSO estimator fitted using the samples indexed by $\mathcal I_k^*$, and let $\hat S_k = \operatorname{supp}(\hat\beta_k)$. These oracle index sets are unavailable to the practical algorithm because $g^*$ is unknown. We assume that observations in regime $k$ follow
\begin{equation*}
y = x^\top \beta_k^* + \varepsilon,
\qquad
\varepsilon \sim \mathcal{N}(0,\sigma^2),
\end{equation*}
independently across samples, with $\sigma>0$.

To bring classical LASSO support-recovery guarantees to bear within each regime, we impose some standard conditions on the within-regime design and signal: a restricted eigenvalue condition (Assumption~\ref{ass:re}), an irrepresentable condition with incoherence parameter $\gamma_k\in(0,1]$ (Assumption~\ref{ass:irrep}), a beta-min condition (Assumption~\ref{ass:betamin}), and the penalty choice $\lambda=A_\lambda\sigma\sqrt{\log p/n_{k^*}}$ with $A_\lambda\ge4/\gamma_{k^*}$, where $\gamma_{k^*}$ is the incoherence parameter of the true regime $k^*$. Under these conditions, the LASSO result for given training inputs~\cite{Wainwright2009Sharp} guarantees support recovery within regime $k^*$ with probability at least $1-c_1\exp(-c_2 n_{k^*}\lambda^2/\sigma^2)$ for some constants $c_1,c_2>0$. We state the corresponding result as Theorem~\ref{thm:wainwright} in the supplementary material.

Since the regime selector is itself estimated, misassignment of the regime is a third source of error, alongside LASSO failure within the correct regime and the perturbation-driven switching of Section~\ref{sec:theory:stability}. We assume that the learned selector $\hat g$ satisfies $\sup_{x\in\mathcal{X}}\mathbb{P}_{\mathcal{D}_n}\big(\hat g(x)\neq g^*(x)\big)\le\eta_n$ for a sequence $\eta_n\to 0$ (Assumption~\ref{ass:regime_learning}). All of these conditions are stated formally in Section~\ref{sec:supp:assumptions} of the supplementary material. Combining the three sources yields an end-to-end guarantee for oracle-trained experts. For fixed $x$ as in Section~\ref{sec:theory:stability}, under Assumptions~\ref{ass:gating_setup}--\ref{ass:gate_norm} and~\ref{ass:re}--\ref{ass:regime_learning} and the stated penalty choice, the probability that the support recovered in the predicted regime differs from the true support satisfies
\begin{equation}\label{eq:15}
\begin{aligned}
\mathbb{P}\big(\hat S_{\hat k} \neq S_{k^*}\big)
&\le
\underbrace{\eta_n}_{\text{regime learning}}
+
\underbrace{c_1\exp\!\left(-c_2 n_{k^*}\lambda^2/\sigma^2\right)}_{\text{support recovery}} \\
&\quad+
\underbrace{\exp\!\left(
-\frac{1}{2}
\Big(
\frac{m(x)}{2\sigma_g B_v}
-
\sqrt{2\log(K-1)}
\Big)_+^2
\right)}_{\text{perturbation stability}}.
\end{aligned}
\end{equation}

The formal statement and proof are given in the supplementary material as Proposition~\ref{prop:joint_full}. Section~\ref{sec:theory:design_implications} explains the implications of this bound for the choice of $K$.

\subsection{Design Implications and Structural Trade-Offs}
\label{sec:theory:design_implications}
\medskip

\noindent\textbf{Regime Learning Accuracy.}
The term $\eta_n$ captures the generalization error of the learned regime selector $\hat g$. Assumption~\ref{ass:regime_learning} abstracts the regime-identification performance of $\hat g$ through this single quantity. Since a full analysis of the regime-learning procedure lies beyond our scope, we instead take $\eta_n$ as given and combine it additively with the switching probability and the within-regime recovery error. This yields an end-to-end reliability bound. A smaller $\eta_n$ tightens the overall support-recovery bound.

\medskip

\noindent\textbf{Perturbation Stability.}
The perturbation-driven instability term is controlled by the margin $m(x)$ relative to the noise-complexity radius $2\sigma_g B_v \sqrt{2\log(K-1)}$. This exposes a trade-off: increasing the number of regimes $K$ raises approximation capacity, but it also enlarges the noise-complexity radius at rate $\sqrt{\log K}$, thereby weakening the stability guarantee. This comparison fixes $m(x)$, $\sigma_g$, and $B_v$. In \method, changing $K$ may also alter margins and improve approximation, but these effects are not captured by the fixed-regime analysis.

\medskip

\noindent\textbf{Per-Regime Sample Complexity.}
Recall from Section~\ref{sec:theory:joint} that $\lambda \asymp \sigma\sqrt{\log p/n_{k^*}}$. With balanced regimes ($n_{k^*}\approx n/K$) and all else fixed, a larger $K$ tightens the beta-min requirement:
\begin{equation*}
\min_{j\in S_{k^*}} |\beta_{k^*,j}^*|
\gtrsim
\sigma \sqrt{\frac{K\log p}{n}}.
\end{equation*}
Consequently, a larger $K$ raises the signal strength required by the support-recovery guarantee.

\medskip

\noindent\textbf{Complexity--Stability Trade-Off.}
Overall, these effects make $K$ a single knob that trades representational capacity against statistical reliability. A larger $K$ allows a finer partition and can reduce approximation bias when the regression function is smooth, but it also shrinks the per-regime sample size and can increase switching instability, potentially increasing estimation variability and making the support-recovery conditions harder to satisfy.

Under balanced sampling, we use the sample-size heuristic
\begin{equation}\label{eq:16}
\frac{n}{K} \gg s_{\max}\log p,
\end{equation}
where $s_{\max}=\max_{1\le k\le K}|S_k|$. Together with reliable regime learning and adequate margins, this heuristic motivates a moderate number of well-separated regimes to balance approximation accuracy and explanation reliability. We examine this qualitative trade-off in the ablation on $K$ in Section~\ref{sec:experiments:ablation:num_regimes}.

\makeatletter
\immediate\write\@auxout{\string\newlabel{count:assumption}{{\arabic{assumption}}{}{}{}{}}}%
\immediate\write\@auxout{\string\newlabel{count:proposition}{{\arabic{proposition}}{}{}{}{}}}%
\immediate\write\@auxout{\string\newlabel{count:theorem}{{\arabic{theorem}}{}{}{}{}}}%
\makeatother

\section{Experiments}\label{sec:experiments}

We evaluate \method on six datasets, comparing it with SEE-Net~\cite{Seo2024Explainable}, MLM~\cite{Seo2022Mixture}, other inherently interpretable models, and standard MLP/CNN models. MLM and the inherently interpretable models are evaluated only on tabular datasets. SEE-Net (Pre-intp) and SEE-Net (Intp) denote SEE-Net's pre-interpretable NN, which predicts through the guidance NN, and its final interpretable NN, respectively. MLM-CELL is the cell-level mixture obtained from layer-wise clustering, whereas MLM-EPIC merges cells into larger Explainable Prediction-induced Input Clusters (EPIC). SEE-Net (Intp) and MLM-EPIC are the interpretable models intended for deployment in practice.

The six datasets comprise three tabular datasets (Bike Sharing, California Housing, Covertype) and three image datasets (MNIST, Flowers, CIFAR-10). We detail the full dataset descriptions, implementation and training details, and additional per-dataset interpretability analyses in the supplementary material.

In the main experiments, \method routes inputs using the WSSN-induced posterior $\pi_k(x)$ rather than the explanatory gate, so that the experiments examine the usefulness of the WSSN activation patterns themselves. Section~\ref{sec:experiments:ablation:gate_routing} examines \method-EG, the derived two-layer linear model in which the explanatory gate takes over routing.

\subsection{Synthetic Data Study}\label{sec:experiments:synthetic}
To illustrate the regime-recovery behavior of \method, we study a synthetic dataset with known ground-truth regimes. We draw $n=2{,}000$ points uniformly over the square $[-1,1]^2$ with independent covariates $X_1,X_2$, and let the response $Y$ follow a piecewise-linear trapezoidal-bowl surface with an affine interior and additive Gaussian noise. Let $r=\max(|X_1|,|X_2|)$. Then
\begin{equation*}
Y=
\begin{cases}
\beta_0+\beta_1 X_1+\beta_2 X_2+\varepsilon & \text{if } r\le t,\\
X_2 - t + \varepsilon & \text{if } r> t,\ X_2\ge|X_1|,\\
-X_2 - t + \varepsilon & \text{if } r> t,\ X_2\le-|X_1|,\\
X_1 - t + \varepsilon & \text{if } r> t,\ X_1>|X_2|,\\
-X_1 - t + \varepsilon & \text{if } r> t,\ X_1<-|X_2|,
\end{cases}
\end{equation*}
where \(t=0.5\), \(\varepsilon\sim\mathcal{N}(0,0.02^2)\) is the Gaussian noise, and the center regime uses \(\beta_0=0.5\), \(\beta_1=0.5\), \(\beta_2=-0.5\). For this experiment we fix the final number of regimes at \(K=5\) for \method, MLM, and SEE-Net. Since each regime is linear in the inputs, the optimal decomposition follows the edges of the trapezoidal bowl. Regime boundaries are marked by jumps or changes in slope. Fig.~\ref{fig:1} compares MLM, SEE-Net, and \method\ on the same synthetic draw. MLM in Fig.~\ref{fig:1}\subref{fig:1a} misses almost the entire right and middle regions, and SEE-Net in Fig.~\ref{fig:1}\subref{fig:1b} cuts across the true boundaries and recovers the regimes only partially, whereas \method\ in Fig.~\ref{fig:1}\subref{fig:1c} aligns most closely with the true geometry.
\begin{figure}[!t]
\centering
\subfloat[MLM]{\includegraphics[width=0.32\linewidth]{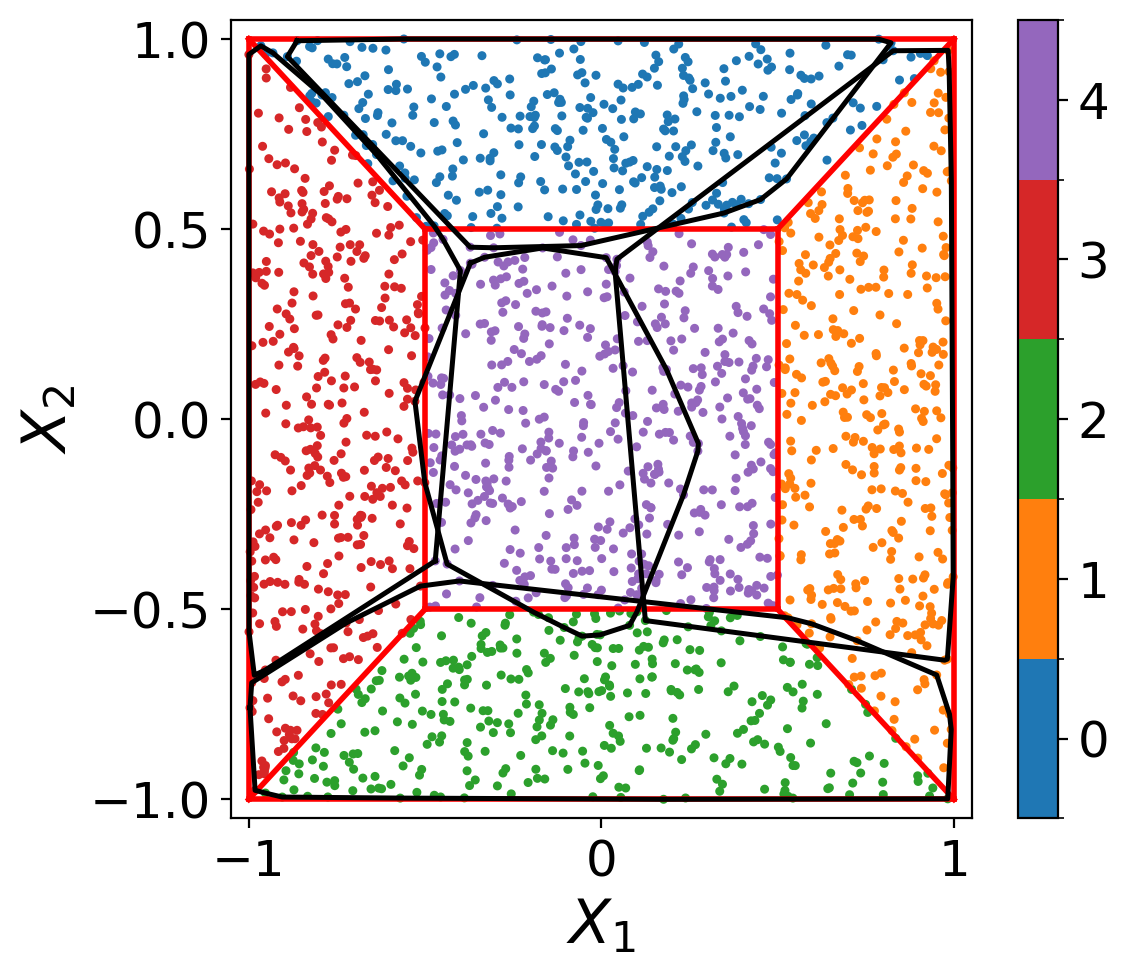}\label{fig:1a}}
\hfill
\subfloat[SEE-Net]{\includegraphics[width=0.32\linewidth]{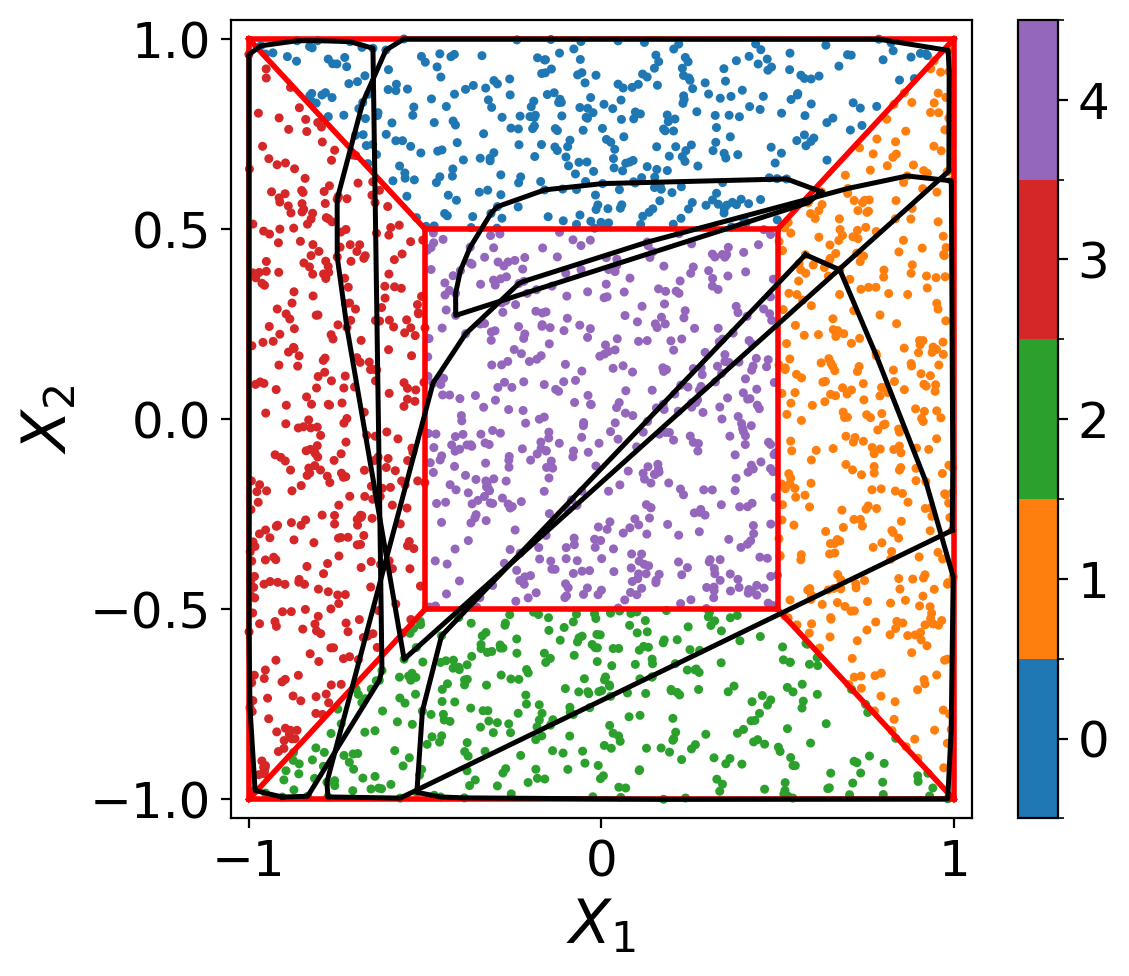}\label{fig:1b}}
\hfill
\subfloat[\method]{\includegraphics[width=0.32\linewidth]{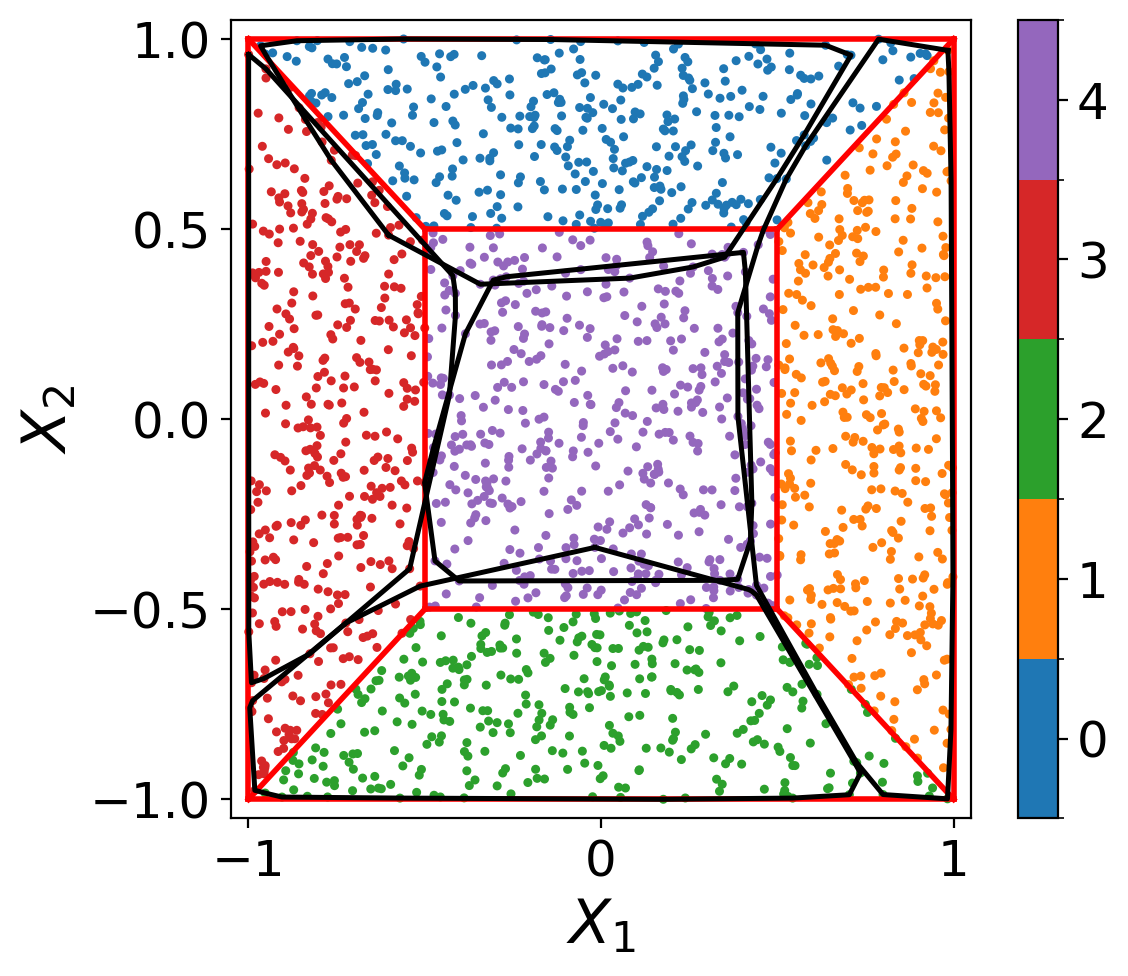}\label{fig:1c}}
\caption{Predicted regime assignments for MLM, SEE-Net, and \method. Ground-truth boundaries in red, predicted boundaries as black outlines.}
\label{fig:1}
\end{figure}

\subsection{Real-World Data Study}\label{sec:experiments:realworld}
\noindent\textbf{Prediction Performance.}\label{sec:experiments:realworld:prediction}
Tables~\ref{tab:1} and~\ref{tab:2} report the predictive performance of \method against inherently interpretable models, SEE-Net, MLM, and MLP/CNN models, where applicable. We also report Gate AUC, which measures how accurately the explanatory gate predicts the WSSN-induced regime assignments. 

Table~\ref{tab:1} shows that \method is the strongest interpretable model on all tabular datasets, outperforming SEE-Net (Intp), MLM-CELL, MLM-EPIC, and the inherently interpretable models. The MLM baselines exhibit substantial variation on Bike Sharing, suggesting instability under our experimental setup. Strong black-box models such as MLP and XGBoost still perform better than \method, but \method remains simple to read because the routing posterior is near one-hot, so each prediction can be explained mostly by a single linear expert, as Fig.~\ref{fig:tab:route}\subref{fig:cov}--\subref{fig:cal} in the supplementary material indicates.

In Table~\ref{tab:2}, \method nearly matches the teacher CNN on MNIST. On Flowers, the distilled WSSN achieves a higher AUC than the teacher, illustrating a potential predictive benefit of the distillation stage. \method retains part of this gain while providing regime-level explanations. Overall, \method improves over SEE-Net~(Intp) on all three image datasets. This shows that \method's performance advantage carries over from tabular features to high-dimensional images. More importantly, these image datasets demonstrate that the interpretability benefits come at little cost in accuracy, although a gap to the strongest black-box models remains on the tabular datasets. \method is also stable across runs since the standard deviations are small and only larger on harder datasets such as California Housing and Flowers. This consistency indicates that the distillation, clustering, and fixed-router fitting pipeline yields robust final predictions. 

The gate-recovery results in both Tables~\ref{tab:1} and~\ref{tab:2} provide a complementary interpretability view. The WSSN-induced regime labels are recovered almost perfectly by the linear gate on Bike Sharing, Covertype, MNIST, and CIFAR-10, suggesting that the switching logic is both effective and explainable.

The lower Gate AUC on California Housing reflects the fact that a larger, more fragmented partition is harder to explain with a simple linear assignment rule, as this dataset uses the largest partition, $K=150$, among all datasets. Among the image datasets, Flowers has the lowest Gate AUC, suggesting that its regime assignments are harder to recover with a linear gate. Nevertheless, its Gate AUC of 0.974 remains high, indicating that the assignments are still largely recoverable from the learned representation.
\begin{table*}[t]
\begingroup
\setlength{\tabcolsep}{2.8pt}
\renewcommand{\arraystretch}{1.08}
\caption{Prediction performance on the three tabular datasets, reported as mean (std) over five runs. Gate AUC measures how accurately the linear gate recovers the discovered regime assignments. We fit the same explanatory gate to the regimes of SEE-Net and the MLM variants, so the column is comparable across all regime-switching methods. All AUC values are macro-averaged one-vs-rest over classes or regimes. Empty cell means not applicable. Bold black text marks the best result for performance column (excluding SEE-Net (Pre-intp), WSSN, and Gate AUC); bold dark blue, the best interpretable model.}
\label{tab:1}
\begin{center}
\begin{adjustbox}{max width=\textwidth}
\begin{tabular}{l cc cc cc}
\hline
\rowcolor{gray!18}
 & \multicolumn{2}{c}{Bike Sharing}
 & \multicolumn{2}{c}{Cal.\ Housing}
 & \multicolumn{2}{c}{Covertype} \\
\cline{2-7}
\rowcolor{gray!18}
\multirow{-2}{*}{Method}
 & Test RMSE $\downarrow$ & Test Gate AUC $\uparrow$
 & Test RMSE $\downarrow$ & Test Gate AUC $\uparrow$
 & Test AUC $\uparrow$ & Test Gate AUC $\uparrow$ \\
\hline
Logistic Reg. &  &  &  &  & 0.937 (0.000) &  \\
SVM &  &  &  &  & 0.925 (0.000) &  \\
LASSO Reg. & 203.485 (0.000) &  & 0.830 (0.008) &  &  &  \\
GAM & 241.067 (0.000) &  & 1.050 (0.009) &  & 0.943 (0.000) &  \\
Random Forest & 94.510 (0.456) &  & 0.536 (0.007) &  & 0.956 (0.001) &  \\
SVR & 127.276 (0.000) &  & 0.568 (0.006) &  &  &  \\
XGBoost & \textbf{89.355 (0.236)} &  & \textbf{0.461 (0.007)} &  & 0.986 (0.000) &  \\
\hline
SEE-Net (Pre-intp) & 269.171 (0.215) & 0.995 (0.005) & 0.663 (0.008) & 0.943 (0.021) & 0.988 (0.002) & 0.948 (0.013) \\
SEE-Net (Intp) & 262.994 (2.442) & 0.997 (0.003) & 3.893 (3.934) & 1.000 (0.000) & 0.924 (0.008) & 1.000 (0.000) \\
MLM-CELL & 875.494 (379.166) & 0.943 (0.013) & 0.651 (0.022) & 0.880 (0.010) & 0.816 (0.014) & 0.994 (0.002) \\
MLM-EPIC & 823.153 (369.246) & 0.739 (0.092) & 0.682 (0.130) & 0.879 (0.014) &  &  \\
\hline
Teacher MLP & 89.498 (3.812) &  & 0.515 (0.006) &  & \textbf{0.995 (0.000)} &  \\
WSSN & 90.367 (5.497) &  & 0.525 (0.006) &  & 0.991 (0.000) &  \\
\textbf{\method} & \textcolor{teal!70!blue}{\textbf{142.273 (4.241)}} & 0.988 (0.008) & \textcolor{teal!70!blue}{\textbf{0.608 (0.022)}} & 0.738 (0.113) & \textcolor{teal!70!blue}{\textbf{0.963 (0.001)}} & 1.000 (0.000) \\
\hline
\end{tabular}
\end{adjustbox}
\end{center}
\endgroup
\end{table*}

\begin{table*}[t]
\begingroup
\setlength{\tabcolsep}{2.8pt}
\renewcommand{\arraystretch}{1.08}
\caption{Prediction performance on the three image datasets, reported as mean (std) over five runs. Inherently interpretable models and MLM-CELL/MLM-EPIC are omitted, as they are not designed for image data. Conventions follow Table~\ref{tab:1}.}
\label{tab:2}
\begin{center}
\begin{adjustbox}{max width=\textwidth}
\begin{tabular}{l cc cc cc}
\hline
\rowcolor{gray!18}
 & \multicolumn{2}{c}{MNIST}
 & \multicolumn{2}{c}{Flowers}
 & \multicolumn{2}{c}{CIFAR-10} \\
\cline{2-7}
\rowcolor{gray!18}
\multirow{-2}{*}{Method}
 & Test AUC $\uparrow$ & Test Gate AUC $\uparrow$
 & Test AUC $\uparrow$ & Test Gate AUC $\uparrow$
 & Test AUC $\uparrow$ & Test Gate AUC $\uparrow$ \\
\hline
SEE-Net (Pre-intp) & 0.999 (0.000) & 0.979 (0.008) & 0.875 (0.007) & 0.804 (0.137) & 0.920 (0.005) & 0.779 (0.015) \\
SEE-Net (Intp) & 0.985 (0.002) & 1.000 (0.000) & 0.841 (0.053) & 0.999 (0.001) & 0.773 (0.029) & 0.994 (0.013) \\
\hline
Teacher CNN & \textbf{1.000 (0.000)} &  & 0.850 (0.018) &  & \textbf{0.961 (0.001)} &  \\
WSSN & 1.000 (0.000) &  & 0.893 (0.015) &  & 0.927 (0.001) &  \\
\textbf{\method} & \textcolor{teal!70!blue}{\textbf{0.999 (0.000)}} & 0.999 (0.000) & \textcolor{teal!70!blue}{\textbf{0.862 (0.019)}} & 0.974 (0.007) & \textcolor{teal!70!blue}{\textbf{0.908 (0.002)}} & 0.993 (0.000) \\
\hline
\end{tabular}
\end{adjustbox}
\end{center}
\endgroup
\end{table*}

\medskip

\noindent\textbf{Regime Recovery and Interpretability Analysis. }\label{sec:experiments:realworld:interpretability}
Beyond prediction, \method provides interpretable regime assignments and regime-specific prediction rules, explaining both \emph{why} an instance is routed to a regime and \emph{how} its prediction is formed. We examine this dual interpretability on the image datasets from three perspectives: 1) whether routing is effective without collapsing to a single regime; 2) what the explanatory gate and local expert coefficients reveal about regime assignment and prediction; 3) whether the resulting regimes and explanations are human-readable. The corresponding diagnostics for the tabular datasets are reported in the supplementary material.

\begin{figure*}[!t]
\vspace{-14pt}
\centering
\captionsetup[subfloat]{captionskip=0pt}
\includegraphics[width=0.32\textwidth]{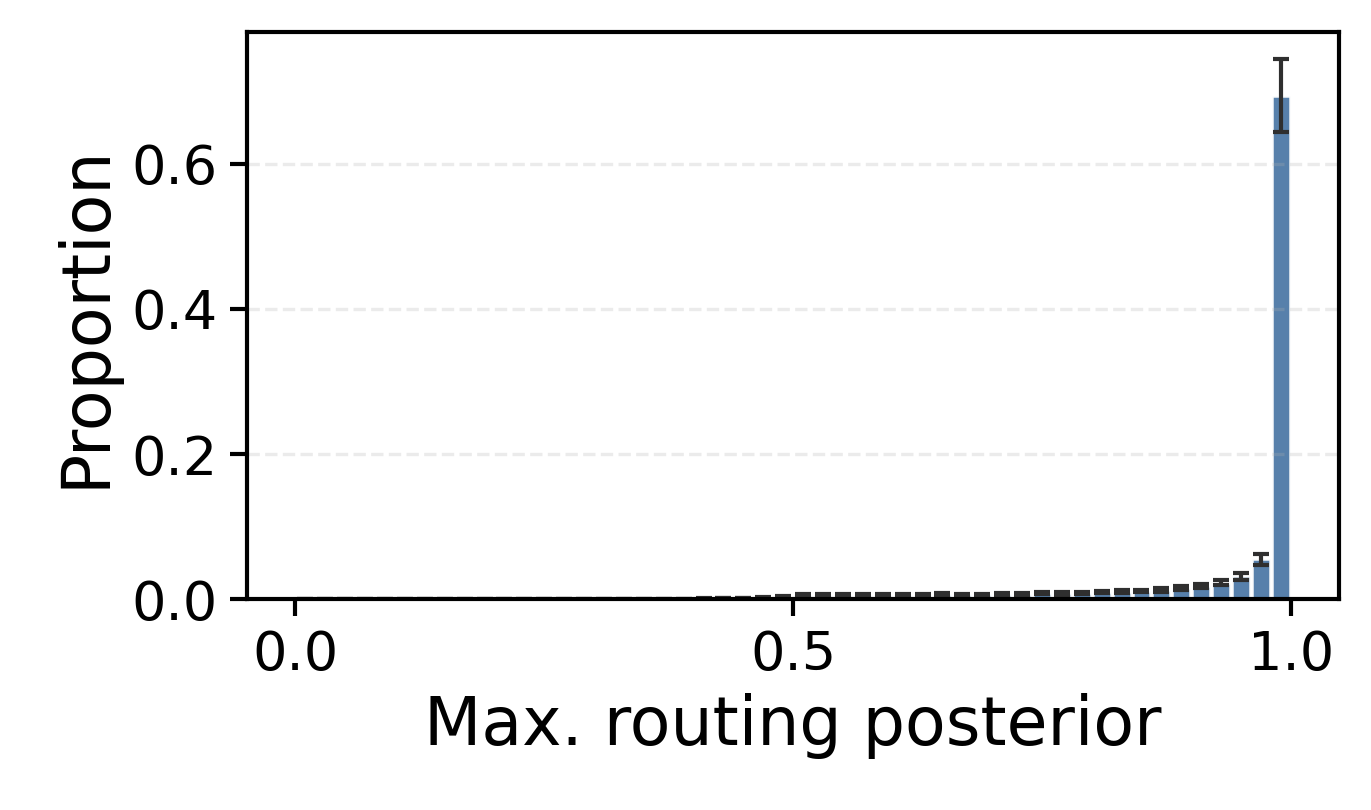}
\hfill
\includegraphics[width=0.32\textwidth]{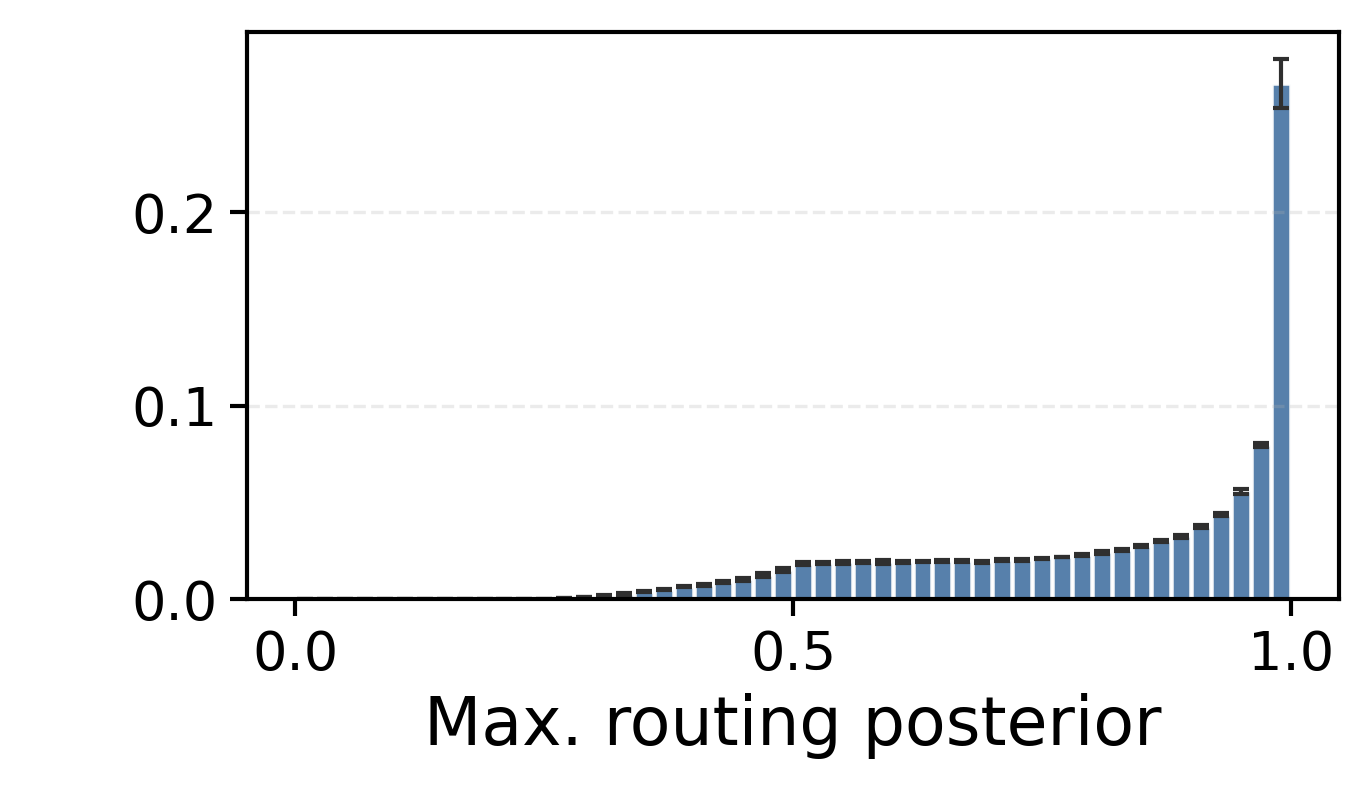}
\hfill
\includegraphics[width=0.32\textwidth]{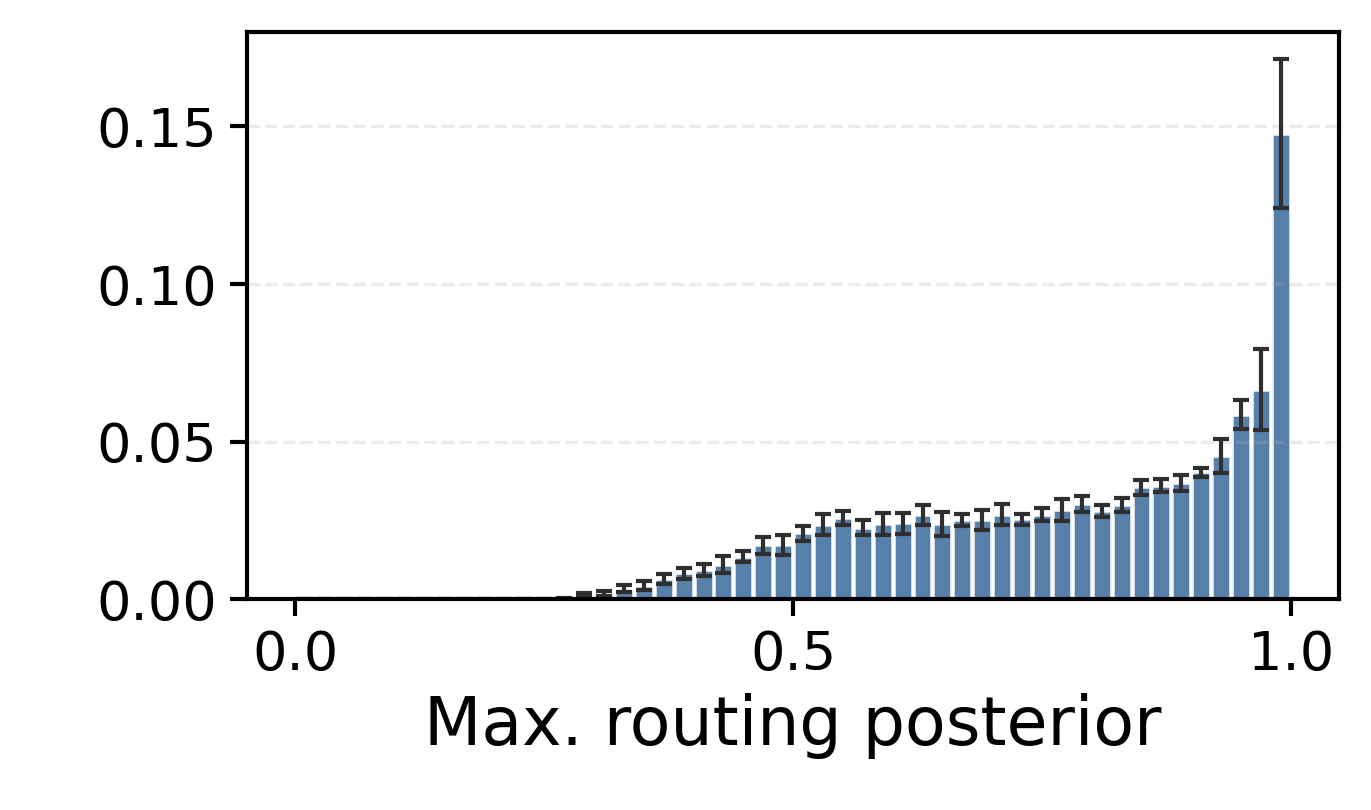}
\par\vspace{-10pt}
\subfloat[MNIST]{\includegraphics[width=0.32\textwidth]{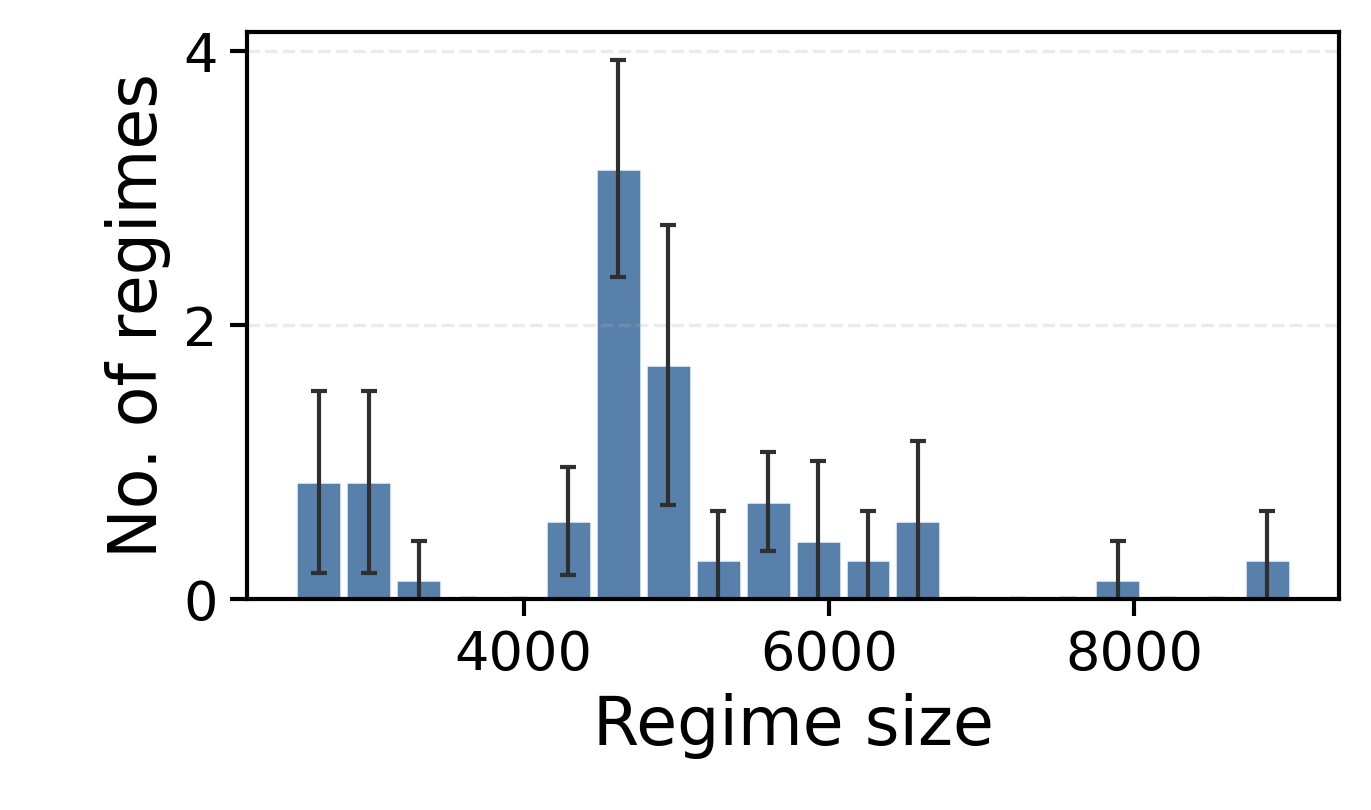}\label{fig:route:mnist}}
\hfill
\subfloat[CIFAR-10]{\includegraphics[width=0.32\textwidth]{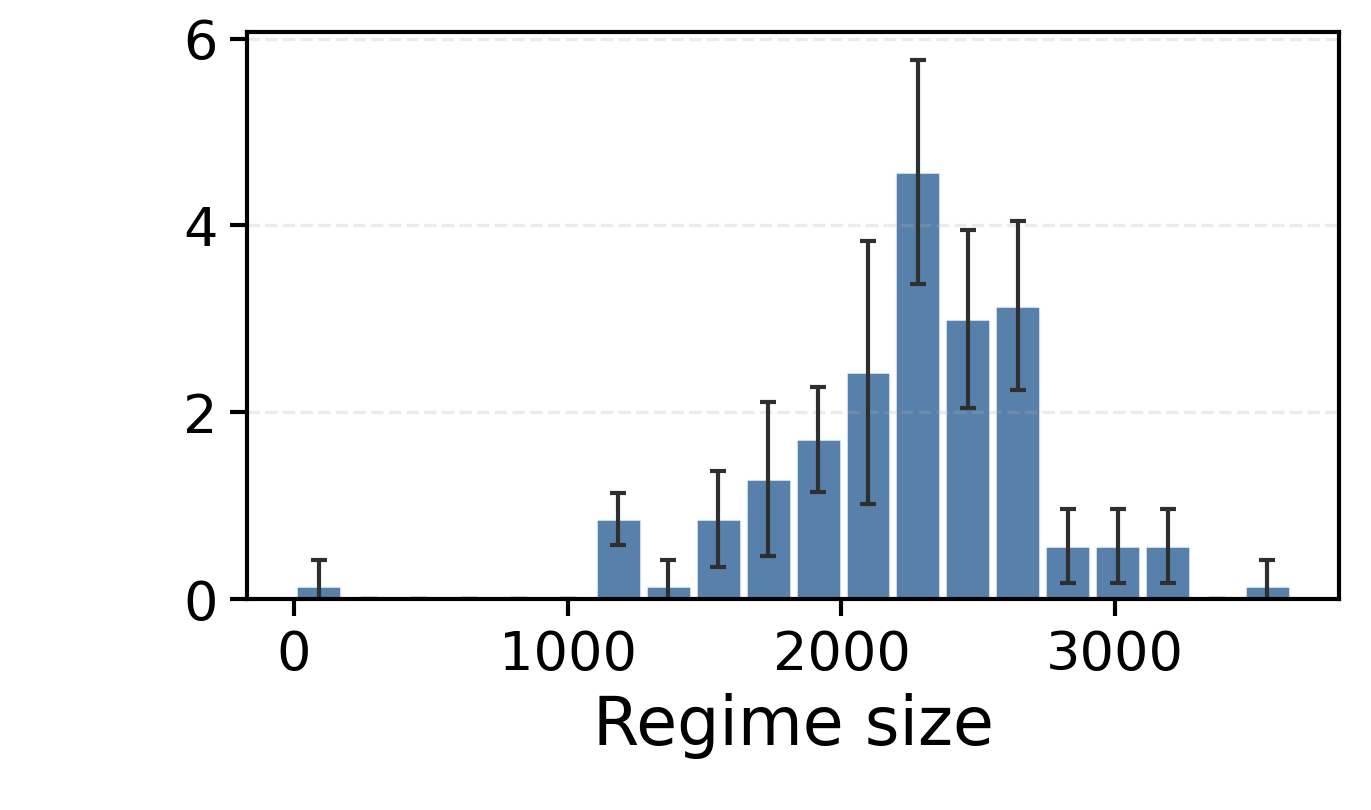}\label{fig:route:cifar}}
\hfill
\subfloat[Flowers]{\includegraphics[width=0.32\textwidth]{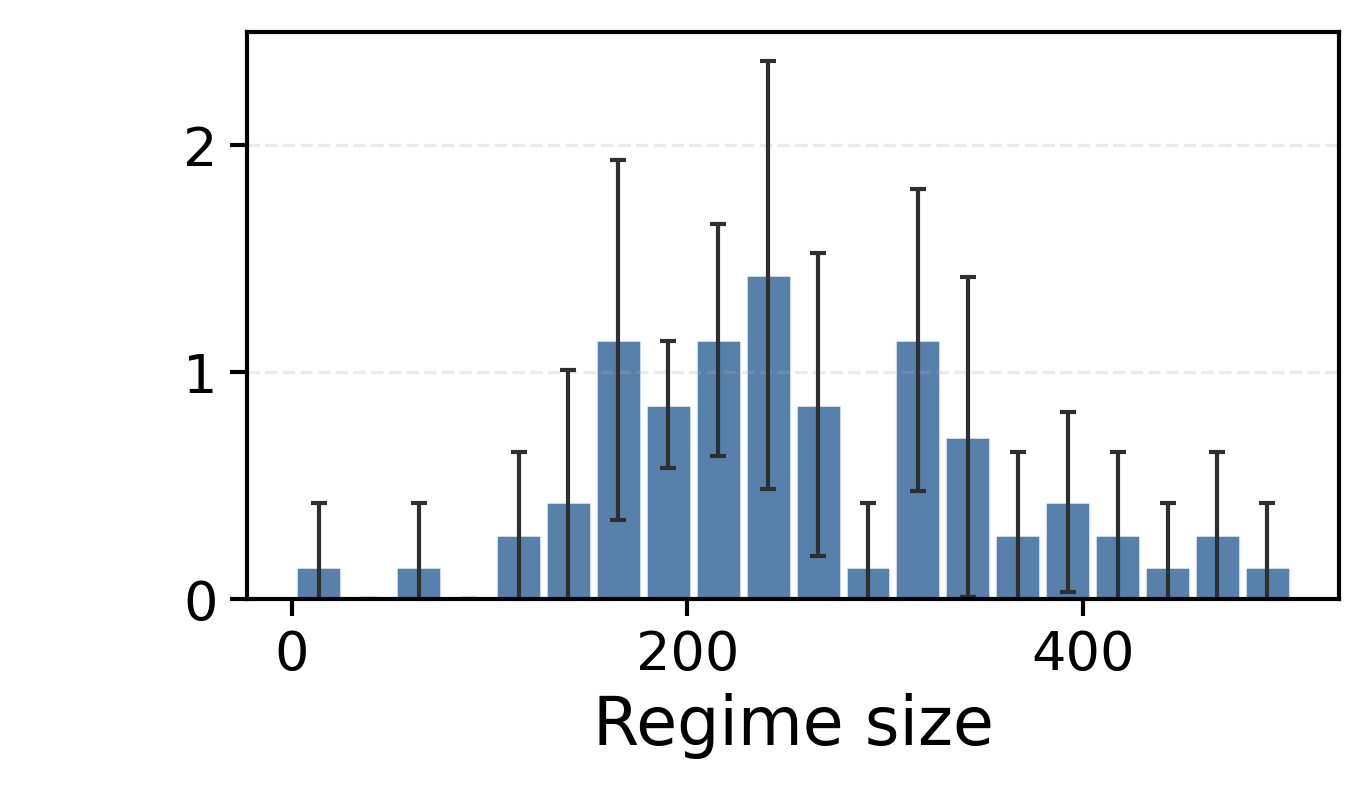}\label{fig:route:flower}}
\caption{Routing sharpness and regime size balance on MNIST~(a), CIFAR-10~(b), and Flowers~(c). Top row: distribution of the maximum routing posterior $\max_k \pi_k(x)$ over training samples. Bottom row: distribution of regime sizes, i.e., the number of training samples assigned to each non-empty regime. Bars are means over five independent runs, and error bars are $\pm1.96$ standard errors across runs.}
\label{fig:img:route}
\end{figure*}

\begin{figure*}[!t]
\centering
\includegraphics[width=\textwidth]{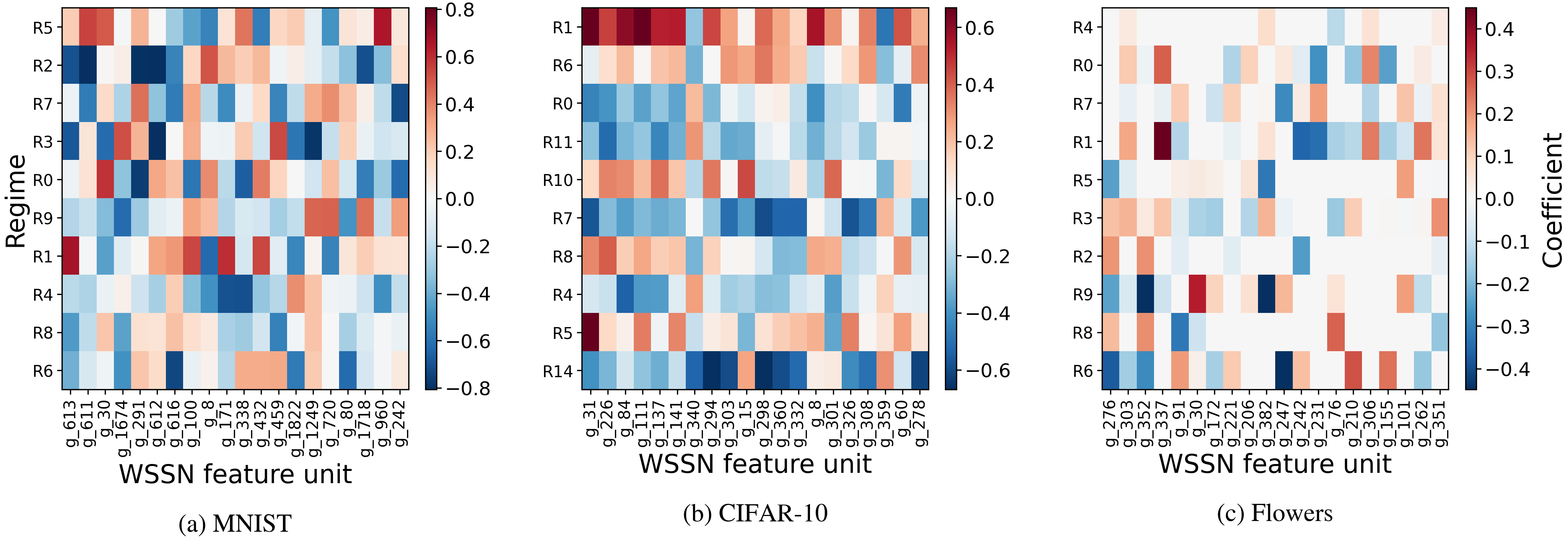}
\caption{Single-run explanatory gate coefficients on MNIST~(a), CIFAR-10~(b), and Flowers~(c). Rows are the 10 largest regimes by training-set size. Columns show feature units of $u(x)$, labeled \texttt{g\_}$j$. For MNIST, a feature unit is one spatial location in a convolutional channel; for CIFAR-10 and Flowers, it is a max-pooled channel. Color represents the gate coefficient of the labeled feature unit for that regime.}
\label{fig:img:gate}
\end{figure*}

Fig.~\ref{fig:img:route} shows that routing is effective on all three datasets. On MNIST, the maximum routing posterior concentrates near one and the regime sizes are balanced across all the regimes (Fig.~\ref{fig:img:route}\subref{fig:route:mnist}), indicating sharp, noncollapsed routing. On CIFAR-10, routing is less sharp than on MNIST, which is consistent with the greater visual complexity, while the regime sizes remain balanced with few oversized or undersized regimes (Fig.~\ref{fig:img:route}\subref{fig:route:cifar}). On Flowers, routing is likewise less sharp than on MNIST, while the regime sizes stay relatively balanced (Fig.~\ref{fig:img:route}\subref{fig:route:flower}).

Fig.~\ref{fig:img:gate} illustrates regime-assignment rules for the image datasets. For MNIST, \texttt{g\_613} has a coefficient of $+0.67$ for regime~1 against $-0.70$ and $-0.68$ for regimes~2 and~3, so stronger activation of this feature-map location favors regime~1 over regimes~2 and~3. For CIFAR-10, \texttt{g\_31} has a coefficient of $+0.73$ for regime~1 against $-0.40$ for regime~14, so stronger activation of this channel favors regime~1 over regime~14. For Flowers, some pooled features receive coefficients of opposite signs across regimes, reflecting different regime-assignment rules.

\begin{figure*}[!t]
\centering
\subfloat[Regime 1]{\includegraphics[height=0.22\textheight]{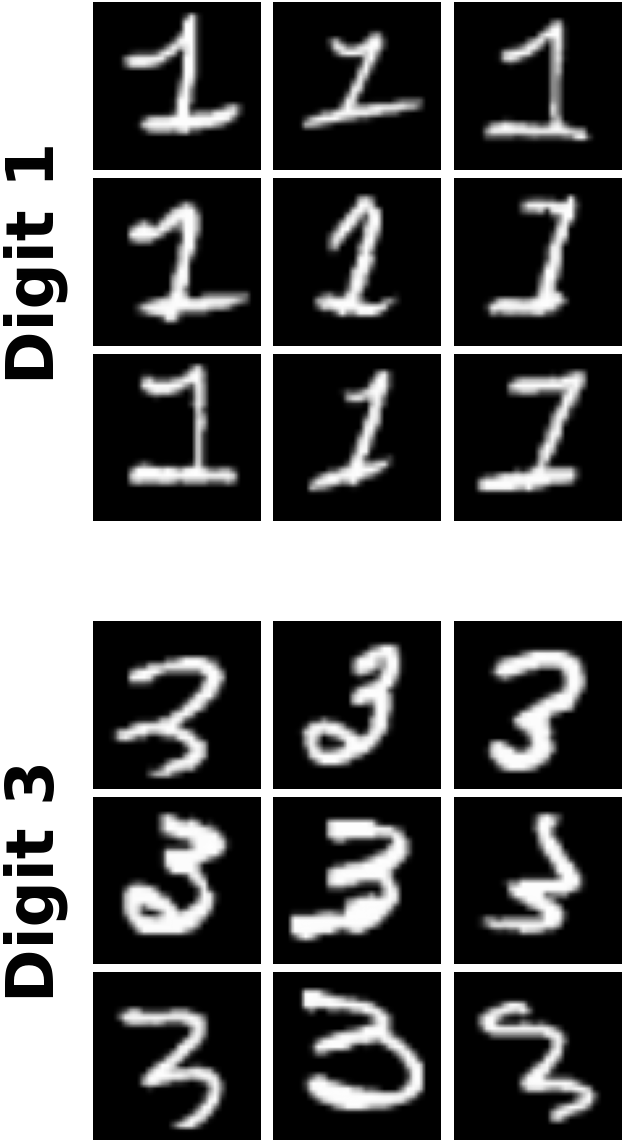}\label{fig:mnist:r1}}
\hspace{2pt}
\subfloat[Regime 2]{\includegraphics[height=0.22\textheight]{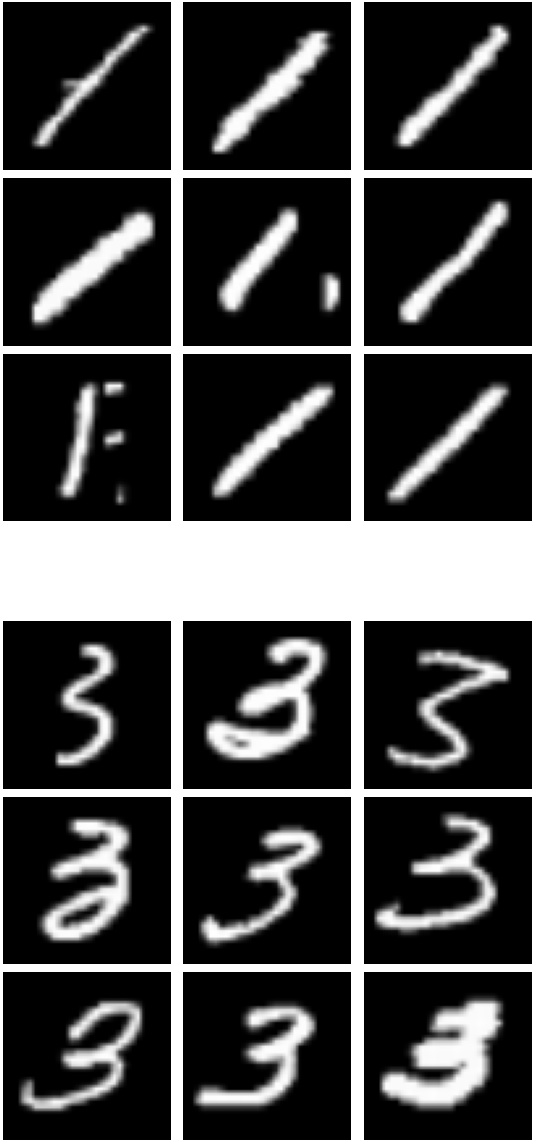}\label{fig:mnist:r2}}
\hspace{2pt}
\subfloat[Regime 6]{\includegraphics[height=0.22\textheight]{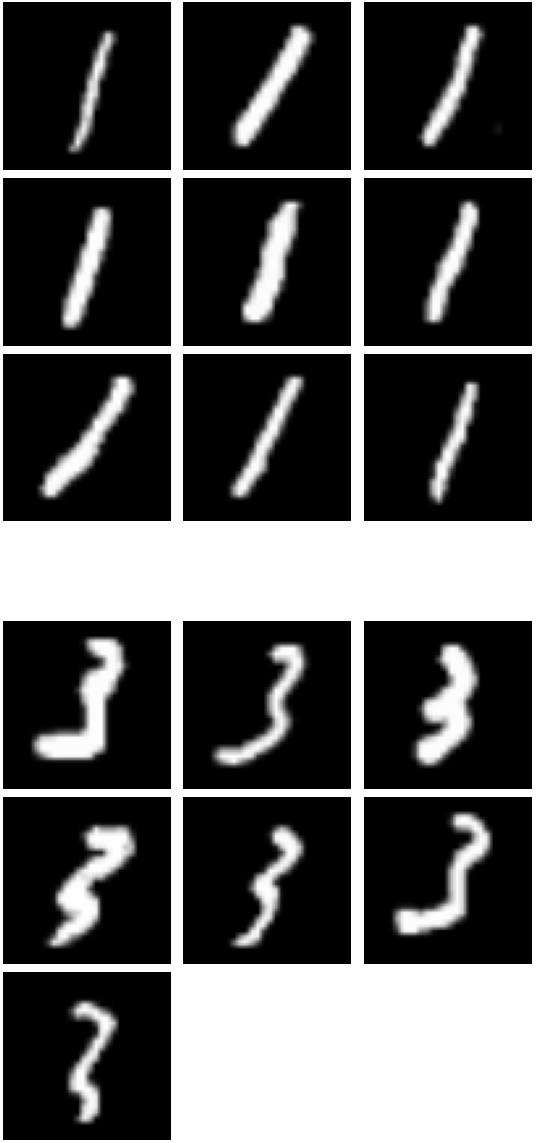}\label{fig:mnist:r6}}
\hspace{2pt}
\subfloat[Regime 7]{\includegraphics[height=0.22\textheight]{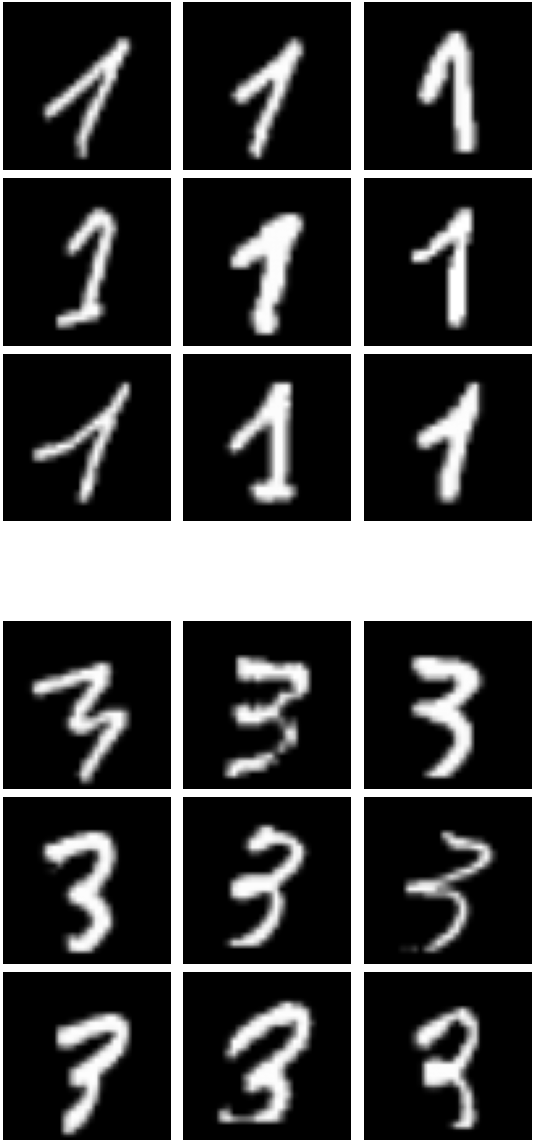}\label{fig:mnist:r7}}
\hspace{2pt}
\subfloat[Regime 8]{\includegraphics[height=0.22\textheight]{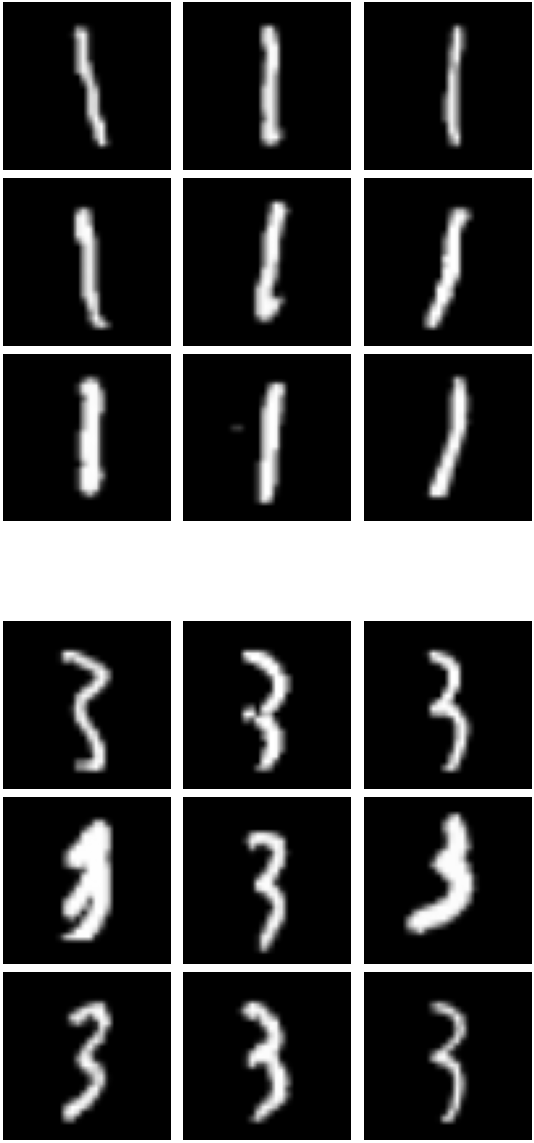}\label{fig:mnist:r8}}
\par\vspace{2pt}
\subfloat[Regime 0]{\includegraphics[height=0.22\textheight]{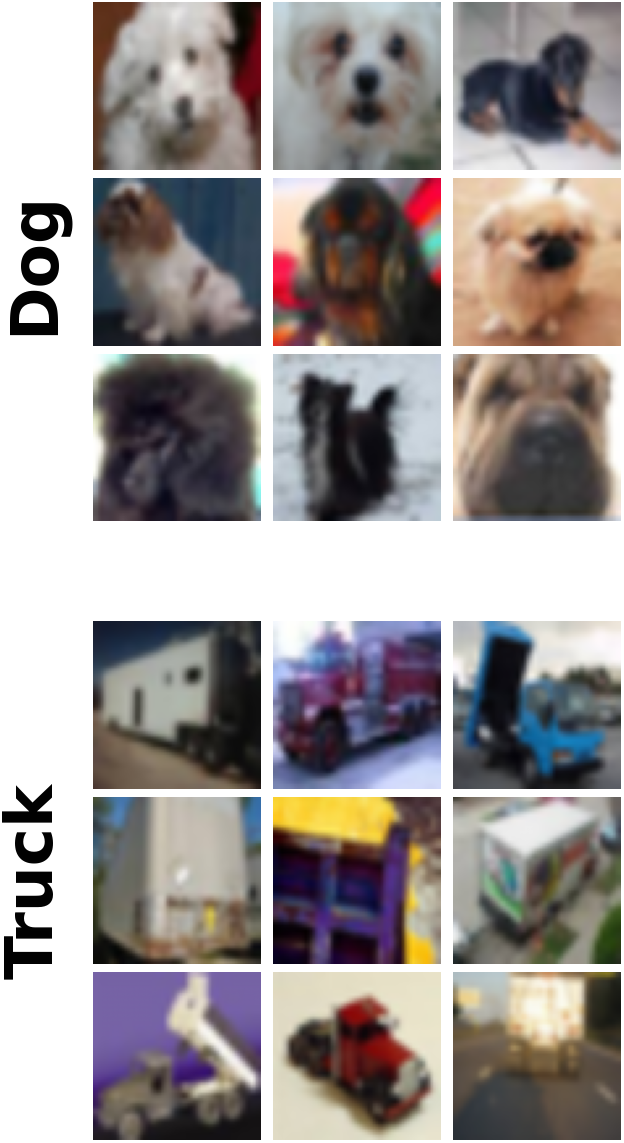}\label{fig:cifar:r0}}
\hspace{2pt}
\subfloat[Regime 3]{\includegraphics[height=0.22\textheight]{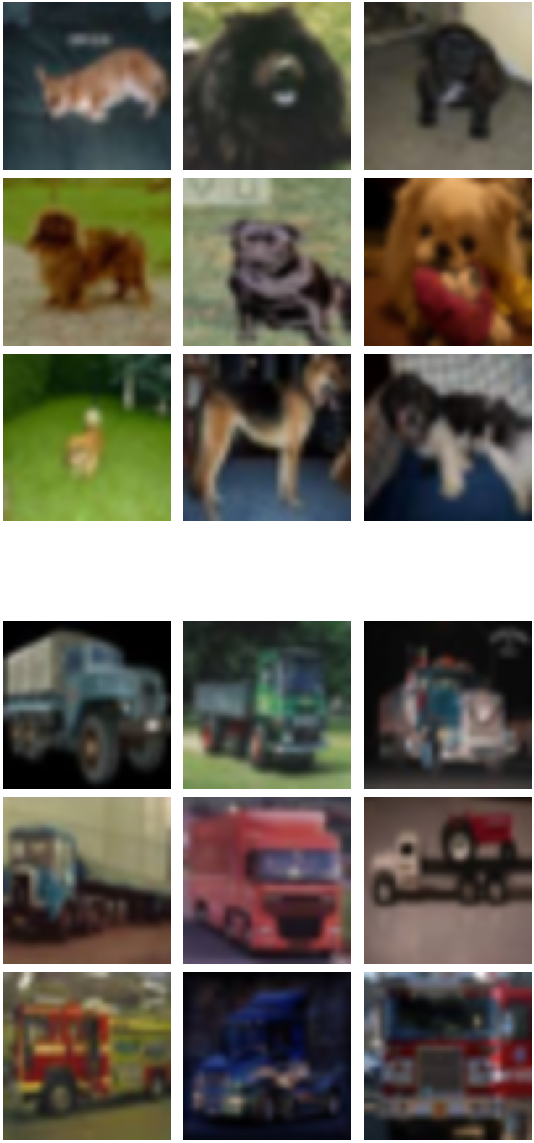}\label{fig:cifar:r3}}
\hspace{2pt}
\subfloat[Regime 6]{\includegraphics[height=0.22\textheight]{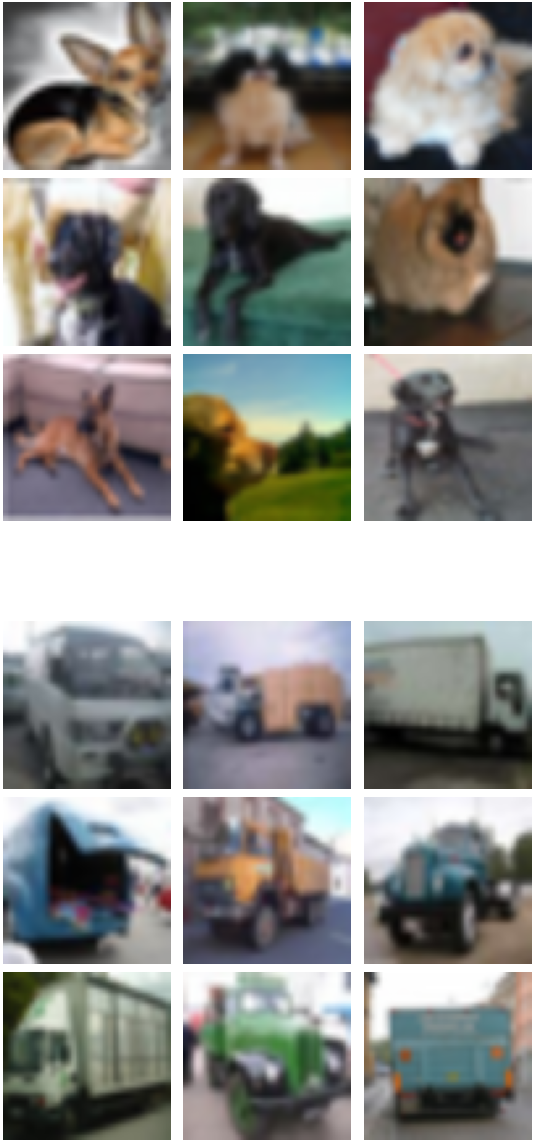}\label{fig:cifar:r6}}
\hspace{2pt}
\subfloat[Regime 7]{\includegraphics[height=0.22\textheight]{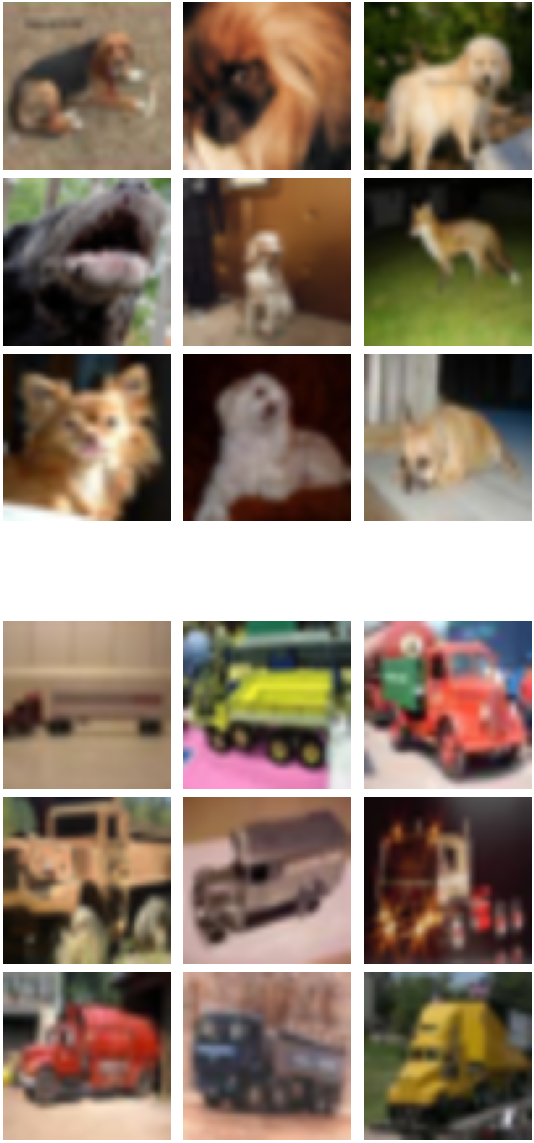}\label{fig:cifar:r7}}
\hspace{2pt}
\subfloat[Regime 9]{\includegraphics[height=0.22\textheight]{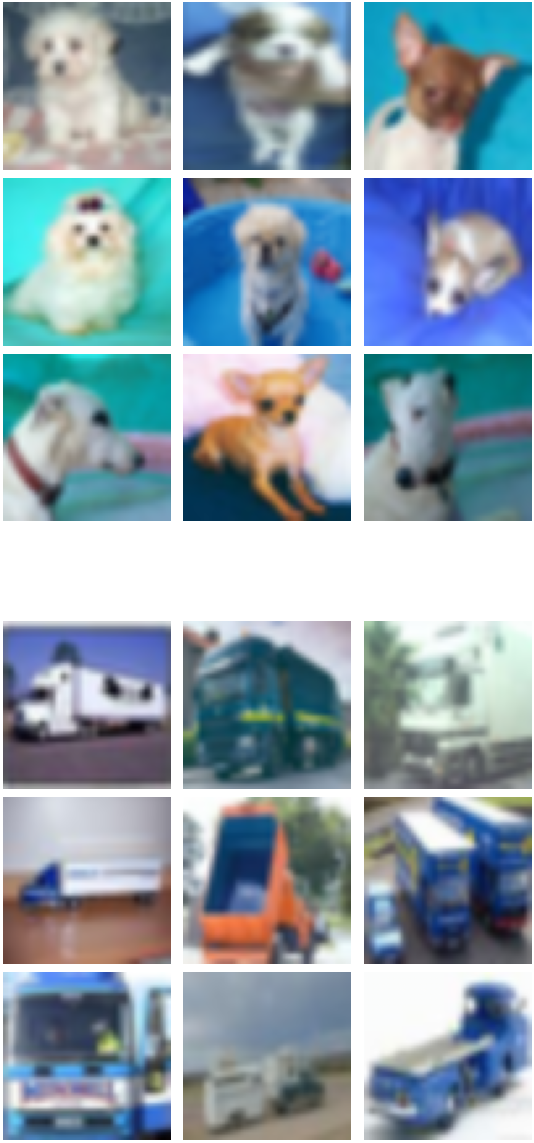}\label{fig:cifar:r9}}
\par\vspace{2pt}
\subfloat[Regime 0]{\includegraphics[height=0.22\textheight]{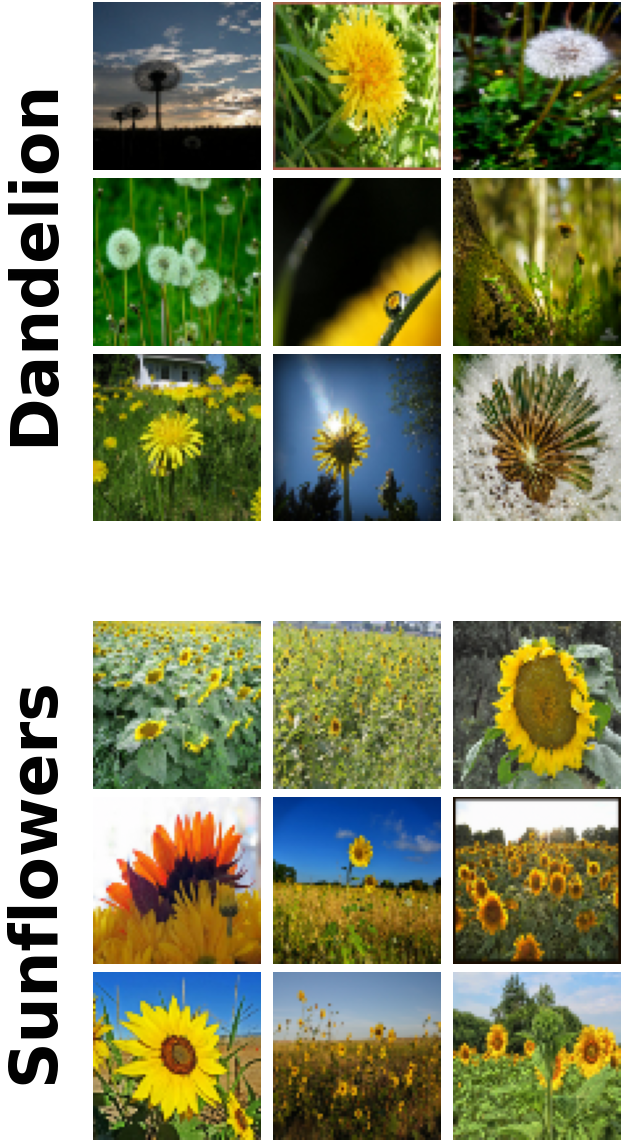}\label{fig:flower:r0}}
\hspace{2pt}
\subfloat[Regime 1]{\includegraphics[height=0.22\textheight]{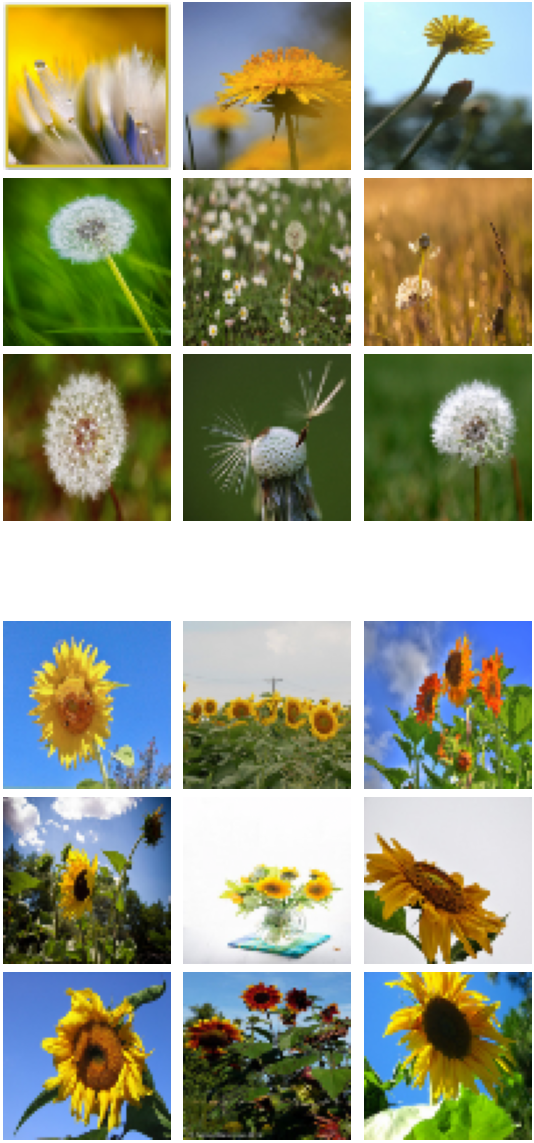}\label{fig:flower:r1}}
\hspace{2pt}
\subfloat[Regime 6]{\includegraphics[height=0.22\textheight]{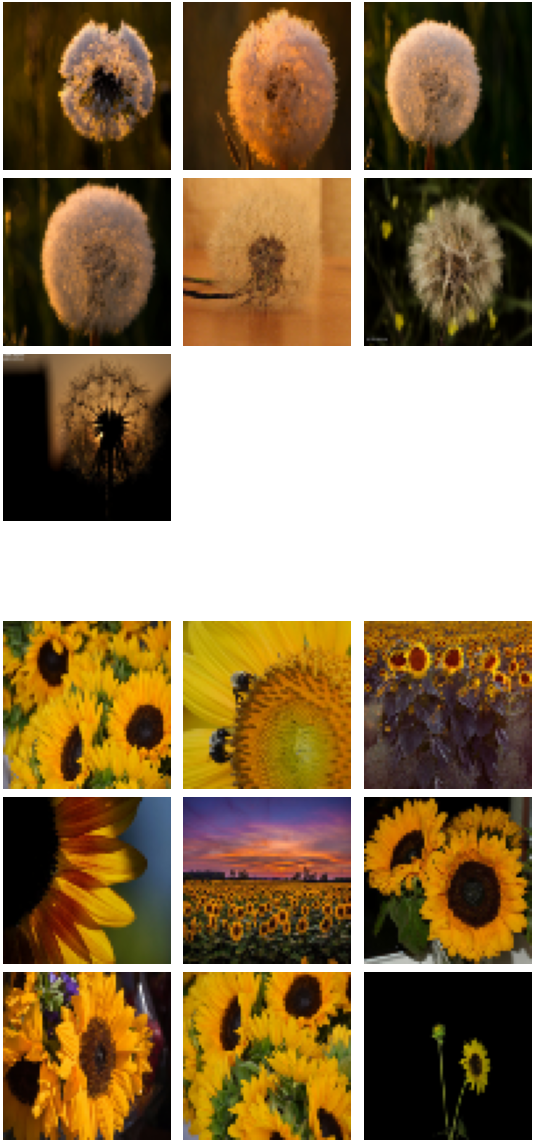}\label{fig:flower:r6}}
\hspace{2pt}
\subfloat[Regime 7]{\includegraphics[height=0.22\textheight]{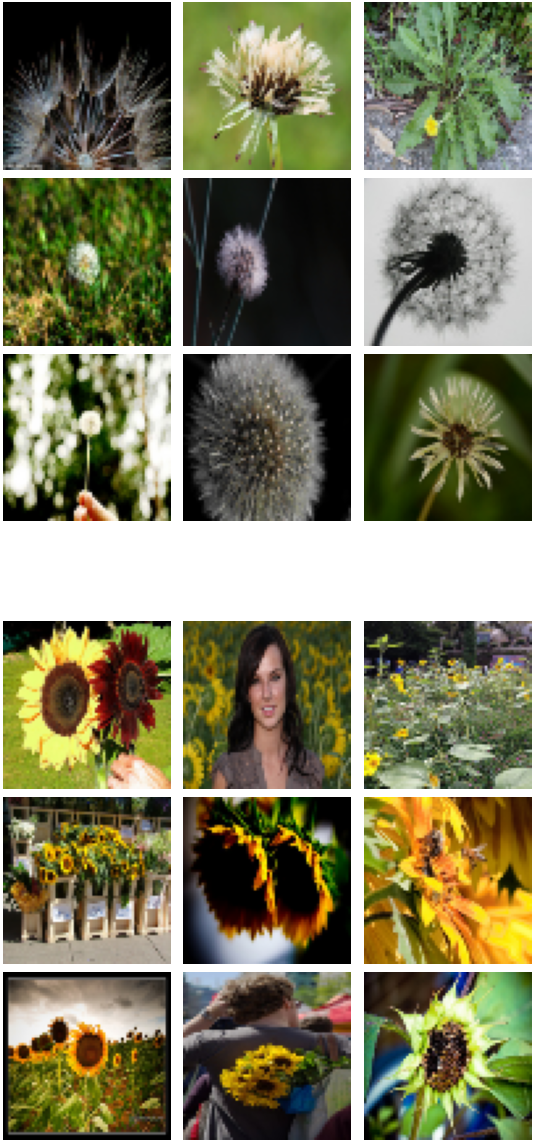}\label{fig:flower:r7}}
\hspace{2pt}
\subfloat[Regime 9]{\includegraphics[height=0.22\textheight]{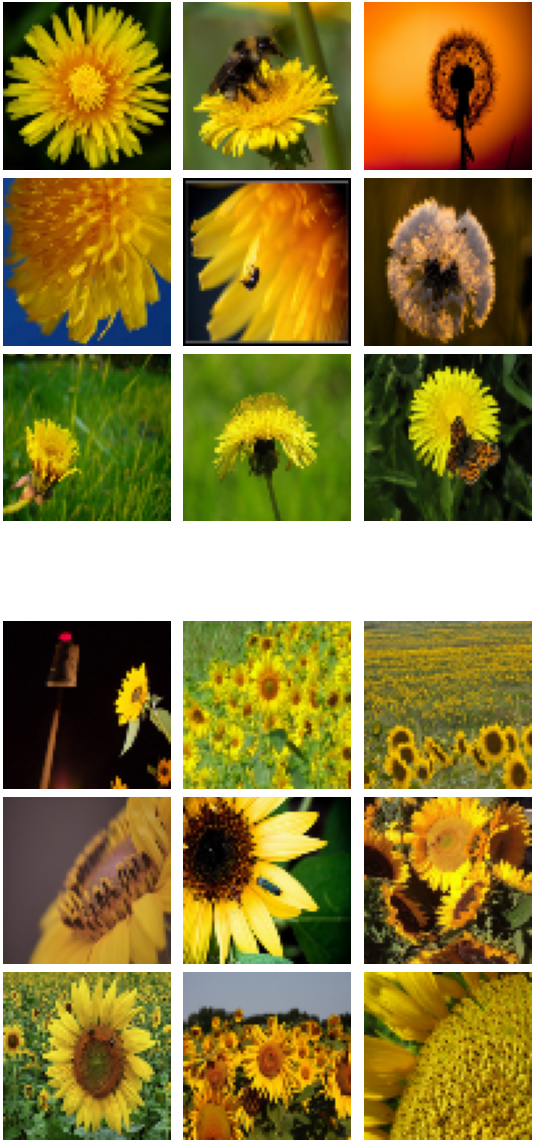}\label{fig:flower:r9}}
\caption{Training images grouped by regime. (a)--(e) MNIST digits~1 and~3 in five selected regimes, revealing variations in writing style among same-digit images. (f)--(j) CIFAR-10 dogs and trucks in five selected regimes, revealing variations in color, shape, and style among same-class images. (k)--(o) Dandelions and sunflowers from Flowers in five selected regimes, revealing variations in visual tone, framing, and bloom stage among same-class images. Within each regime, each class is shown as a \(3\times3\) grid of randomly sampled training images.}
\label{fig:img:regimes}
\end{figure*}

Fig.~\ref{fig:img:regimes} shows that the learned regimes are human-readable. For MNIST in Fig.~\ref{fig:img:regimes}\subref{fig:mnist:r1}--\subref{fig:mnist:r8}, each regime isolates a coherent writing style. For example, digit~1 has a more formal style in regime~1, a strongly slanted version in regime~2, and thin upright strokes in regime~8. Digit~3 also shows distinct variations across regimes. For CIFAR-10 in Fig.~\ref{fig:img:regimes}\subref{fig:cifar:r0}--\subref{fig:cifar:r9}, the partition captures interpretable structure even at \(32\times32\) resolution. Dogs separate largely by pose and framing: regime~0 collects head-only close-ups facing the camera; regime~7 gathers side-facing views; regimes~3 and~6 contain full-body shots. Trucks are instead organized by viewpoint, with regimes~3, 6, and~7 each grouping trucks facing a consistent direction. Regime~9 is color-driven across both classes, with predominantly white dogs on blue backgrounds and blue-toned trucks. For Flowers in Fig.~\ref{fig:img:regimes}\subref{fig:flower:r0}--\subref{fig:flower:r9}, regimes~6 and~9 share a warm visual tone and differ mainly in dandelions, with white seed heads in regime~6 and yellow blooms in regime~9; regime~1 gathers sunflowers on lighter backgrounds, and regime~0 leans toward distant, wide-field shots. The visual differences between the displayed regimes are less pronounced for Flowers than for MNIST and CIFAR-10. Overall, \method organizes images by writing style, pose, framing, and color, suggesting that it partitions the input space along human-readable features.
\begin{figure*}[!t]
\centering
\includegraphics[width=0.7\textwidth]{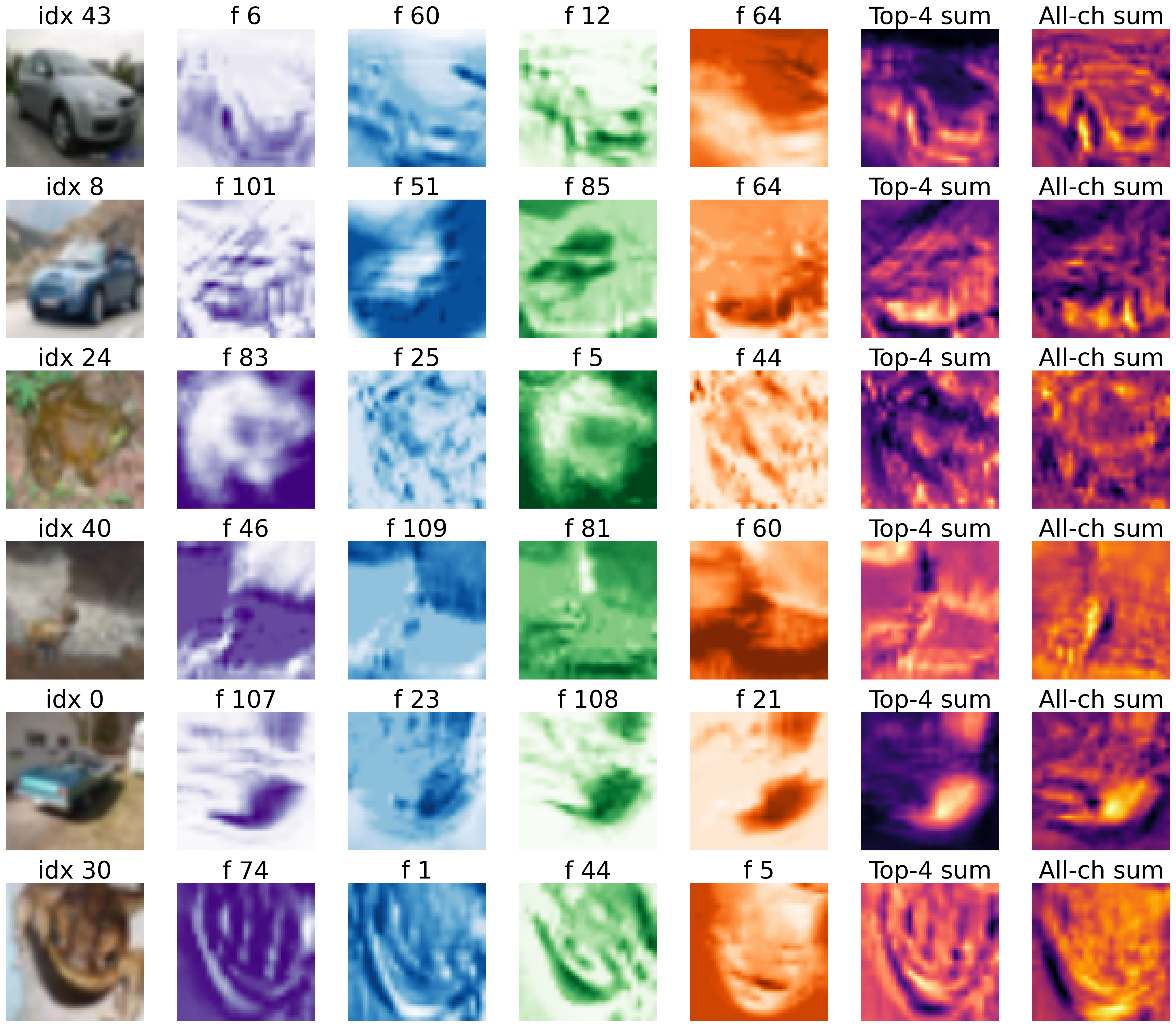}
\caption{CIFAR-10 local expert contribution maps. The first column shows the original image and its index (``idx''). The next four columns show the filter groups with the largest total absolute weight$\times$activation. Each group sums the expert-weighted pre-pooling feature maps across the three RGB input channels. The final two columns sum contributions over these four groups (Top-4) and all channels, respectively. All maps are computed for the ground-truth class using the expert with the highest routing posterior. Larger positive contributions appear darker in the per-group maps and brighter in the aggregate maps.}
\label{fig:cifar:heat}
\end{figure*}

After analyzing regime discovery, we now turn to regime-specific prediction logic. Fig.~\ref{fig:cifar:heat} visualizes the feature-map locations that are most influential for each prediction. In row~1, \texttt{f60} and \texttt{f12} trace the car's silhouette while \texttt{f64} focuses on the windshield. The Top-4 sum concentrates on the silhouette. In row~2, \texttt{f101} outlines the silhouette, \texttt{f85} responds to the upper body and \texttt{f64} to the lower body, while \texttt{f51} instead focuses on the background. The Top-4 sum emphasizes the lower body together with the full silhouette. Row~5 is the most consistent example in which every displayed filter group, together with the Top-4 and All-channel sums, concentrates on the shadowed part of the car. For the frog in row~6, \texttt{f74}, \texttt{f1}, and \texttt{f44} spread over the whole animal while \texttt{f5} focuses on the background, and the Top-4 sum covers the body and traces its outline. Notably, several filter groups recur across images: \texttt{f64} on both cars (rows~1--2), \texttt{f44} and \texttt{f5} on both frogs (rows~3 and~6), and \texttt{f60} on a car (row~1) and the deer (row~4), which suggests that these filter groups capture reusable visual patterns. Across rows, the Top-4 and All-channel sums remain similar, indicating that predictive contributions are concentrated in a few dominant filter groups.

\subsection{Ablation Studies}\label{sec:experiments:ablation}

\noindent\textbf{Does the Router Carry the Predictive Burden?}\label{sec:experiments:ablation:heavy_lifting}
A key design claim of \method is that the routing stage discovers a meaningful partition while the predictive burden is carried mostly by the local linear experts rather than by the router itself. To verify this claim, we replace the local linear experts with a majority-vote baseline that fits no model within each regime. For classification, the majority vote replaces each local expert prediction $e_k(x)$ in Eq.~\eqref{eq:6} with the empirical class-frequency vector $d_k$ of regime $k$. Each component $d_{k,c}$ is the proportion of samples in class $c$ among all training samples in regime $k$. The final prediction aggregates them as in Eq.~\eqref{eq:6}, giving the weighted vote $\hat{f}_{\mathrm{mv}}(x) = \argmax_c \sum_{k=1}^{K} \pi_k(x)\, d_{k,c}$. For AUC evaluation, we use the mixed class-frequency vector $\sum_{k=1}^{K}\pi_k(x)d_k$ before taking the argmax. For regression, we use $\bar{y}_k$, the mean response for training samples in regime $k$, so the prediction becomes $\hat{f}_{\mathrm{mv}}(x) = \sum_{k=1}^{K} \pi_k(x)\, \bar{y}_k$. Since linear modeling is skipped in all regimes, this experiment reveals the standalone prediction capability of the routing stage.

\begin{figure}[t]
\centering
\includegraphics[width=\columnwidth]{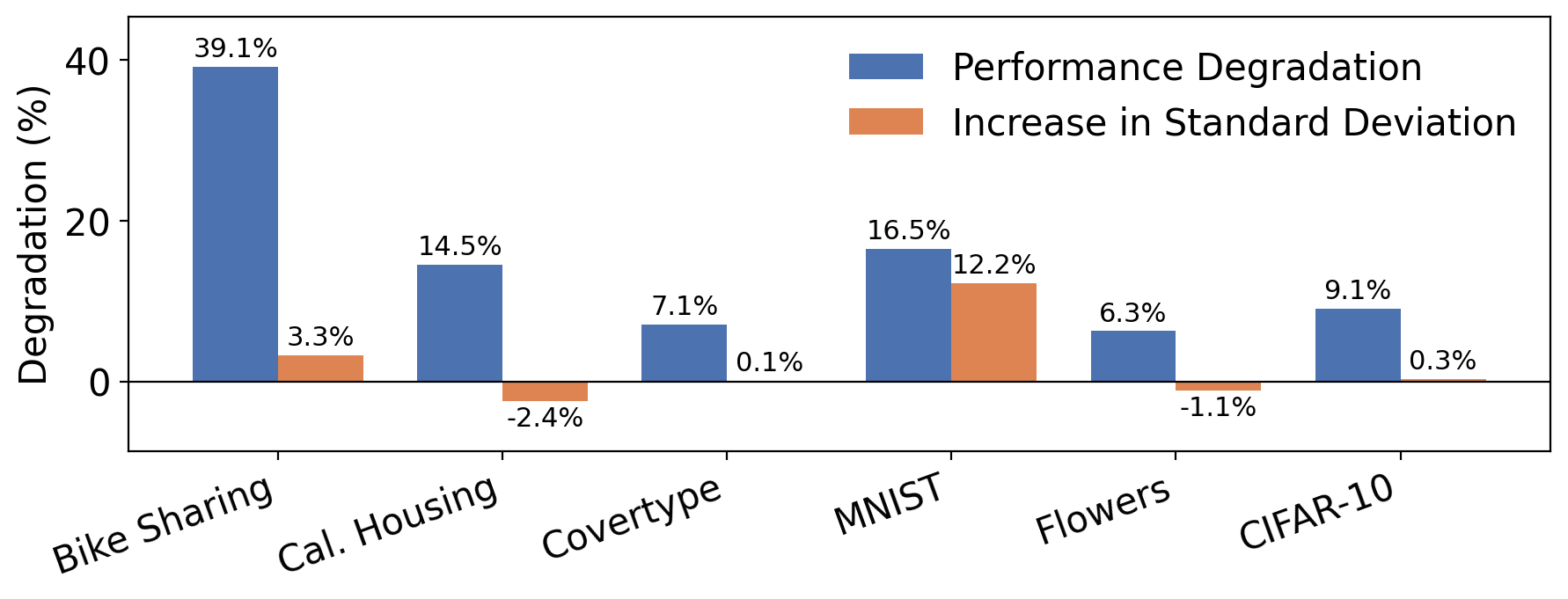}
\caption{Effect of replacing the local linear experts with majority-vote predictions. Blue bars give the performance degradation (percentage increase in RMSE for regression, or percentage drop in AUC for classification). Orange bars show the increase in the cross-run standard deviation of the test metric, expressed as a percentage of the full model's test metric. Negative orange bars indicate a decrease in standard deviation. Positive values mean the ablation is worse than the full \method.}
\label{fig:abl:vote}
\end{figure}

Fig.~\ref{fig:abl:vote} summarizes the results across all six datasets. Replacing the local experts with majority-vote predictions reduces predictive performance on every dataset and increases the cross-run standard deviation of the test metric on four of the six datasets, most noticeably on MNIST and Bike Sharing. This indicates that the discovered regimes alone are insufficient for strong predictive performance. These results show that the \emph{linear} experts recover predictive performance that within-regime constants cannot, suggesting that the learned regimes successfully decompose a nonlinear prediction task into subproblems that can be effectively modeled by linear models. Overall, these findings indicate that \method identifies a meaningful two-level structure that achieves both interpretability and strong predictive accuracy.

\medskip

\noindent\textbf{Routing with the Explanatory Gate.}\label{sec:experiments:ablation:gate_routing}
Section~\ref{sec:method} introduced \method-EG, the two-layer linear model in which the explanatory gate is used as the router instead of as a post hoc explanation of the router. Since the gate already recovers the discovered regimes with high Gate AUC, as shown in Tables~\ref{tab:1} and~\ref{tab:2}, we examine whether \method-EG predicts as well as \method.

\begin{figure}[t]
\centering
\includegraphics[width=\columnwidth]{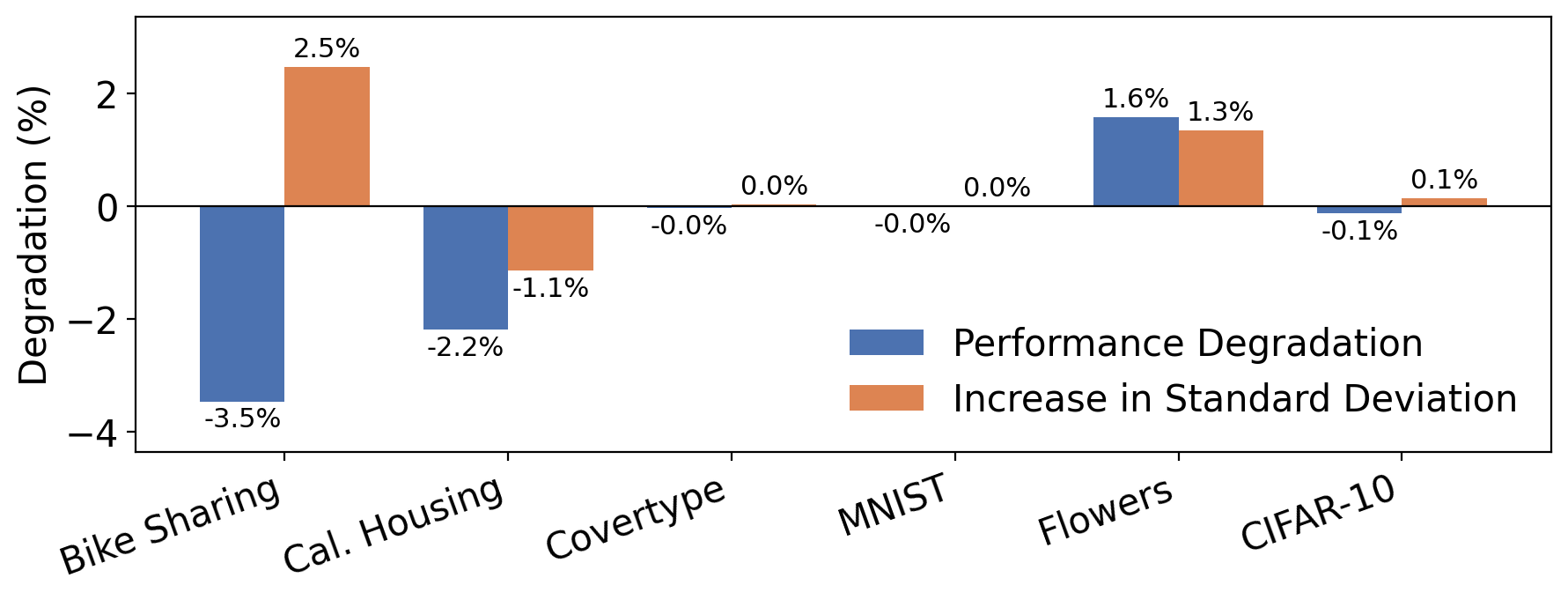}
\caption{Effect of replacing the router with the explanatory gate and refitting the local linear experts on all training samples under the gate posterior $\hat q_k(x)$. Bars and sign convention as in Fig.~\ref{fig:abl:vote}.}
\label{fig:abl:gate}
\end{figure}

As shown in Fig.~\ref{fig:abl:gate}, \method-EG performs comparably to \method. The largest performance degradation is 1.6\% on Flowers, and \method-EG slightly improves on Bike Sharing and California Housing. The differences on Covertype, MNIST, and CIFAR-10 are negligible. These results indicate that \method-EG achieves predictive performance comparable to \method, demonstrating that a two-layer linear model can achieve strong predictive performance while remaining inherently interpretable. We retain WSSN-based routing in the main results to evaluate the predictive usefulness of activation-induced regimes directly.

\medskip

\noindent\textbf{Effect of the Number of Regimes \texorpdfstring{$K$}{K}.}\label{sec:experiments:ablation:num_regimes}
Section~\ref{sec:theory:design_implications} motivates a trade-off between representational capacity and statistical reliability as $K$ varies. We therefore examine how \method performs with varying $K$ on the shared grid $K\in\{5,\,10,\,50,\,100,\,150,\,200\}$ for each dataset. We additionally include results for CIFAR-10 and Covertype at $K=20$ and $K=30$, respectively, which are the values of $K$ used for these datasets in the main experiments. Fig.~\ref{fig:abl:k} illustrates the results for each dataset, with regression and classification tasks separated into two plots.

Eq.~\eqref{eq:16} provides a viewpoint for interpreting Fig.~\ref{fig:abl:k}. Since the fitted experts retain nearly all features (Table~\ref{tab:4} of the supplementary material), we take $s_{\max}=p$ in Eq.~\eqref{eq:16} and report the ratio $n/(K\,p\log p)$ as a heuristic diagnostic of sample size relative to model complexity as $K$ varies. Larger values suggest a more favorable balance between the available samples and the complexity of the local models.

For Covertype and California Housing, the ratio remains above $1.0$ over the entire grid, decreasing from $345.3$ to $8.6$ on Covertype ($n=371{,}847$, $p=54$) and from $158.8$ to $4.0$ on California Housing ($n=13{,}209$, $p=8$). Performance improves monotonically on both datasets. As $K=200$ yields little gain over $K=150$ on California Housing, we chose $K=150$ for better interpretability, and likewise $K=30$ for Covertype. Bike Sharing ($n=11{,}122$, $p=19$) shows that this ratio alone does not identify the best $K$. Its RMSE is minimized at $K=10$ (mean 134.707, std 14.026), where the ratio is still $19.9$, and increases thereafter. The ratio decreases to approximately $1.0$ at $K=200$. Nevertheless, Bike Sharing clearly shows a U-shaped curve, with the best performance at an intermediate $K$, consistent with the trade-off between representational capacity and statistical reliability.

For the image datasets, the representation is high-dimensional, and the ratio is below $1.0$ over most of the grid: the ratio drops below $1.0$ at $K=20$ on CIFAR-10 ($n=45{,}000$, $p=384$), and at $K=5$ it is already $0.54$ on MNIST ($n=48{,}000$, $p=2{,}304$) and $0.22$ on Flowers ($n=2{,}569$, $p=384$). Fig.~\ref{fig:abl:k} shows that CIFAR-10's AUC decreases steadily, with a sharper decline beyond $K=50$, while Flowers shows the largest AUC decrease among the three image datasets. MNIST maintains a near-perfect AUC, illustrating that this ratio alone does not determine predictive performance.

Except for Covertype and California Housing, $K$ was matched to the applicable baseline settings.

\begin{figure}[t]
\centering
\includegraphics[width=\columnwidth]{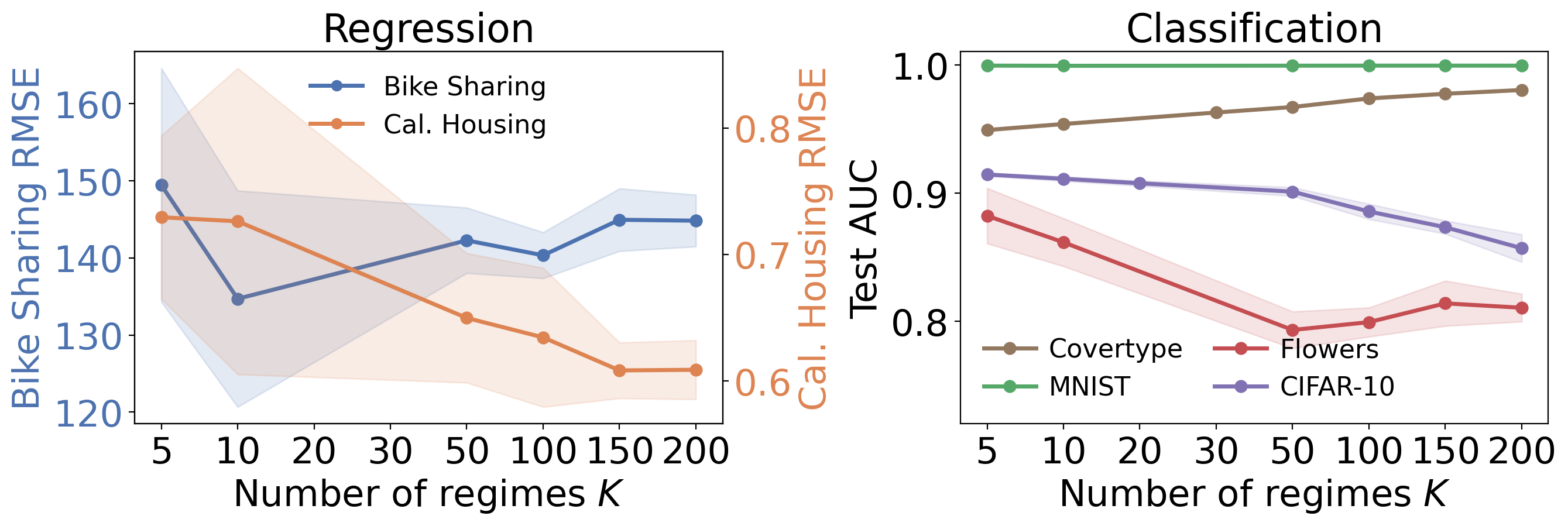}
\caption{Effect of the number of regimes $K$ on test performance. Left: regression RMSE (Bike Sharing on the left axis, California Housing on the right). Right: classification AUC.}
\label{fig:abl:k}
\end{figure}

\section{Conclusion}\label{sec:conclusion}
We introduced \method, regime-switching, explainable, and activation-induced linear models that derive latent regimes from the piecewise-affine activation geometry of a distilled wide, shallow ReLU student network (WSSN) and fit a mixture of local linear experts in the discovered regimes. \method provides dual interpretability: 1) regime-specific feature attributions explaining each prediction; 2) an explanatory gate explaining the regime-switching logic. By analyzing an idealized setting in which the true regimes are known and the model is exactly linear within each, we characterized how regime assignment responds to input perturbations and when within-regime supports are recoverable, revealing a complexity--stability trade-off in the number of regimes $K$, consistent with the performance curves in the ablation study in Section~\ref{sec:experiments:ablation:num_regimes}.

Experiments on tabular and image datasets showed that \method attains competitive predictive accuracy relative to prior DNN-guided mixture surrogates and inherently interpretable models, with local explanations illustrated qualitatively through gate coefficient heat maps and regime-level analysis. Ablations further showed that the router does not carry the predictive burden and that the explanatory gate can itself serve as the routing signal with little loss in accuracy.

\putbib[refs]

\makeatletter
\immediate\write\@auxout{\string\newlabel{count:bib}{{\arabic{enumiv}}{}{}{}{}}}%
\makeatother
\clearpage
\end{bibunit}

\begin{bibunit}[IEEEtran]
\spaceskip=0.33em plus 0.02em minus 0.2em
\setcounter{section}{0}
\renewcommand{\theHsection}{supp.\arabic{section}}

\makeatletter
\let\@REALMoldbibliography\thebibliography
\renewcommand{\thebibliography}[1]{%
  \@REALMoldbibliography{#1}%
  \setcounterref{enumiv}{count:bib}%
  \settowidth{\labelwidth}{\@biblabel{99}}%
  \leftmargin\labelwidth
  \advance\leftmargin\labelsep\relax
}
\makeatother

\title{Supplementary Material for\\ \method: Regime-Switching, Explainable, and Activation-Induced \\Linear Models}

\author{Xiaoran Cheng, Sen Na, and Jia Li,~\IEEEmembership{Fellow,~IEEE}}

\markboth{}{Cheng \MakeLowercase{\textit{et al.}}: \method---Supplementary Material}

\makesupplementtitle

\section{Datasets}\label{sec:datasets}
\noindent\textbf{Bike Sharing.}
This dataset records hourly bike rentals in the Capital Bikeshare system during 2011 and 2012. The target is the total number of rented bikes. The inputs are temporal and weather covariates, including year, month, hour, holiday indicator, weekday, working-day indicator, temperature, apparent temperature, humidity, wind speed, season, and weather condition. Hour, weekday, and month are encoded as sine and cosine pairs; season and weather condition are dummy encoded with the first level dropped, yielding \(19\) input features.

\medskip

\noindent\textbf{California Housing.}
Derived from the 1990 U.S. Census, this dataset contains census block groups in California, with median house value as the target. The 8 covariates are median income, house age, average numbers of rooms and bedrooms, population, average occupancy, latitude, and longitude.

\medskip

\noindent\textbf{Covertype.}
This dataset is used to classify forest cover types in the Roosevelt National Forest of northern Colorado. Each instance represents a \(30\,\mathrm{m}\times30\,\mathrm{m}\) patch of forested land. The inputs comprise 10 continuous variables (elevation, aspect, slope, horizontal and vertical distances to hydrology, horizontal distance to roadways, hillshade at 9 a.m., noon, and 3 p.m., and horizontal distance to fire points), 4 binary wilderness-area indicators, and 40 binary soil-type indicators, yielding \(54\) features.

\medskip

\noindent\textbf{MNIST.}
This handwritten digit classification dataset contains \(28\times28\) grayscale images of digits \(0\) through \(9\).

\medskip

\noindent\textbf{Flowers.}
This dataset contains \(180\times180\) color images (after resizing) from five flower categories: daisy, dandelion, roses, sunflowers, and tulips.

\medskip

\noindent\textbf{CIFAR-10.}
This dataset contains \(32\times32\) color images from ten object classes: airplane, automobile, bird, cat, deer, dog, frog, horse, ship, and truck.

\section{Implementation Details}\label{sec:implementation}
Across all datasets, we first train a DNN teacher and then follow the three-stage pipeline in Section~\ref{sec:method} of the main paper: 1)~we distill the teacher into the WSSN with distillation weight \(\alpha=0.7\) and, for classification, temperature \(T=2\) in Eq.~\eqref{eq:2}; 2)~we freeze the WSSN, fixing both \(s(x)\) and \(u(x)\), and optimize only the \(K\) centroids to discover regimes; 3)~with the routing posterior \(\pi_k(x)\) fixed, we fit the local experts and explanatory gate using Eqs.~\eqref{eq:7} and~\eqref{eq:9}, respectively. Table~\ref{tab:3} lists the dataset-specific hyperparameters.

\begin{table*}[t]
\begingroup
\footnotesize
\setlength{\tabcolsep}{5pt}
\renewcommand{\arraystretch}{1.08}
\caption{Dataset-specific hyperparameters. Shared hyperparameters are given in the text. The minimum regime size \(\tau\) is used only for tabular data augmentation.}
\label{tab:3}
\begin{center}
\begin{tabular}{l cccccc}
\hline
\rowcolor{gray!18}
 & Bike Sharing & Cal.\ Housing & Covertype & MNIST & Flowers & CIFAR-10 \\
\hline
Train / validation / test & 64/16/20\% & 64/16/20\% & 64/16/20\% & 48k/12k/10k & 2,569/550/551 & 45k/5k/10k \\
Batch size & 256 & 256 & 512 & 300 & 32 & 200 \\
Learning rate (teacher) & \(10^{-3}\) & \(10^{-3}\) & \(10^{-3}\) & \(10^{-3}\) & \(10^{-3}\) & \(10^{-3}\) \\
Learning rate (Stage~1) & \(10^{-3}\) & \(10^{-3}\) & \(3\times10^{-4}\) & \(10^{-3}\) & \(10^{-3}\) & \(10^{-3}\) \\
Learning rate (Stages~2--3) & \(10^{-3}\) & \(10^{-3}\) & \(10^{-3}\) & \(10^{-3}\) & \(10^{-3}\) & \(10^{-3}\) \\
Representation dimension \(p\) & 19 & 8 & 54 & 2,304 & 384 & 384 \\
WSSN width \(n_{\mathrm{hid}}\) & 256 & 256 & 256 & 512 & 512 & 512 \\
Regimes \(K\) & 50 & 150 & 30 & 10 & 10 & 20 \\
Active-unit weight \(\omega\) & 10.0 & 2.0 & 2.0 & 2.0 & 2.0 & 2.0 \\
Sharpening factor \(\eta\) & 1.02 & 1.02 & 1.02 & 1.1 & 1.1 & 1.1 \\
Maximum inverse temperature \(\beta_{\max}\) & 150 & 150 & 150 & 100 & 300 & 300 \\
Minimum regime size \(\tau\) & 50 & 100 & 50 &  &  &  \\
\hline
\end{tabular}
\end{center}
\endgroup
\end{table*}

\medskip

\noindent\textbf{Optimization.}
All training stages use Adam for at most \(200\) epochs with an early-stopping patience of \(10\) epochs. Learning rates are listed in Table~\ref{tab:3}. The expert and gate penalty weights are selected from \(\{0,10^{-3},10^{-2},10^{-1}\}\) using the validation set. For final predictions, we use the checkpoint with the highest validation AUC for classification or the lowest validation RMSE for regression. For all datasets, the inverse temperature starts at \(\beta_{\mathrm{clus}}=1\) and follows Eq.~\eqref{eq:4}.

\medskip

\noindent\textbf{Augmentation.}
For tabular data, we generate \(\tau-n_k\) teacher-labeled synthetic samples for each regime with \(0<n_k<\tau\) using Eq.~\eqref{eq:10} with \(\epsilon=0.05\). Image data are not augmented because sufficient training samples are already available in each regime.

\medskip

\noindent\textbf{Architectures.}
For tabular data, the teacher is a ReLU MLP with two hidden layers of widths \((256,64)\). The WSSN has one ReLU hidden layer and operates directly on the original features, so \(u(x)=x\). For MNIST, the teacher has two \(5\times5\) convolutional layers with \(4\) and \(8\) channels, respectively, each followed by average pooling, and then a fully connected layer of width \(120\) and a final classification layer. The WSSN feature extractor uses a single \(5\times5\) convolution with \(4\) channels. Note that, for MNIST, the \(4\times24\times24\) feature maps are flattened without pooling. For Flowers and CIFAR-10, the teacher is a LeNet-style CNN with two \(3\times3\) convolutional layers, each with \(32\) channels. The WSSN feature extractor uses a depthwise \(6\times6\) convolution with \(128\) filters per input channel, followed by adaptive max pooling to one value per channel. For image datasets, the WSSN's ReLU head, local experts, and explanatory gate all operate on the extracted representation \(u(x)\).

\medskip

\noindent\textbf{Data Splits and Preprocessing.}
For tabular datasets, we randomly split the data into training, validation, and test sets in each run, using class-stratified sampling for classification tasks. For MNIST and CIFAR-10, we use the official test sets and randomly split the training sets into training and validation subsets. For Flowers, we randomly divide the \(3{,}670\) images into training, validation, and test sets and resize them to \(180\times180\). All image inputs are normalized separately for each channel.

\section{Additional Interpretability Analyses}\label{sec:tab_interp} 
We provide tabular routing and coefficient diagnostics, followed by additional local expert contribution maps for MNIST and Flowers.

\medskip

\noindent\textbf{Routing and Regime Sizes.}
Fig.~\ref{fig:tab:route} shows sharp routing across all three datasets, with maximum posteriors concentrated near one and no single dominant regime. California Housing has many small regimes and a long right tail in regime size, indicating a more fragmented partition.

\begin{figure*}[!t]
\centering
\captionsetup[subfloat]{captionskip=0pt}
\includegraphics[width=0.32\textwidth]{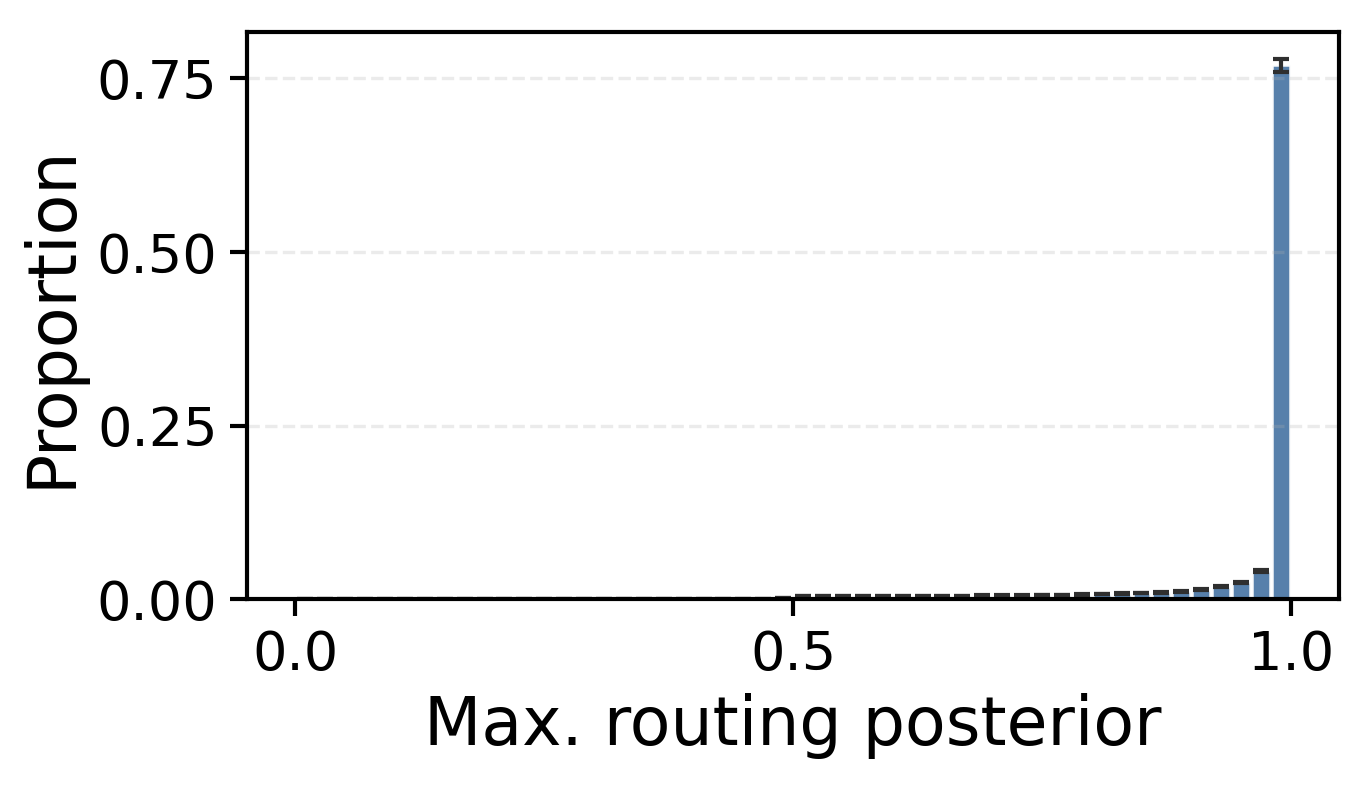}
\hfill
\includegraphics[width=0.32\textwidth]{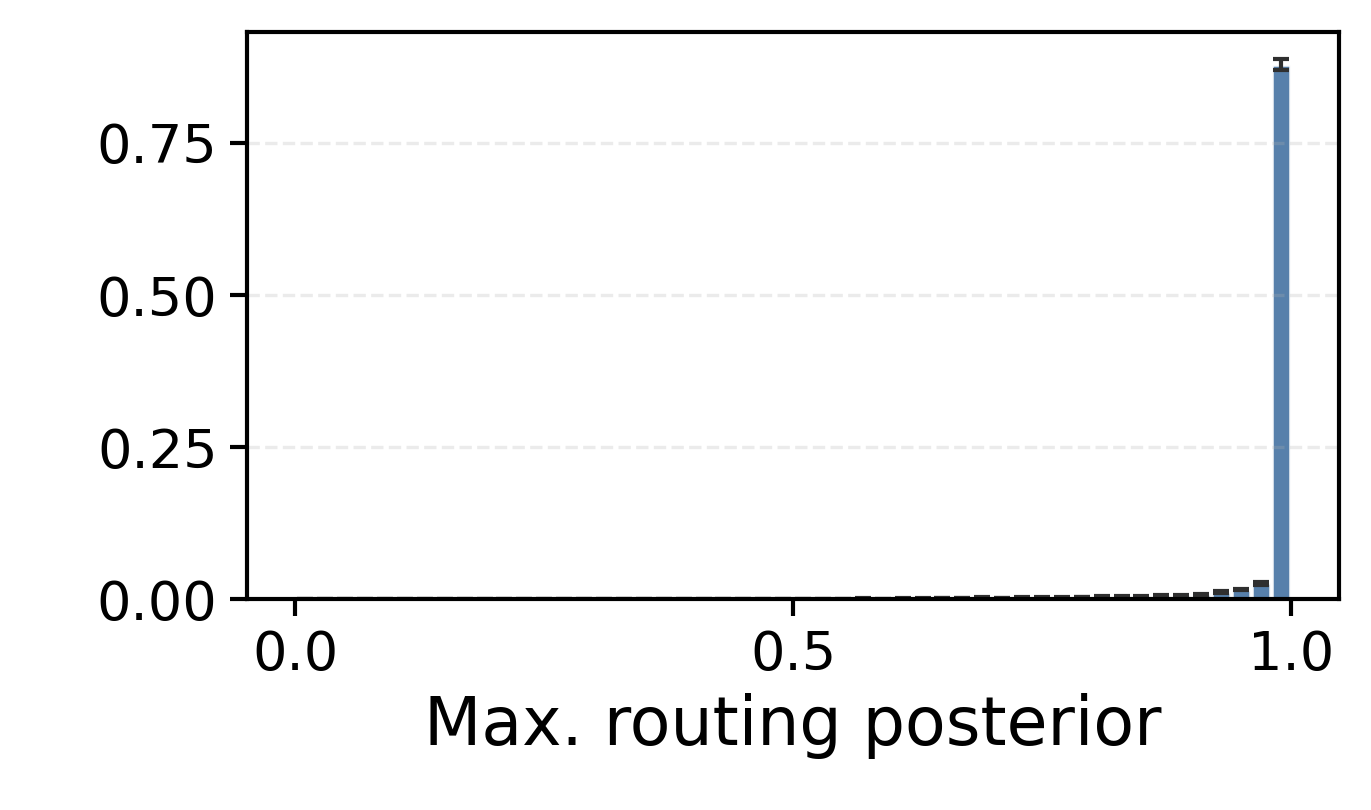}
\hfill
\includegraphics[width=0.32\textwidth]{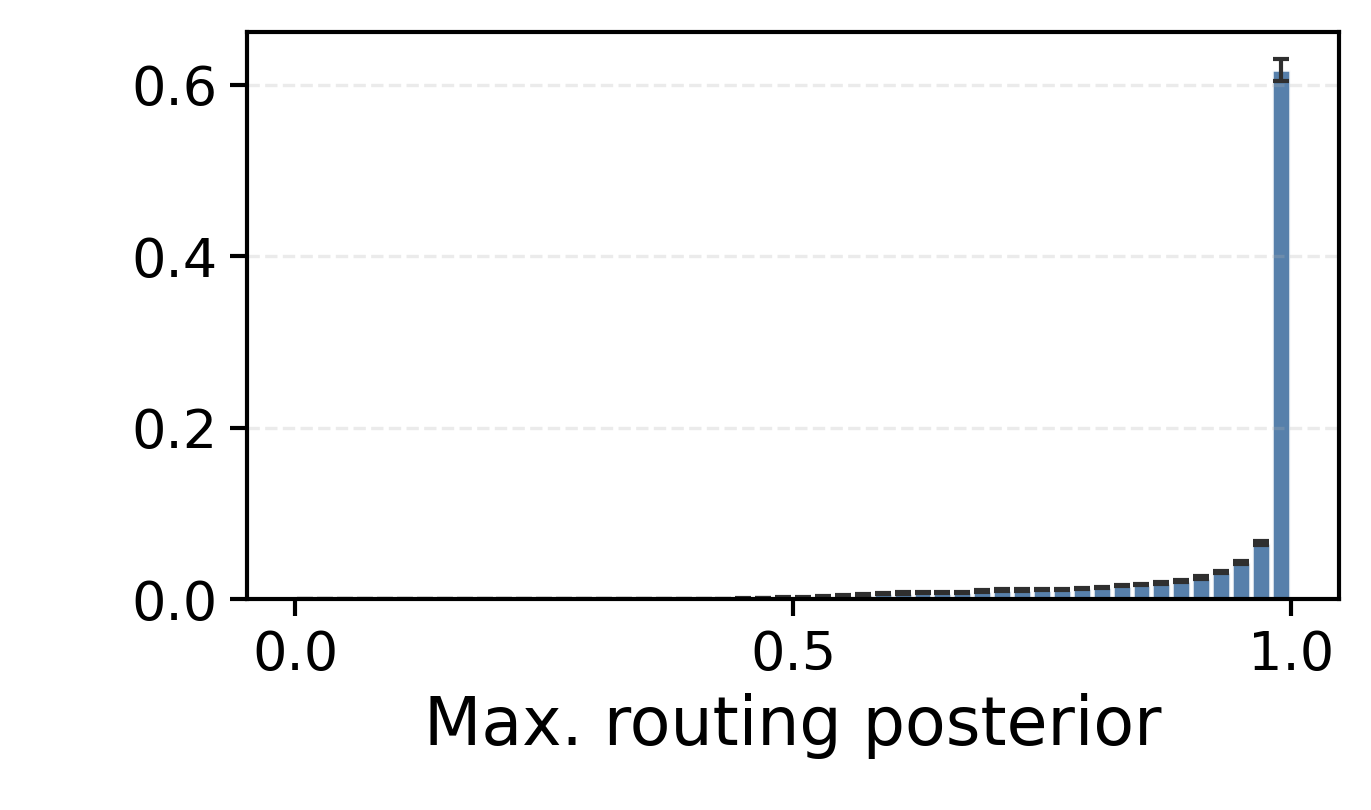}
\par\vspace{-10pt}
\subfloat[Covertype]{\includegraphics[width=0.32\textwidth]{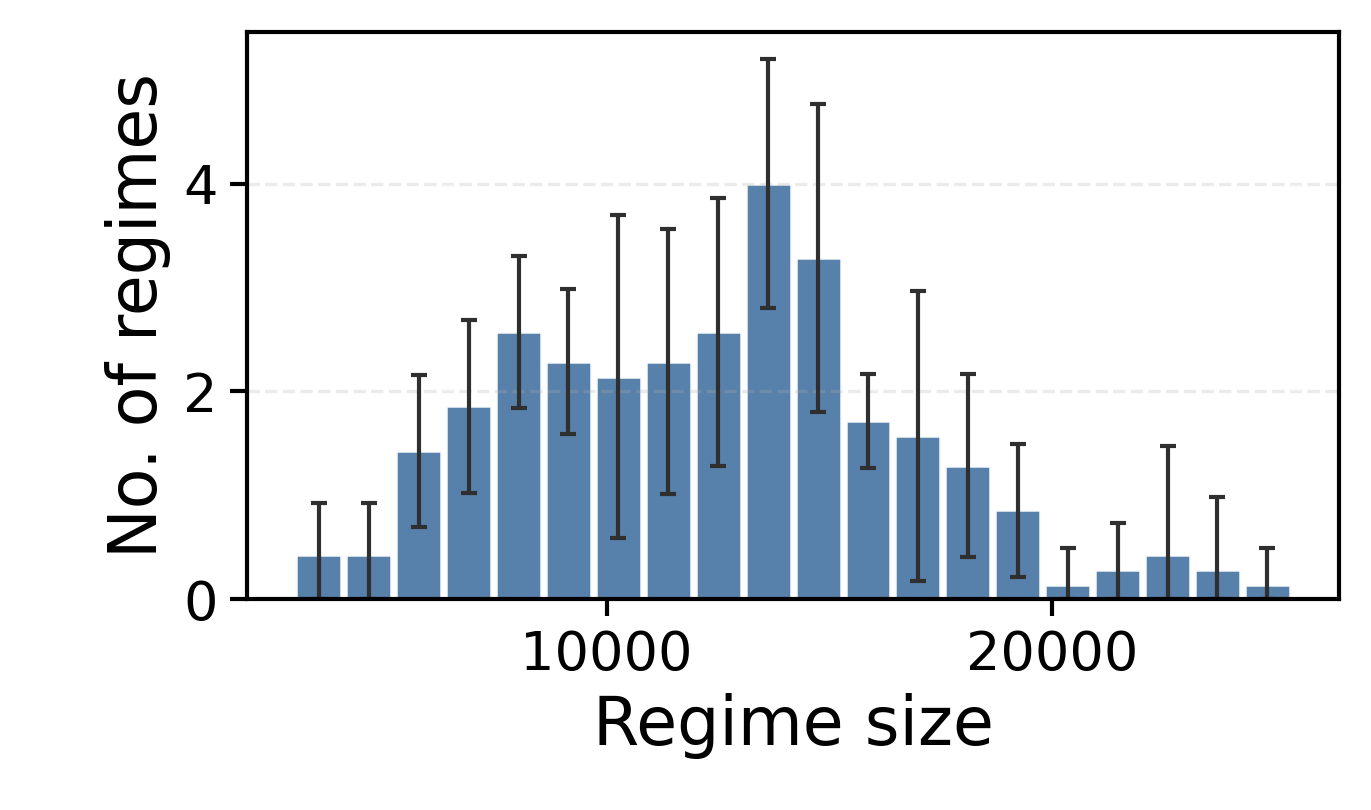}\label{fig:cov:post}\label{fig:cov:size}\label{fig:cov}}
\hfill
\subfloat[Bike Sharing]{\includegraphics[width=0.32\textwidth]{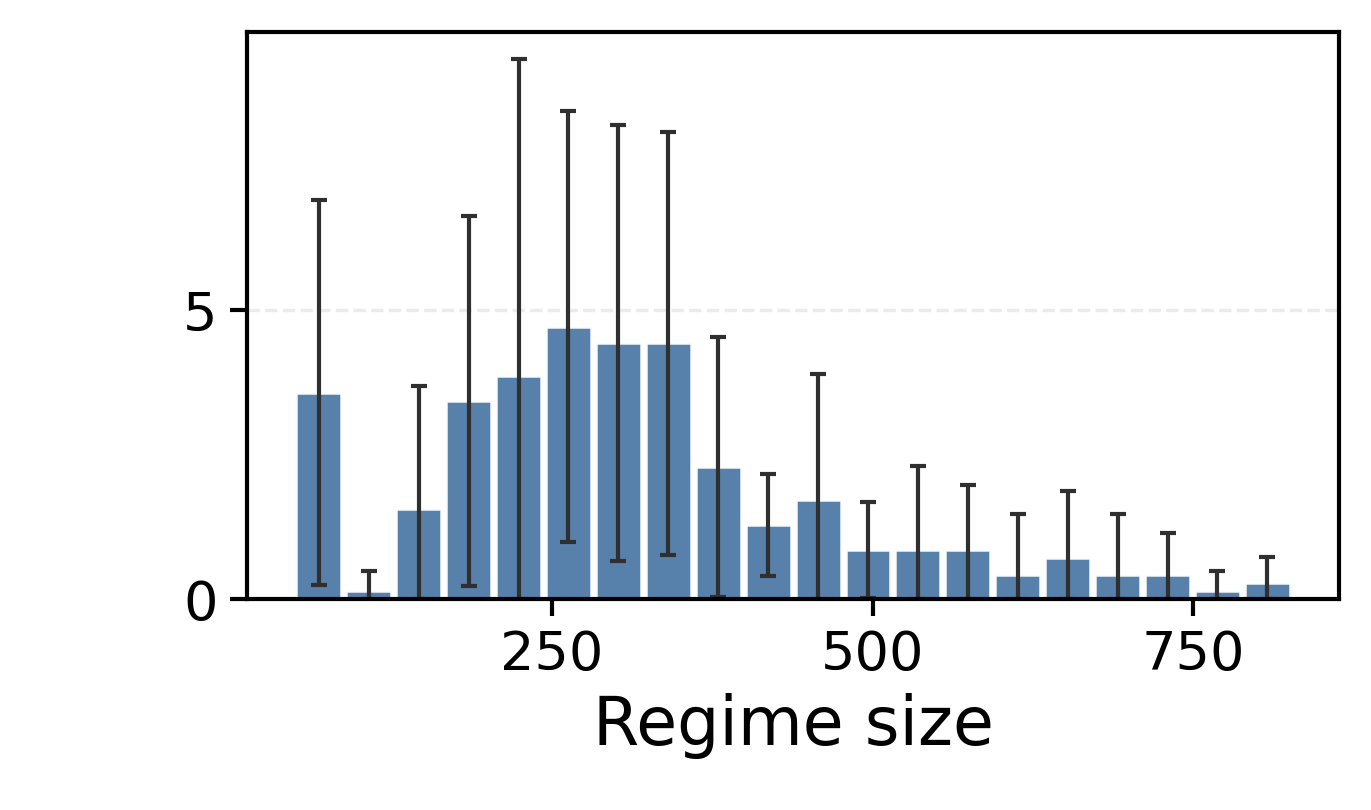}\label{fig:bike:post}\label{fig:bike:size}\label{fig:bike}}
\hfill
\subfloat[California Housing]{\includegraphics[width=0.32\textwidth]{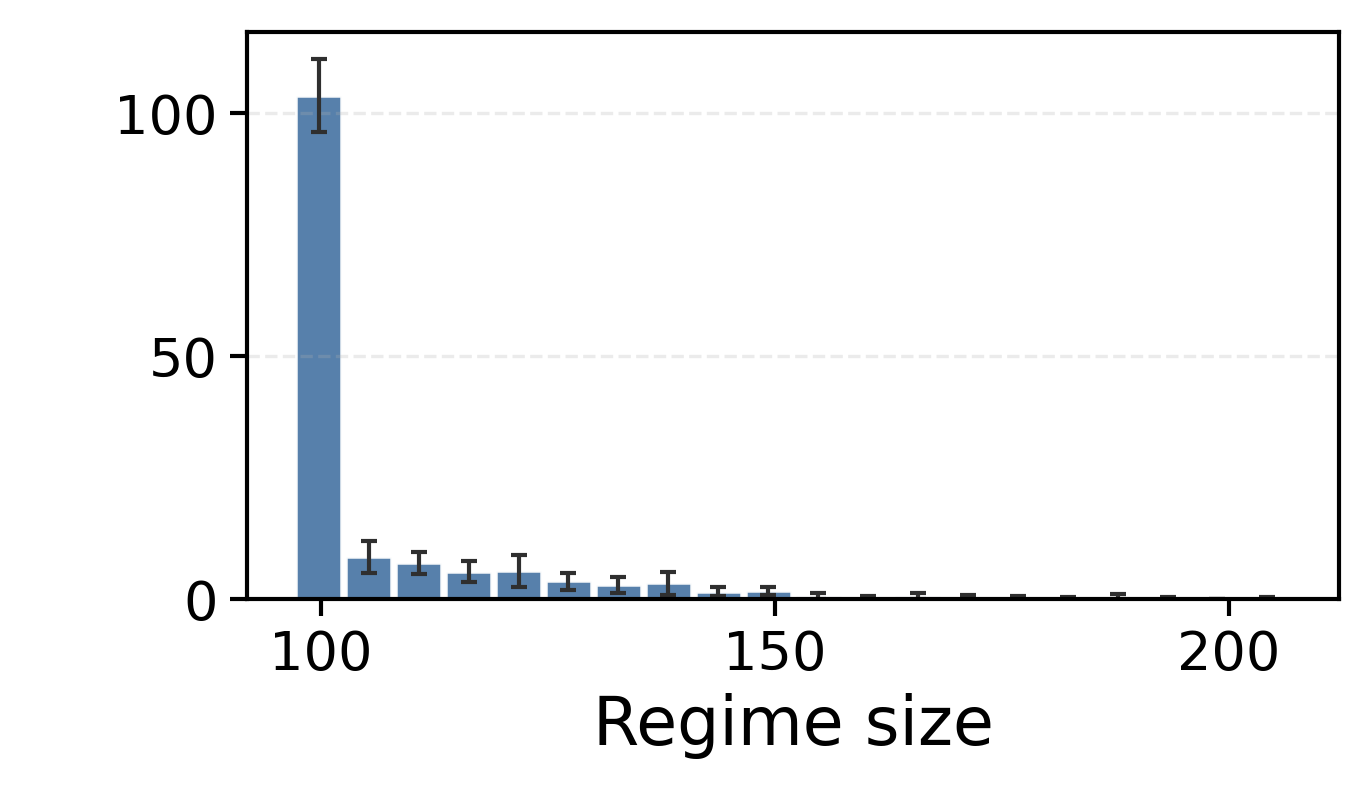}\label{fig:cal:post}\label{fig:cal:size}\label{fig:cal}}
\caption{Routing sharpness and regime sizes on Covertype~(a), Bike Sharing~(b), and California Housing~(c). Conventions follow Fig.~\ref{fig:img:route} of the main paper.}
\label{fig:tab:route}
\end{figure*}

\begin{figure*}[!t]
\centering
\begingroup
\captionsetup[subfloat]{captionskip=0pt}
\includegraphics[width=0.33\textwidth]{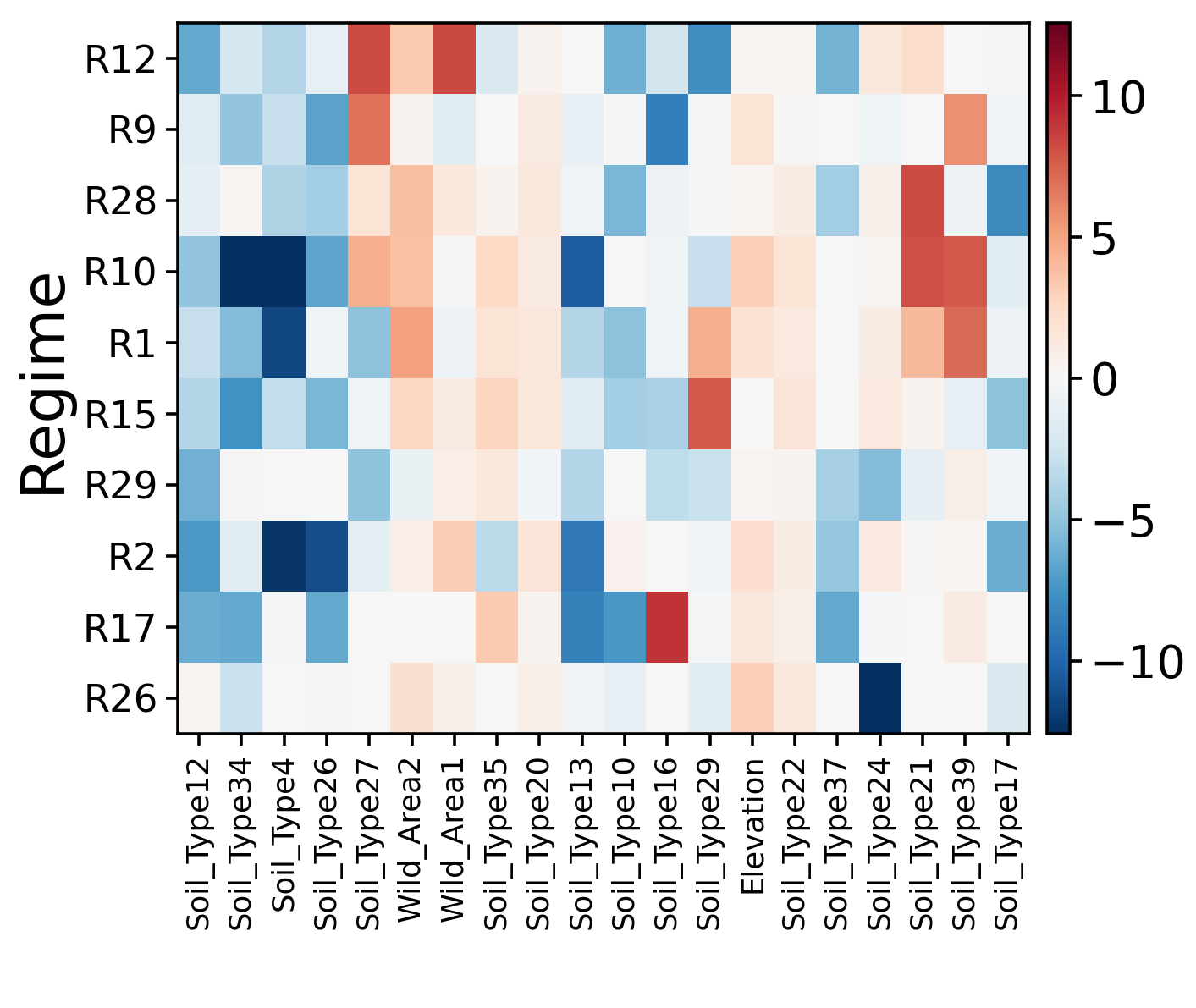}%
\hfill
\includegraphics[width=0.33\textwidth]{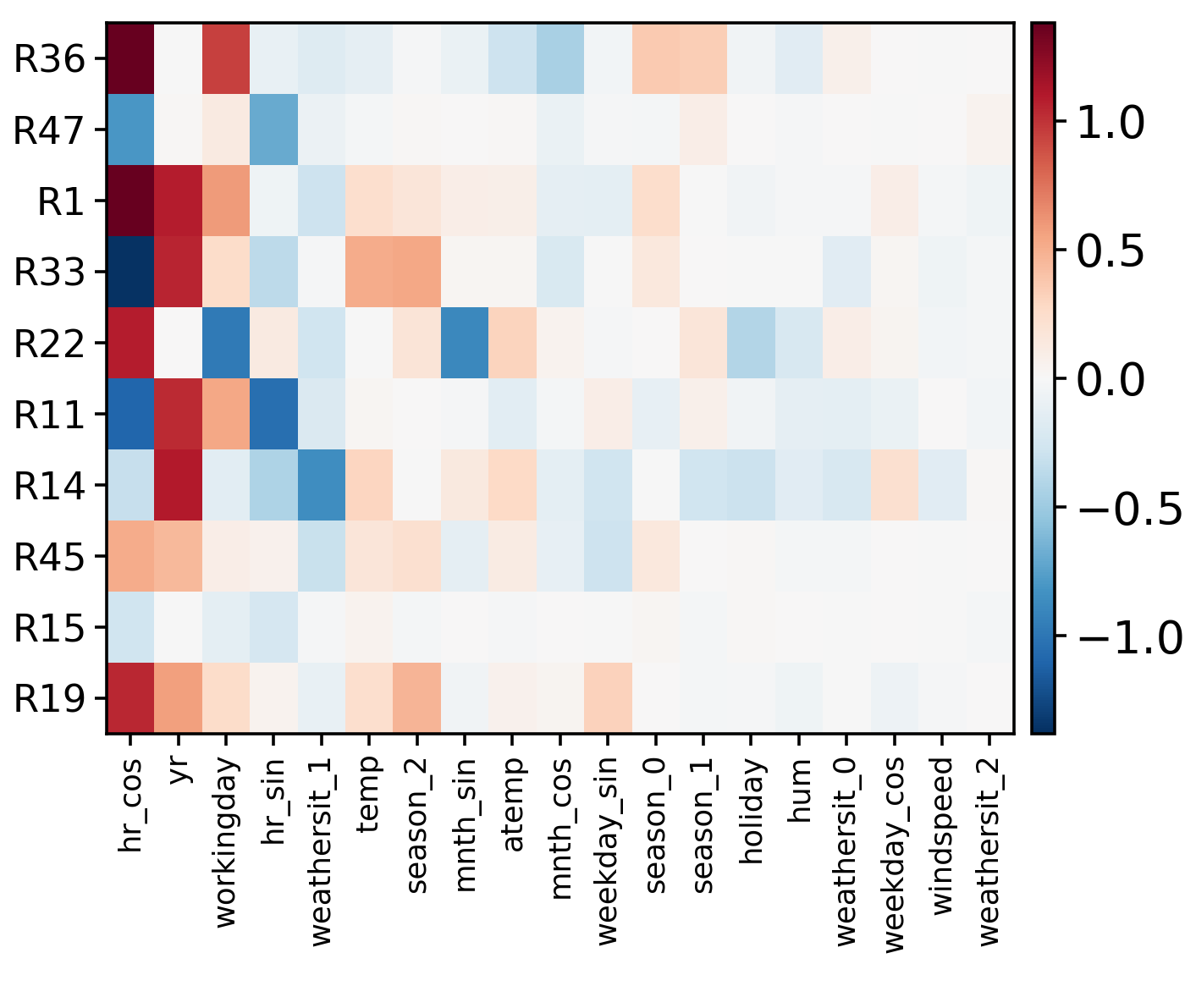}%
\hfill 
\includegraphics[width=0.33\textwidth]{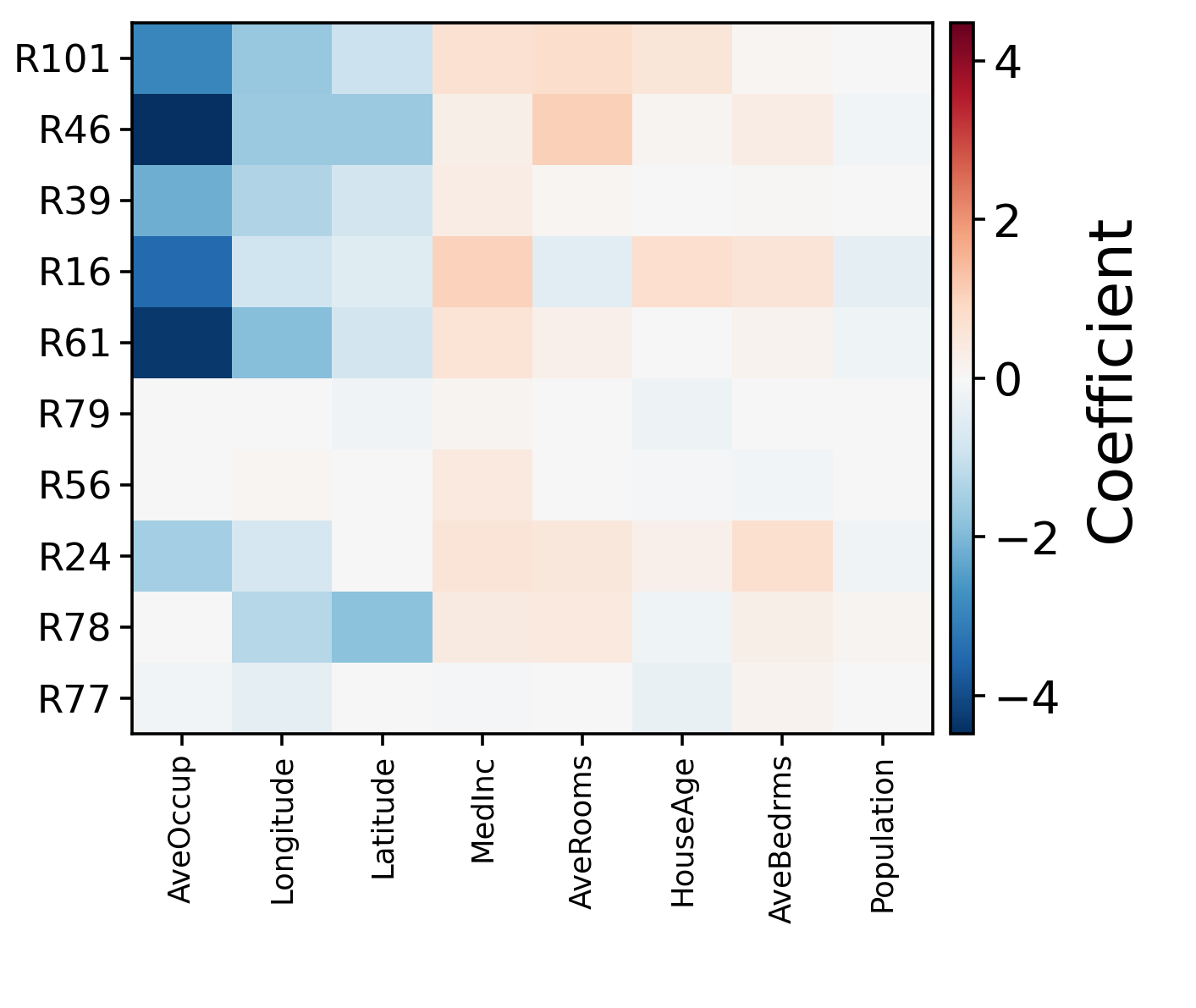}
\par\vspace{-10pt}
\subfloat[Covertype]{\includegraphics[width=0.33\textwidth]{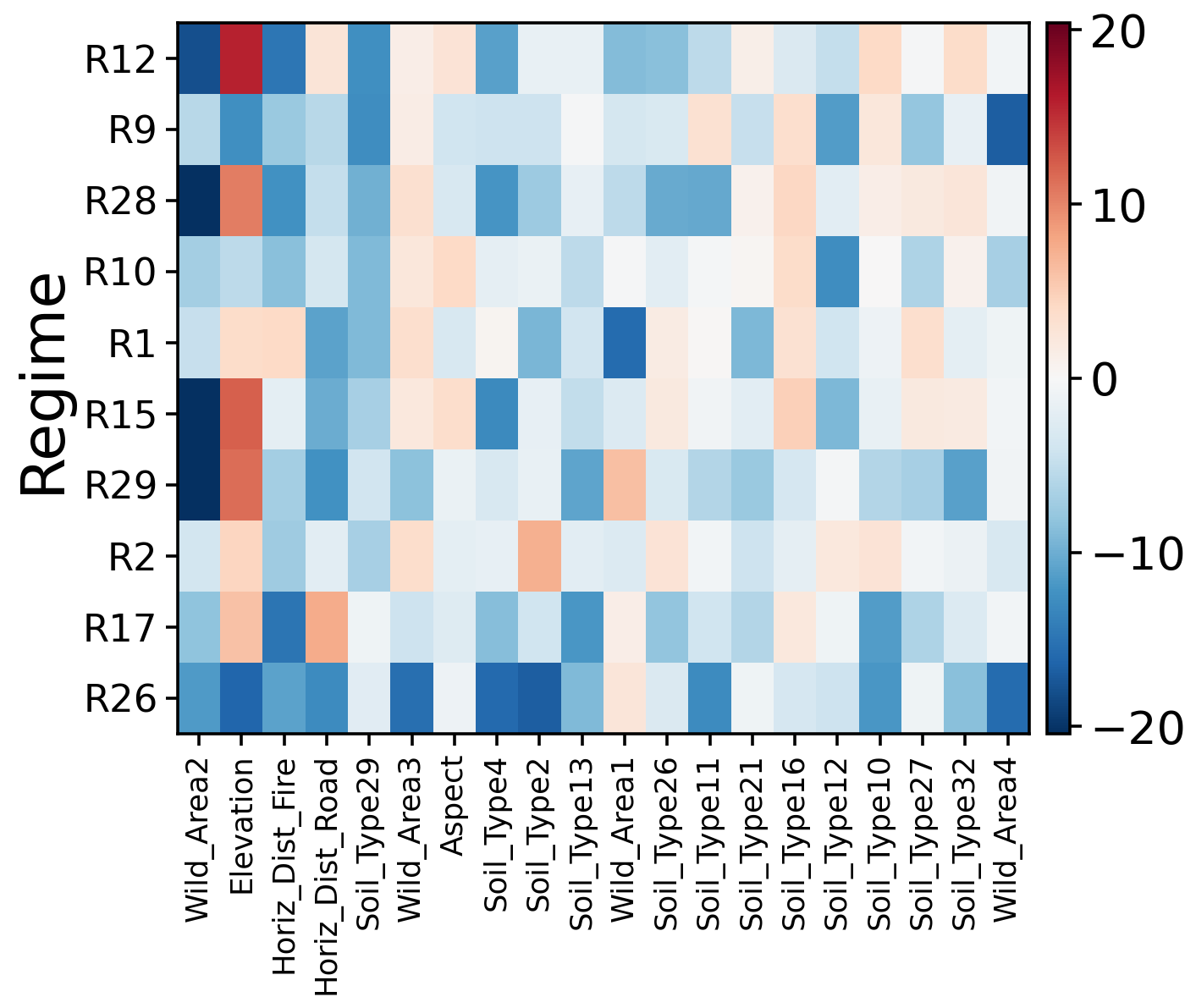}\label{fig:cov:expert}\label{fig:cov:gate}}%
\hfill
\subfloat[Bike Sharing]{\includegraphics[width=0.33\textwidth]{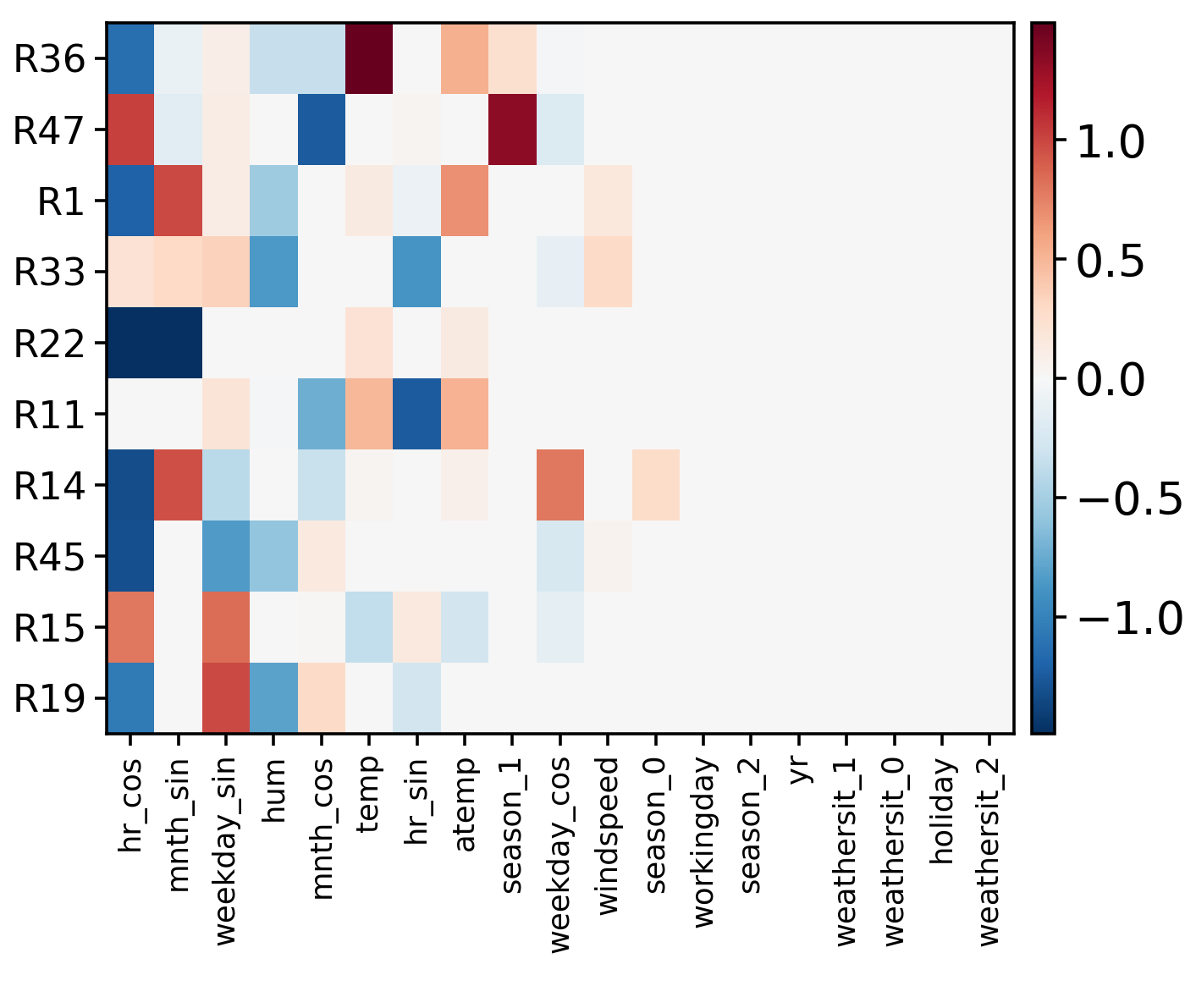}\label{fig:bike:expert}\label{fig:bike:gate}}%
\hfill
\subfloat[California Housing]{\includegraphics[width=0.33\textwidth]{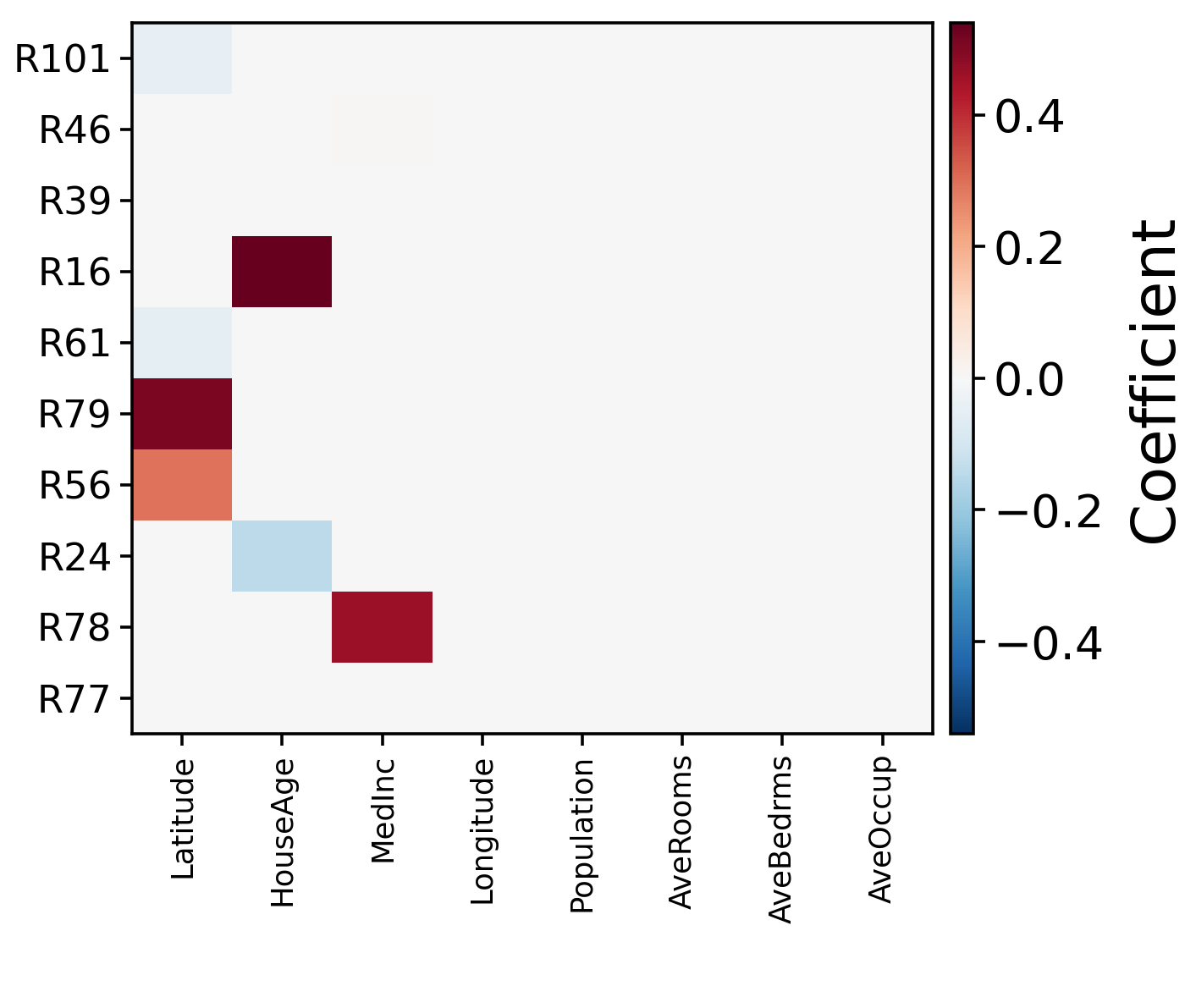}\label{fig:cal:expert}\label{fig:cal:gate}}
\endgroup
\caption{Single-run local expert coefficients (top) and explanatory gate coefficients (bottom) on Covertype~(a), Bike Sharing~(b), and California Housing~(c). Covertype experts show class-0 logit coefficients. Conventions follow Fig.~\ref{fig:img:gate} of the main paper.}
\label{fig:tab:coef}
\end{figure*}

\medskip

\noindent\textbf{Local Expert Coefficients.}
The top row of Fig.~\ref{fig:tab:coef} shows that the same feature can have coefficients of opposite signs across regimes. These sign changes reflect distinct local prediction rules. For example, on Covertype, \texttt{Soil\_Type29} has a coefficient of \(-7.77\) in regime~12 and \(+7.76\) in regime~15. On Bike Sharing, \texttt{hr\_cos} has a coefficient of \(+1.90\) in regime~36 and \(-1.36\) in regime~33. On California Housing, the coefficient of \texttt{HouseAge} changes from \(+0.74\) in regime~16 to \(-0.33\) in regime~77. As shown in Table~\ref{tab:4}, local experts retain most or all features in the largest regime, indicating limited feature sparsity.

\medskip

\noindent\textbf{Explanatory Gate Coefficients.}
The bottom row of Fig.~\ref{fig:tab:coef} shows regime-assignment rules. On Covertype, the coefficient of \texttt{Elevation} is positive in regime~29 and negative in regime~26, so higher \texttt{Elevation} favors regime~29 over regime~26 in the explanatory gate. On Bike Sharing and California Housing, many cells are near white, indicating greater feature sparsity in the explanatory gate. These patterns agree with Table~\ref{tab:4}. In the largest regime, the gate retains an average of \(54\) out of \(54\) features on Covertype, \(11.9\) out of \(19\) on Bike Sharing, and \(1.4\) out of \(8\) on California Housing.

\begin{table}[t]
\centering
\begingroup
\setlength{\tabcolsep}{3pt}
\renewcommand{\arraystretch}{1.08}
\caption{Feature-selection statistics for the main results in Tables~\ref{tab:1} and~\ref{tab:2}, reported as mean (std) over five runs. \emph{Selected features} counts features with absolute coefficient $\ge 10^{-3}$ in the largest regime by mean training-set size, separately for the explanatory gate and the local expert. \emph{Total features} is the representation dimension $p$.}
\label{tab:4}
\begin{adjustbox}{max width=\columnwidth}
\begin{tabular}{@{}lccc@{}}
\hline
\rowcolor{gray!18}
 & \multicolumn{2}{c}{Selected features} & \\
\cline{2-3}
\rowcolor{gray!18}
\multirow{-2}{*}{Dataset}
 & Explanatory gate & Local expert & \multirow{-2}{*}{Total features} \\
\hline
Bike Sharing & 11.9 (5.4) & 17.3 (1.6) & 19 \\
Cal.\ Housing & 1.4 (2.6) & 8.0 (0.0) & 8 \\
Covertype & 54.0 (0.0) & 54.0 (0.0) & 54 \\
MNIST & 2,286.6 (9.8) & 2,290.1 (5.6) & 2,304 \\
Flowers & 237.3 (127.1) & 359.9 (49.6) & 384 \\
CIFAR-10 & 381.9 (0.9) & 383.3 (0.8) & 384 \\
\hline
\end{tabular}
\end{adjustbox}
\endgroup
\end{table}

\medskip

\noindent\textbf{Local Expert Contribution Maps.}
Fig.~\ref{fig:supp:heat}\subref{fig:6a} shows how the WSSN's four convolutional channels capture different parts of MNIST digits. Since the feature extractor has only four channels, the Top-4 and All-channel sums coincide. For the ``3'', channel~2 emphasizes the bends, channels~1 and~3 follow the horizontal strokes, and channel~0 highlights the left side. For the ``7'', channel~1 highlights the top bar, channel~0 follows the vertical stroke, and channel~2 emphasizes the inner corner. For Flowers in Fig.~\ref{fig:supp:heat}\subref{fig:6b}, all four selected filter groups emphasize the radial structure of the dandelion seed head in row~1, a pattern also visible in both aggregate maps. For the daisy in row~3, the selected filter groups trace the white petals and yellow disc.

\begin{figure*}[t]
\centering
\subfloat[MNIST]{\includegraphics[width=0.49\textwidth]{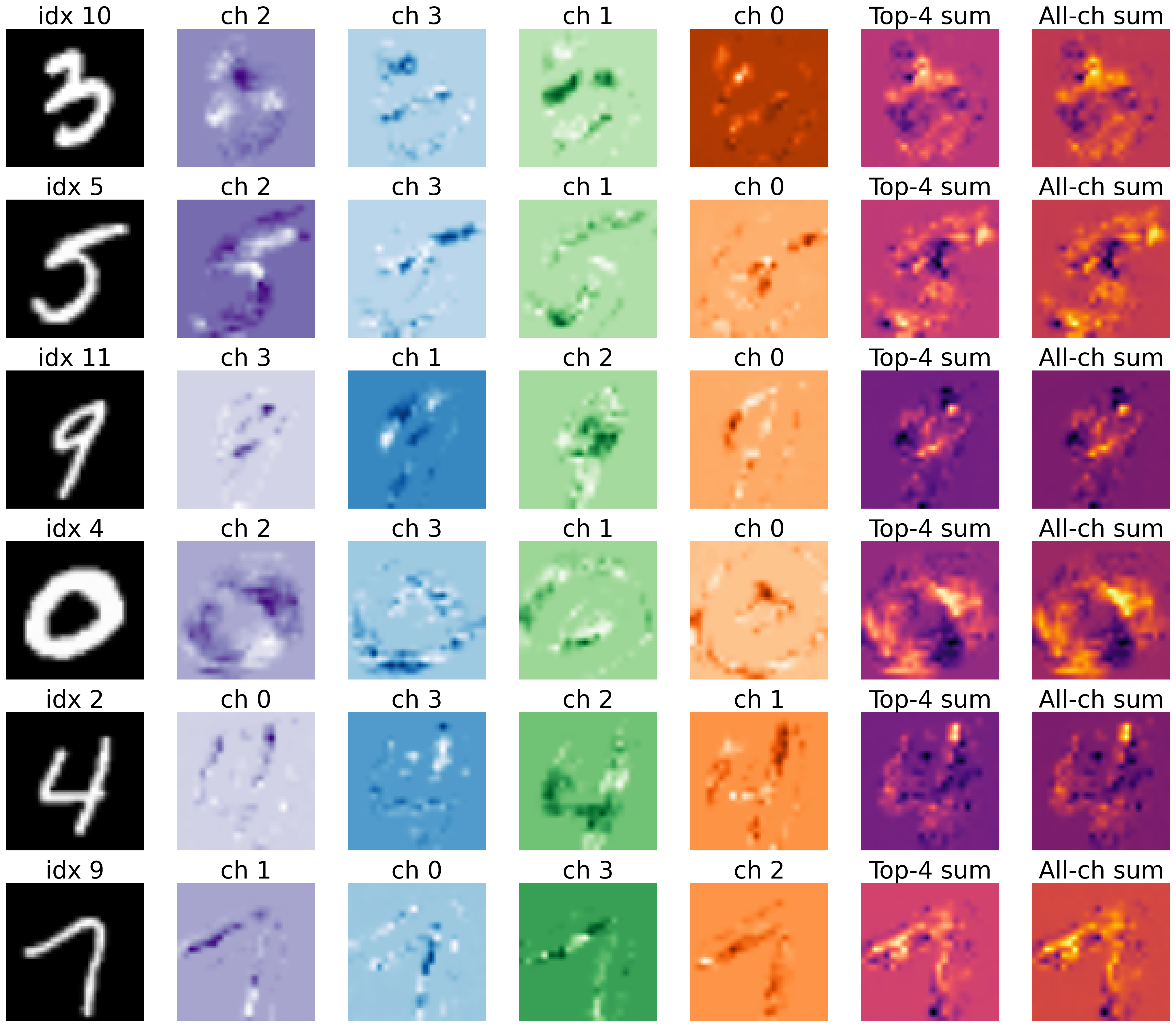}\label{fig:6a}}
\hfill
\subfloat[Flowers]{\includegraphics[width=0.49\textwidth]{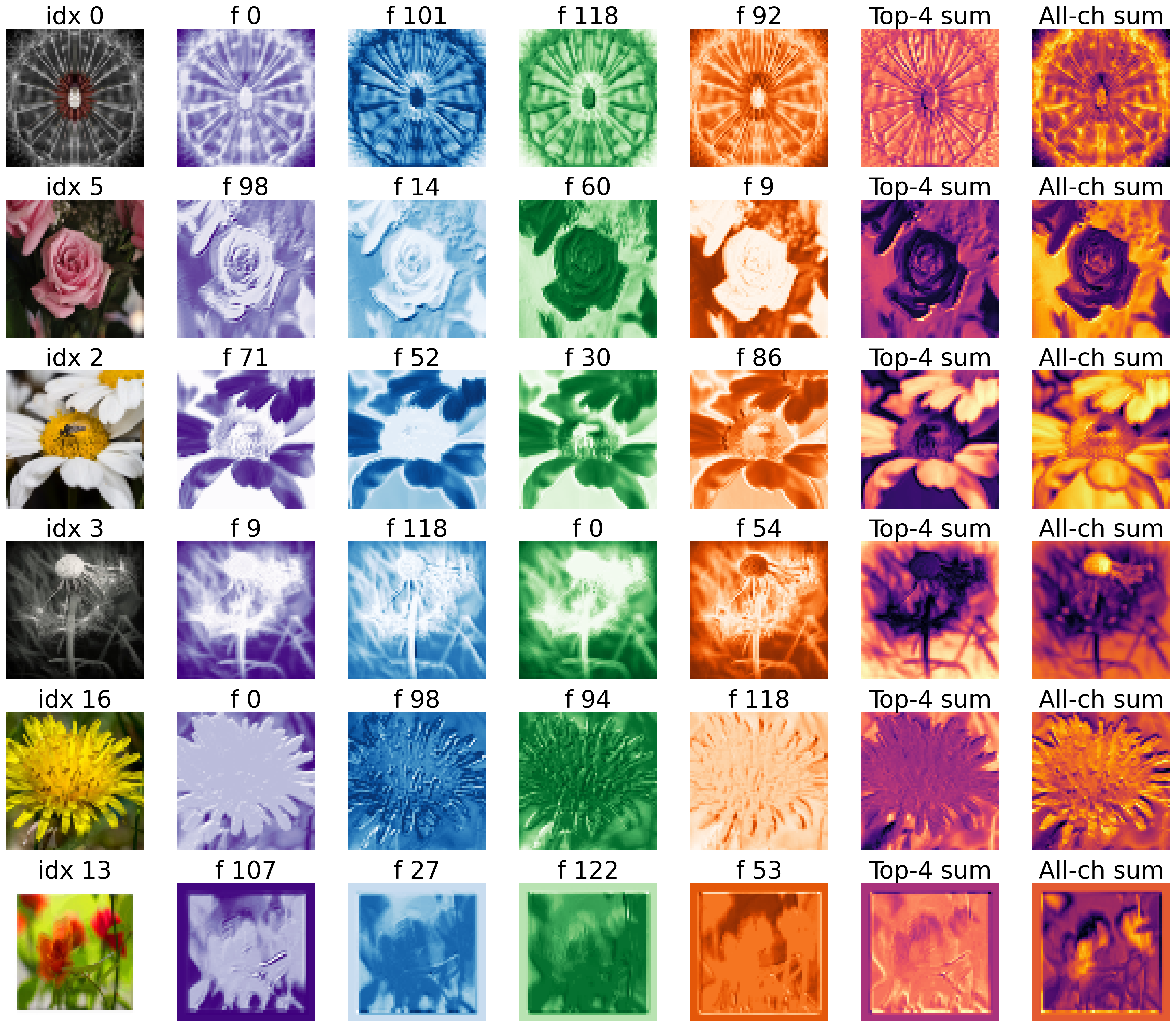}\label{fig:6b}}
\caption{Local expert contribution maps on MNIST~(a) and Flowers~(b). For MNIST, \texttt{ch} $j$ denotes convolutional channel $j$. Other conventions follow Fig.~\ref{fig:cifar:heat}.}
\label{fig:supp:heat}
\end{figure*}

\section{Assumptions and Auxiliary Results for Section~\ref{sec:theory:joint} of the Main Paper}
\label{sec:supp:assumptions}
\setcounterref{assumption}{count:assumption}
\setcounterref{proposition}{count:proposition}
\setcounterref{theorem}{count:theorem}
\renewcommand{\theassumption}{\getrefnumber{sec:theory}.\arabic{assumption}}
\renewcommand{\theproposition}{\getrefnumber{sec:theory}.\arabic{proposition}}
\renewcommand{\thetheorem}{\getrefnumber{sec:theory}.\arabic{theorem}}

We continue to use the notation and idealized regression setting of Section~\ref{sec:theory:joint} of the main paper. Below, we state the assumptions and the LASSO support-recovery bound used in that analysis.
All probabilities in this support-recovery analysis are conditional on the training inputs.

\begin{assumption}[Restricted Eigenvalue Condition]
\label{ass:re}

Consider regime $k$ with true support $S_k = \operatorname{supp}(\beta_k^*)$. Let $\Delta \in \mathbb{R}^p$ denote a generic vector, and in particular the estimation error $\Delta = \hat\beta_k - \beta_k^*$. For any index set $A \subseteq \{1,\dots,p\}$, let $\Delta_A$ denote the restriction of $\Delta$ to coordinates in $A$.

We assume that $X_k \in \mathbb{R}^{n_k \times p}$ satisfies the following condition: there exists a constant $\kappa_k > 0$ such that
\begin{equation*}
\frac{1}{n_k} \|X_k \Delta\|_2^2
\;\ge\;
\kappa_k \|\Delta_{S_k}\|_2^2
\end{equation*}
for all vectors $\Delta$ satisfying the cone constraint
\begin{equation*}
\|\Delta_{S_k^c}\|_1
\le
3 \|\Delta_{S_k}\|_1.
\end{equation*}
We also assume $\max_j\|X_{k,j}\|_2/\sqrt{n_k}\le1$, where $X_{k,j}$ is column $j$ of $X_k$.
\end{assumption}

\begin{assumption}[Irrepresentable Condition]
\label{ass:irrep}
Let $\Sigma_k = \frac{1}{n_k}X_k^\top X_k$ denote the within-regime Gram matrix. There exists $\gamma_k\in(0,1]$ such that
\begin{equation*}
\big\|\Sigma_{k,S_k^cS_k}\,\Sigma_{k,S_kS_k}^{-1}\big\|_\infty
\;\le\;
1-\gamma_k,
\end{equation*}
where, for a matrix $A$, $\|A\|_\infty=\max_i\sum_j|A_{ij}|$ is the $\ell_\infty$ operator norm (maximum absolute row sum).
\end{assumption}

\begin{assumption}[Beta-Min Condition]
\label{ass:betamin}
For each regime $k$,
\begin{equation*}
\min_{j \in S_k} |\beta_{k,j}^*|
\ge C_k \lambda
\end{equation*}
for $C_k>\|\Sigma_{k,S_kS_k}^{-1}\|_\infty+4/\sqrt{\kappa_k}$, where $\lambda$ is the LASSO regularization hyperparameter.
\end{assumption}

These conditions are standard in high-dimensional sparse regression. The restricted eigenvalue condition rules out degenerate design directions, the irrepresentable condition limits dependence between on-support and off-support variables, and the beta-min condition ensures that non-zero coefficients are large enough to be detected. Under Gaussian noise, these assumptions and the penalty choice below yield the support-recovery bound in Theorem~\ref{thm:wainwright} (see, e.g.,~\cite{Wainwright2009Sharp,Buehlmann2011Statistics}).

Since the regime selector is estimated, we additionally quantify how accurately it recovers the regime assignments.

\begin{assumption}[Regime Learning Accuracy]
\label{ass:regime_learning}
Let $\hat g$ be a regime selector learned from $\mathcal D_n$. For the sequence of training inputs considered here, there exists a sequence $\eta_n\to 0$ such that
\begin{equation*}
\sup_{x\in\mathcal X}
\mathbb P_{\mathcal D_n}\big(\hat g(x)\neq g^*(x)\big)\le \eta_n.
\end{equation*}
\end{assumption}

\begin{theorem}[LASSO support recovery; adapted from Wainwright \cite{Wainwright2009Sharp}]
\label{thm:wainwright}

Consider the linear model within the true regime $k^*$:
\begin{equation*}
y = X\beta^* + \varepsilon,
\qquad
\varepsilon \sim \mathcal N(0,\sigma^2 I_n),
\end{equation*}
where $n=n_{k^*}$, $\beta^*=\beta_{k^*}^*$, $\sigma>0$, and $X=X_{k^*}\in\mathbb{R}^{n\times p}$ is fixed.
The LASSO estimator minimizes $\|y-Xb\|_2^2/(2n)+\lambda\|b\|_1$ over $b\in\mathbb R^p$.

Suppose Assumptions~\ref{ass:re}--\ref{ass:betamin} hold for this regime, with $\lambda=A_\lambda\sigma\sqrt{\log p/n}$ and $A_\lambda\ge4/\gamma_{k^*}$.

Then there exist constants $c_1,c_2>0$, depending only on the design constants, such that
\begin{equation}\label{eq:wainwright}
\mathbb{P}\big(\hat S_{k^*} \neq S_{k^*}\big)
\le
c_1 e^{-c_2 n_{k^*}\lambda^2/\sigma^2}.
\end{equation}

\end{theorem}
Here $\hat S_{k^*}$ is the support estimated from the $n_{k^*}$ oracle training samples in regime $k^*$. The bound follows from Theorem~1 of~\cite{Wainwright2009Sharp} applied to $y/\sigma$ with penalty $\lambda/\sigma$.

\begin{proposition}[Oracle End-to-End Support-Recovery Bound]
\label{prop:joint_full}
Fix a deterministic signal $x$ with true regime $k^* = g^*(x)$ and margin $m(x)>0$, and let $\hat k=\hat g(x+\delta)$ denote the predicted regime. Each expert is fitted on the oracle index set $\mathcal I_k^*$ defined in Section~\ref{sec:theory:joint} of the main paper. The training data $\mathcal D_n$ are independent of the test perturbation $\delta$.
Under Assumptions~\ref{ass:gating_setup}--\ref{ass:gate_norm} of the main paper, Assumptions~\ref{ass:re}--\ref{ass:regime_learning} above, and the stated penalty choice, Proposition~\ref{prop:sharp_stability} and Theorem~\ref{thm:wainwright} give
\begin{equation}\label{eq:joint_full}
\begin{aligned}
\mathbb{P}\big(\hat S_{\hat k} \neq S_{k^*}\big)
&\le
\underbrace{\eta_n}_{\text{regime learning}}
+
\underbrace{c_1 e^{-c_2 n_{k^*}\lambda^2/\sigma^2}}_{\text{support recovery}} \\
&\quad+
\underbrace{\exp\!\left(
-\frac{1}{2}
\Big(
\frac{m(x)}{2\sigma_g B_v}
-
\sqrt{2\log(K-1)}
\Big)_+^2
\right)}_{\text{perturbation stability}}.
\end{aligned}
\end{equation}
Here $c_1,c_2>0$ are the constants in Theorem~\ref{thm:wainwright}. The probability is over training noise, learner randomness, and the independent perturbation $\delta$, conditional on the training inputs and with $x$ fixed.

\end{proposition}

\section{Proofs}\label{sec:proofs}
\subsection{Proof of Proposition~\ref{prop:sharp_stability}}
\label{sec:proof_sharp_stability}
Since $k^*=g^*(x)$ and
\begin{equation*}
m(x)=\min_{j\neq k^*}(v_{k^*}^*-v_j^*)^\top x,
\end{equation*}
misassignment implies that some competitor $j\neq k^*$ satisfies

\begin{equation*}
v_j^{*\top}(x+\delta)\ge v_{k^*}^{*\top}(x+\delta).
\end{equation*}
With $a_j=v_j^*-v_{k^*}^*$, this implies

\begin{equation*}
a_j^\top \delta \ge m(x),
\end{equation*}
so
\begin{equation*}
\{g^*(x+\delta)\neq k^*\}
\subseteq
\{M_{k^*}(\delta)\ge m(x)\}.
\end{equation*}

The function $\delta\mapsto M_{k^*}(\delta)$ is $L$-Lipschitz with

\begin{equation*}
L = \max_{j\neq k^*}\|a_j\|_2,
\end{equation*}
since each map $\delta\mapsto a_j^\top\delta$ is linear and $\|a_j\|_2$-Lipschitz.

By the Borell--Tsirelson--Ibragimov--Sudakov inequality for Lipschitz functions of Gaussian vectors~\cite{Vershynin2018High},
\begin{equation*}
\mathbb{P}_\delta\!\left(
M_{k^*}(\delta)
-
\mathbb{E}_\delta M_{k^*}(\delta)
\ge t
\right)
\le
\exp\!\left(
-\frac{t^2}{2\sigma_g^2 L^2}
\right),
\end{equation*}
for all $t\geq 0$.

Set
\begin{equation*}
t = m(x)-\mathbb{E}_\delta M_{k^*}(\delta),
\end{equation*}
and apply the concentration inequality for $t\ge 0$ to obtain Eq.~\eqref{eq:13}. For $t<0$, the positive-part term is zero and the right-hand side of Eq.~\eqref{eq:13} equals $1$, so the bound follows from the trivial probability bound.

The stated bound on $\mathbb{E}_\delta M_{k^*}(\delta)$ follows from a standard maximal inequality for sub-Gaussian random variables~\cite{Vershynin2018High}.

Finally,
\begin{equation*}
\|a_j\|_2
=
\|v_j^*-v_{k^*}^*\|_2
\le
\|v_j^*\|_2+\|v_{k^*}^*\|_2
\le 2B_v.
\end{equation*}
Substituting these two bounds into Eq.~\eqref{eq:13} yields Eq.~\eqref{eq:14}.

\subsection{Proof of Proposition~\ref{prop:joint_full}}
\label{sec:proof_joint_full}
We decompose the event $\{\hat S_{\hat k} \neq S_{k^*}\}$ (the estimated support in the predicted regime differs from the true support in the true regime) into three error events:
\begin{align*}
\{\hat S_{\hat k} \neq S_{k^*}\}
&\subseteq
\underbrace{\{\hat g(x+\delta)\neq g^*(x+\delta)\}}_{\text{(i) learned regime selector errors}} \notag\\
&\quad\cup\;
\underbrace{\{g^*(x+\delta)\neq g^*(x)\}}_{\text{(ii) true regime selector switches}} \notag\\
&\quad\cup\;
\underbrace{\{\hat S_{k^*}\neq S_{k^*},\, \hat k=k^*\}}_{\text{(iii) LASSO fails in the correct regime}}.
\end{align*}

We bound the three probabilities as follows:

\begin{itemize}
    \item Assumption~\ref{ass:regime_learning} bounds the learned selector's error by $\eta_n$ uniformly over inputs. Independence of the training randomness and $\delta$ lets us apply this conditional bound at $x+\delta$.
    \item Proposition~\ref{prop:sharp_stability} bounds the probability that the true regime switches under $\delta$.
    \item The third event is contained in $\{\hat S_{k^*}\neq S_{k^*}\}$. For the oracle-trained expert in regime $k^*$, Theorem~\ref{thm:wainwright} bounds its probability by $c_1\exp(-c_2 n_{k^*}\lambda^2/\sigma^2)$.
\end{itemize}

The union bound yields Eq.~\eqref{eq:joint_full}.

\putbib[refs]

\clearpage
\end{bibunit}
\end{document}